\documentclass[sigconf]{acmart}

\copyrightyear{2026}
\acmYear{2026}
\setcopyright{cc}
\setcctype{by}
\acmConference[KDD '26]{Proceedings of the 32nd ACM SIGKDD Conference on Knowledge Discovery and Data Mining V.2}{August 09--13, 2026}{Jeju Island, Republic of Korea}
\acmBooktitle{Proceedings of the 32nd ACM SIGKDD Conference on Knowledge Discovery and Data Mining V.2 (KDD '26), August 09--13, 2026, Jeju Island, Republic of Korea}
\acmDOI{10.1145/3770855.3818926}
\acmISBN{979-8-4007-2259-2/2026/08}

\usepackage{url}
\usepackage{graphicx}
\usepackage{algorithm}
\usepackage{algorithmic}
\usepackage{amsmath}
\usepackage{dsfont}
\usepackage{algorithm,algorithmic}
\usepackage{enumitem}
\usepackage{array}
\usepackage{subfigure}
\usepackage{multirow}
\usepackage{nameref}
\usepackage{amsfonts}
\usepackage{booktabs}
\usepackage{chngcntr}
\usepackage{amsthm}
\usepackage{thmtools}
\usepackage{caption}
\usepackage{wrapfig}
\usepackage{enumitem}
\usepackage{bm}
\usepackage{thm-restate}
\usepackage[table]{xcolor} 

\newcounter{problemcounter}
\renewcommand{\theproblemcounter}{\Roman{problemcounter}}

\newenvironment{problem}[1][]{%
  \refstepcounter{problemcounter}
}{%
  \par
}
\newcommand{\problemtag}{\tag{\textbf{Problem \theproblemcounter}}}
\newcommand{\methodname}{MOSIC}
\newtheorem{theorem}{Theorem}

\newtheorem{definition}{Definition}
\newtheorem{remark}{Remark}

\newcommand{\E}{\mathbb{E}}

\begin{document}

\title{MOSIC: Model-Agnostic Optimal Subgroup Identification with Multi-Constraint for Improved Reliability}

\author{Wenxin Chen}
\email{wc645@cornell.edu}
\affiliation{%
  \institution{Computer Science, Cornell University}
  \city{New York City}
  \state{New York}
  \country{USA}}

\author{Weishen Pan}
\email{wep4001@med.cornell.edu}
\affiliation{%
  \institution{Weill Cornell Medicine, Cornell University}
  \city{New York City}
  \state{New York}
  \country{USA}}

\author{Kyra Gan}
\email{kyragan@cornell.edu}
\affiliation{%
  \institution{Operations Research and Information Engineering, Cornell University}
  \city{New York City}
  \state{New York}
  \country{USA}}

\author{Fei Wang}
\email{few2001@med.cornell.edu}
\affiliation{%
  \institution{Weill Cornell Medicine, Cornell University}
  \city{New York City}
  \state{New York}
  \country{USA}}


\begin{abstract}
Current subgroup identification methods typically follow a two-step approach: first estimate \emph{conditional average treatment effects} (CATEs) and then apply thresholding or rule-based procedures to define subgroups. While intuitive, this decoupled approach fails to incorporate key constraints essential for real-world clinical decision-making—such as subgroup size and propensity overlap. These constraints operate on fundamentally different axes than CATE estimation and are not naturally accommodated within existing frameworks, thereby limiting the practical applicability of these methods. We propose a \emph{unified optimization framework} that directly solves the \emph{primal} constrained optimization problem to identify optimal subgroups. Our key innovation is a reformulation of the constrained primal problem as an \emph{unconstrained differentiable min-max objective}, solved via a gradient descent-ascent algorithm. We theoretically establish that our solution converges to a feasible and locally optimal solution. Unlike threshold-based CATE methods that apply constraints as post-hoc filters, our approach enforces them directly during optimization. The framework is \emph{model-agnostic}, compatible with a wide range of CATE estimators, and extensible to additional constraints like cost limits or fairness criteria. Extensive experiments on synthetic and real-world datasets demonstrate its effectiveness in identifying high-benefit subgroups while maintaining better satisfaction of constraints.
\end{abstract}




\maketitle

\vspace{-5pt}
\section{Introduction}
In precision medicine, a fundamental challenge is identifying patient subgroups that benefit most from specific treatments, where heterogeneous effects must be estimated from observational data \citep{kosorok2019precision, kravitz2004evidence}. 
Most existing methods adopt a two-step paradigm: they first estimate CATEs using machine learning methods, then derive subgroups through either thresholding or simplified rule-based models (see Section~\ref{sec:lit}). 

However, this two-stage approach falls short in real-world settings, where optimal subgroup identification must account for diverse, interacting constraints.
Clinical deployment often requires satisfying statistical conditions like minimum subgroup size and overlap \citep{vanderweele2019selecting, crump2009dealing}, as well as operational and ethical considerations such as budget, safety, and fairness. 
Existing methods typically treat these constraints as post-hoc filters rather than integrating them into the optimization process. As a result, they \emph{struggle to jointly satisfy multiple constraints}, leading to instability and poor performance.
These challenges highlight a deeper disconnect between the continuous nature of CATE estimates and the discrete, constraint-driven structure of clinically actionable subgroup identification. This motivates the need for new frameworks that can incorporate and optimize over multiple real-world constraints in a unified and principled way.

We propose MOSIC (Model-agnostic Optimal Subgroup Identification with multi-Constraints), a \emph{novel optimization framework} that identifies subgroups with maximal CATE while satisfying group size and overlap constraints—with flexibility to incorporate additional constraints. Our approach addresses the challenge of nonconvex/nonconcave optimization, and our contributions are threefold:
%
\begin{itemize}[itemsep=0pt, topsep=2pt, partopsep=0pt, parsep=0pt, leftmargin=*]
\item \textbf{Problem Reformulation for Stable Solutions}:  
We develop a stable optimization procedure that: (1) formulates the task as a constrained problem 
(Section~\ref{subsec:framework}), 
(2) absorbs constraints into the objective via a reformulation
(Section~\ref{subsec:relax_formulation}), and 
(3) modifies the objective to improve stability and solves it using a gradient descent-ascent algorithm (Section~\ref{subsec:minmax_formulation}). Finally, we 
establish that the resulting solution is locally optimal and feasible (Section~\ref{subsec:modified_minmax}).

\item \textbf{Flexibility}:
MOSIC offers flexibility across three dimensions: (1) supporting multiple subgroup model architectures (e.g., multilayer perceptrons, decision trees) for different performance-interpretability tradeoffs, (2) compatibility with various ATE estimators, and (3) extensibility to diverse clinical constraints beyond our focus on size and overlap.
\item \textbf{Comprehensive Evaluation}: We extensively evaluate our framework on both synthetic and real-world data, demonstrating its effectiveness in optimal subgroup identification under multiple constraints (Section~\ref{sec:experiments}). Our implementation is available at
\url{https://github.com/Wenxin-Elmon-Chen/MOSIC-KDD}. Supplementary appendix is available at \url{https://github.com/Wenxin-Elmon-Chen/MOSIC-KDD/blob/main/supplementary/appendix.pdf}.
\end{itemize}

\vspace{-10pt}
\section{Related Work}\label{sec:lit}
We review treatment effect estimation, overlap handling, subgroup identification, and constrained optimization, highlighting that multi-constraint subgroup identification remains under-explored.

\textbf{\emph{Treatment Effect Estimation}.}\;
MOSIC accommodates various average treatment effect (ATE) estimators. Traditional methods like IPTW, meta-learners~\citep{kunzel2019metalearners}, R-learner~\citep{nie2021quasi}, BART~\citep{chipman2010bart} rely on either the treatment or outcome model, making them sensitive to model misspecification. In contrast, doubly robust estimators such as AIPTW~\citep{robins1995analysis}, DR-learner~\citep{kennedy2023towards}, TMLE~\citep{van2006targeted}, and DML~\citep{chernozhukov2018double} require only one model (treatment or outcome) to be correctly specified. This paper adopts AIPTW for illustration.

While ATE captures population-level effects, CATE estimation enables subgroup-specific recommendations.
Modern methods include 1) tree-based methods like Causal Tree (CT)~\citep{athey2016recursive}, Causal Forest (CF)~\citep{wager2018estimation}, and 2) neural-network-based approaches, such as TARNet~\citep{shalit2017estimating} and Dragonnet~\citep{shi2019adapting}.
In our AIPTW estimator, we estimate the outcome model with Dragonnet and the treatment model with Logistic Regression (LR).

\textbf{\emph{Dealing with Limited Overlap}.}\;
Limited sample overlap can bias treatment effect estimates or inflate variance.
Common solutions include truncating propensity scores \citep{10.1093/aje/kwac087, cole2008constructing} and excluding
low-overlap units
\citep{crump2009dealing, 10.5555/3666122.3669092, kallus2020efficientpolicylearningoptimal, li2018balancing}. 
We adopt the latter approach by incorporating a set of constraints to avoid low-overlap regions in subgroup identification. This ensures more reliable ATE estimation within the identified subgroup.

\textbf{\emph{Optimal Subgroup Identification}.}\;
Existing methods fall into three categories: (i) baseline methods without interpretability or constraints, (ii) \emph{interpretable} methods, and (iii) \emph{constrained} methods.

\emph{Baseline methods}
either: (1)
rank patients by estimated CATE values \citep{cai2011analysis, vanderweele2019selecting} (we benchmarked this approach employing CT, CF, and Dragonnet 
in Section~\ref{sec:experiments}) and DR-learner, R-learner, BART in 
Appendix E.4, or (2)
optimize individual treatment rules using methods like
Outcome-Weighted Learning (OWL)~\citep{zhao2012estimating, liu2018augmented}.
These methods provide useful benchmarks but do not readily accommodate additional constraints, a limitation that our approach addresses.

\emph{Interpretable methods} often rely on 
\emph{Decision Tree} (DT) \citep{lipkovich2011subgroup, dusseldorp2016quint, huang2017patient,athey2021policy}.
A representative example is
Virtual Twins (VT) \citep{foster2011subgroup}. It first estimates CATE and then applies a DT for interpretable subgroup identification.
Beyond trees, rule learning approaches have been adopted~\cite{wang2021causalrulesetsidentifying, 10.1145/3637528.3671951}. However, these methods rely on combinatorial searches and do not scale. Our method can instead leverage DTs as the backbone model, achieving the same level of interpretability while remaining scalable. We compare its performance against other DT-based methods in Section~\ref{subsec:dt_based}.

\emph{Constrained methods} explicitly incorporate constraints into subgroup search.
CAPITAL~\citep{cai2022capital} is the most closely related approach to ours: it maximizes subgroup size 
under a single constraint on subgroup ATE and allows extension to one additional constraint via
Lagrangian relaxation. However, it struggles with multiple constraints due to instability and a lack of feasibility guarantees.
In Section~\ref{sec:experiments}, we compare our approach with VT, OWL, and CAPITAL.

Related ideas also appear in constrained policy learning, where the goal is to optimize policies subject to explicit safety and budget constraints. 
Examples include constrained policy optimization~\citep{pmlr-v70-achiam17a, pmlr-v162-polosky22a}, contextual bandits with knapsacks~\citep{pmlr-v162-sivakumar22a, 10.1145/3164539}, and safe RL~\citep{garcia2015comprehensive, 10444044}. 
These methods generally impose trajectory-level cumulative cost or risk constraints in sequential decision-making settings. In contrast, MOSIC addresses a different class of structural constraints aimed at improving the reliability of a learned subgroup rule in the static setting, such as the overlap and group-size constraints. It can additionally accommodate linear and ratio-form constraints such as risk and budget restrictions.

\textbf{\emph{Constrained Optimization}.}\; 
Constrained optimization algorithms 
vary by problem convexity. For \emph{convex objective and constraints},
classical methods such as Projected Gradient Descent (PGD), Frank-Wolfe (FW), Interior Point Methods (IPM), and Lagrangian-based methods (e.g., Alternating Direction Method of Multipliers, ADMM~\citep{boyd2011distributed}) are effective. 
For \emph{non-convex objectives with convex feasible regions}, global optimality is NP-hard, and convergence guarantees are typically weakened
~\citep{lacoste2016convergence,wang2019global}. 
Our setting—\emph{non-convex objective with non-convex constraints}—poses greater challenges:
PGD struggles with complex projection, IPM scales poorly with constraint count, FW assumes convexity, and Lagrangian methods often lack stability and feasibility guarantees. ADMM fails here due to the non-separable objective. While ADMM variants exist for structured problems~\cite{gao2020admm}, none apply to our setting.
In contrast, our method guarantees both constraint feasibility and local optimality, critical for real-world deployment.

\vspace{-5pt}
\section{Problem Setting}\label{sec:problem_setting}
We consider an observational dataset with baseline covariates $\bm{X} \in \mathcal{X} \subseteq \mathds{R}^d$ (e.g., severity scores, laboratory measurements, and comorbidity), binary treatment assignment $A \in \mathcal{A} = \{0,1\}$ (e.g., corticosteroids), and clinical outcome $Y \in \mathcal{Y}$, where $Y$ may be binary (e.g., mortality) or continuous. The dataset consists of $n$ samples ${(\bm{x}_i, a_i, y_i)}_{i=1}^n$. 

A subgroup in our setting is a subset of patients defined by $X$, such as "patients with high inflammation biomarkers and moderate respiratory severity". Our goal is to identify a subgroup in which the treatment effect is favorable, while ensuring the subgroup is sufficiently large and has adequate treatment-control overlap in the observed data so that the estimated ATE is reliable.

Let $Y(0)$ and $Y(1)$ denote the potential outcomes under control and treatment. The propensity score is denoted as $e(\bm{x}) = P(A = 1 | \bm{X} = \bm{x})$, and potential outcomes as $\mu_a(\bm{x}) = \mathds{E}[Y | \bm{X} = \bm{x}, A = a]$.
Throughout this work, we adopt standard causal inference assumptions: 1) Stable Unit Treatment Value Assumption (SUTVA): $Y = A \cdot Y(1) + (1 - A)\cdot Y(0)$; 2) Unconfoundedness: $\{Y(0),Y(1)\} \perp A|\bm{X}$; 3) Overlap: $0 < e(\bm{x}) < 1, \forall \bm{x} \in \mathcal{X}$.

\vspace{-5pt}
\subsection{Reliable Subgroup Identification Framework}\label{subsec:framework}
We define a subgroup identification model: $S: R^d\mapsto\{0,1\}$, which assigns a patient with covariates X=x to the subgroup ($S(x)=1$) or excludes them ($S(x)=0$). The target estimand is subgroup ATE: $\mathds{E}[Y(1)-Y(0)|S(X)=1]$,  which reads as "the mean difference of potential outcome under treatment and that under control, within the selected subgroup". 
Under SUTVA and unconfoundedness, this estimand can be expressed as
$$\mathds{E}\left[ \mathds{E}[Y|A=1,\bm{X}] - \mathds{E}[Y|A=0,\bm{X}] | \tilde{S}(\bm{X}) = 1\right].$$


Let $\hat\phi(\bm{x}_i, a_i, y_i)$ denote the estimated pseudo-outcomes for each sample, and $\mathds{1}(\cdot)$ as the indicator function, a general estimator of the subgroup ATE is then
\begin{equation*}
    \frac{\sum_{i=1}^n\mathds{1}(\tilde{S}(\bm{X})=1)\hat{\phi}(\bm{x}_i, a_i, y_i)}{\sum_{i=1}^n\mathds{1}(\tilde{S}(\bm{X})=1)}.
\end{equation*}

Our goal is to identify a subgroup such that its subgroup ATE is maximized, while ensuring the following constraints:
\begin{itemize}[noitemsep, topsep=0pt, leftmargin=2em]
    \item \textbf{Minimum size requirement.} 
    A sufficiently large subgroup is essential for reliable estimation, robust statistical power, and economic viability in applications like drug repurposing.
   
    \item \textbf{Each sample in the identified subgroup lies in a region with sufficient overlap.} Intuitively, the propensity score $e(\bm{x})$ represents the treatment assignment probability given covariates. Patients with $e(\bm{x})$ close to 0 or 1 lie in low-overlap regions, meaning that either one of the treatment arms is rarely observed for patients with $\bm X=\bm x$, making the effect estimation and subsequent subgroup identification unreliable. For this reason, we impose that each sample in the selected subgroup has a propensity score bounded away from 0 and 1, similar to excluding samples with extreme propensity scores \cite{crump2009dealing}.


\end{itemize} 
\noindent The above task can be formally stated as follows:
\begin{problem}
\begin{align}\label{eq:raw_problem}
    \underset{\tilde{S}}{\min}  \;&-\frac{\sum_{i=1}^n\mathds{1}(\tilde{S}(\bm{X})=1)\hat{\phi}(\bm{x}_i, a_i, y_i)}{\sum_{i=1}^n\mathds{1}(\tilde{S}(\bm{X})=1)} \problemtag\\
     \text{s.t.}&\quad \mathds{E}\left[\tilde{S}(\bm{X})\right] \geq c\notag \\
     &\quad \alpha \leq e(\bm{x}) \leq 1 - \alpha, \forall \bm{x}: \tilde{S}(\bm{x}) = 1, \label{eq:propensity_constraint}
\end{align}
\end{problem}
\noindent where $c \in (0, 1)$ is the desired subgroup size, and $\alpha \in [0, 0.5)$ is the threshold controlling the overlap constraint. Beyond the size and overlap, our method naturally accommodates more general \textbf{linear and ratio-form constraints}, which we formally introduce in Lemma~\ref{main_lemma} and Remark~\ref{remark:ratio_constraint} (Section~\ref{subsec:modified_minmax}).

\vspace{-5pt}
\subsection{Relaxing the Combinatorial Formulation}\label{subsec:relax_formulation}
Since \ref{eq:raw_problem} is
combinatorial and difficult to optimize, we relax the subgroup identification to a probabilistic assignment. This relaxation is implemented using a parametric surrogate model $S: \mathds{R}^d \mapsto (0,1)$ with parameters $\bm{\theta}$.
The ATE on the identified subgroup can be expressed as
\begin{equation}\label{eq:subgroup_ate}
  f(\bm{\theta}) = \frac{\sum_{i=1}^n S(\bm{x}_i;\bm{\theta}) \hat\phi(\bm{x}_i, a_i, y_i)}{\sum_{i=1}^n S(\bm{x}_i;\bm{\theta})}.\footnote{$\hat\phi(\bm{x}_i, a_i, y_i)$ does not depend on parameter $\bm{\theta}$ as the estimation problem is separate from the parametric surrogate model $S$.}  
\end{equation}
Let $\hat e(\bm{x}_i)$, $\hat \mu_{a}(\bm{x}_i)$ denote the estimated $e(\bm{x}_i)$ and $\mu_{a}(\bm{x}_i)$. We then adopt the AIPTW estimator,
which can be expressed as functions of
$\hat e(\bm{x}_i)$ and $\hat \mu_{a}(\bm{x}_i)$:
\begin{align}
    \hat\phi_\mathrm{aiptw}(\bm{x}_i, a_i, y_i) &= \hat \mu_1(\bm{x}_i) - \hat \mu_0(\bm{x}_i)\\
    &+ \frac{a_i}{\hat e(\bm{x}_i)}(y_i - \hat \mu_1(\bm{x}_i)) \notag - \frac{1-a_i}{1-\hat e(\bm{x}_i)}(y_i - \hat \mu_0(\bm{x}_i)).\label{eq:aiptw}
\end{align}

While this study focuses on AIPTW, $\hat\phi(\bm{x}_i, a_i, y_i)$ can be derived using other ATE estimators, making this formulation flexible. More details are illustrated in 
Appendix C.1.

Due to the relaxation of subgroup identification into a probabilistic assignment, the overlap constraint in \eqref{eq:propensity_constraint}, originally designed for discrete subgroup selection, must be adapted.
To achieve this, we introduce $h(\bm{x}_i, \alpha)$, a surrogate function that reformulates the overlap constraint from \ref{eq:raw_problem}:
\begin{equation}
h(\bm{x}_i, \alpha) = 1 - \frac{\hat e(\bm{x}_i)(1 - \hat e(\bm{x}_i))}{\alpha (1 - \alpha)}.\label{eq:overlap}
\end{equation}

The following result, Lemma~\ref{lemma:surrogate_overlap} (proof in 
Appendix B.1), establishes that the overlap constraint in \ref{eq:raw_problem} can be replaced by a constraint on $h(\bm{x}_i, \alpha)$:

\begin{restatable}{lemma}{LemmaSurrogateOverlap}\label{lemma:surrogate_overlap}
$S(\bm{x}_i;\bm{\theta})h(\bm{x}_i,\alpha) \leq 0$ if and only if $\alpha \leq \hat e(\bm{x}_i) \leq 1- \alpha$.
\end{restatable}

With Lemma~\ref{lemma:surrogate_overlap}, our optimization can be reformulated as:
\begin{problem}
\begin{align*}\label{prob:surrogate}
    \underset{\bm{\theta}}{\min} \quad & -f(\bm{\theta})  \problemtag\\
    \text{s.t. }\; & \frac{1}{n}\sum_{i=1}^n S(\bm{x}_i; \bm{\theta}) \geq c,
    \quad S(\bm{x}_i;\bm{\theta})h(\bm{x}_i, \alpha) \leq 0, \ \forall i.
\end{align*}
\end{problem}

Solving \ref{prob:surrogate} remains a significant challenge.
As shown in \eqref{eq:subgroup_ate}, the parametric model $S(\bm{x_i};\bm{\theta})$ appears in both the numerator and denominator, making the objective function $f(\bm{\theta})$ neither convex nor concave, even for simple models like logistic regression. This nonconvexity also extends to the feasible set, rendering standard convex optimization methods, such as ADMM and FW, unsuitable. 
While Lagrangian relaxation could in principle be applied, doing so would require tuning a separate multiplier for each constraint, making the hyperparameter tuning process impractical at scale. 
To address this, we reformulate \ref{prob:surrogate} and present the final framework in Section~\ref{sec:opt_methods}. 

\vspace{-5pt}
\section{Optimization Methods}\label{sec:opt_methods}
Section~\ref{subsec:minmax_formulation} reformulates the task as a min-max optimization and adopts the \emph{$\gamma$-Gradient Descent Ascent} ($\gamma$-GDA) algorithm~\citep{10.5555/3666122.3669092}.
Section~\ref{subsec:modified_minmax} establishes feasibility guarantees, showing that MOSIC can identify the optimal subgroup while satisfying multiple constraints.

\vspace{-5pt}
\subsection{Min-Max Formulation and GDA}\label{subsec:minmax_formulation}

Since neither the objective nor the feasible region in \ref{prob:surrogate} is convex,
we rewrite it using the saddle-point formulation:
    \begin{equation}\label{eq:partially_updated_obj}
        L(\theta,\lambda):= -f(\bm{\theta}) + \bm{\lambda}^T \bm{g}(\bm{\theta}),\;\;\underset{\bm{\theta}}{\min} \;\underset{\bm{\lambda}\geq 0}{\max} \; L(\theta,\lambda).
    \end{equation}
We note that this is solving the primal constrained problem directly through a \textit{min-max Lagrangian formulation}, unlike Lagrangian relaxation, which operates on the dual problem. We can solve this problem by $\gamma$-GDA~\citep{pmlr-v119-jin20e}, as described in Algorithm~\ref{alg:GDA_algo}. The solution to Eq.~\eqref{eq:partially_updated_obj} is equivalent to solving \ref{prob:surrogate} \citep{boyd2004convex}. 


While the saddle-point formulation provides a correct representation of \ref{prob:surrogate}, applying standard GDA can lead to numerical instability and hinder convergence.
Consider a scenario with a single group size constraint ($c - \frac{1}{n}\sum_{i=1}^n S(\bm{x}_i; \bm{\theta})$) and its Lagrange multiplier $\lambda_{n+1}$.
Initially, the constraint is violated because the group size is below the threshold,
leading to a positive gradient 
that increments $\lambda_{n+1}$ at each iteration.
Once the constraint is met, the gradient becomes negative, driving $\lambda_{n+1}$ toward $0$. 
However, due to the small learning rate, $\lambda_{n+1}$ remains nonzero in subsequent iterations, continuing to penalize the objective and leading to artificially shrunk feasible regions. 
This instability extends to multiple constraints, causing the algorithm to oscillate between a strictly feasible but suboptimal $\bm{\theta}$ and an optimal but infeasible $\bm{\theta}$, undermining convergence.
To obtain stable and convergent $\gamma$-GDA dynamics, additional structural conditions are needed. We thus refine the objective to satisfy two key properties:
\begin{itemize}[noitemsep, topsep=0pt, leftmargin=2em]
    \item Only violated constraints contribute gradients. This can be achieved by introducing a ReLU transform on the constraint vector. When $g_i(\bm{\theta})\leq 0$ (i.e., the constraint is satisfied), the penalty term becomes zero, so satisfied constraints no longer affect the optimization dynamics.
    \item Ensuring convergence to a feasible, locally optimal solution.
\end{itemize}
For the second requirement, we begin by defining local optimality in the context of a nonconvex-nonconcave min-max problem, introducing the concept of \emph{local minmax point} 
(Definition 1). 
Informally, it is a fixed point where the objective remains stable under small perturbations in $\bm{\theta}$ (the parameters over which we minimize) and worst-case perturbations in $\bm{\lambda}$ (the parameters over which we maximize). 
Theorem 1~\citep{pmlr-v119-jin20e} states that if the min-max objective function is twice differentiable, then the $\gamma$-GDA algorithm, upon convergence, reaches either a local minmax point or a stationary point with a degenerate Hessian. 


However, the Hessian of our objective in \eqref{eq:partially_updated_obj} is inherently degenerate:
$\nabla^2_{\bm{\lambda}\bm{\lambda}}L(\bm{\theta},\bm{\lambda}) = \bm{0}$, which hinders guarantees of local optimum convergence and feasibility.
Unlike \citet{nandwani2019primal}, who did not exclude degenerate points--potentially invalidating their results--we introduce a modification that eliminates degenerate points, yielding the final objective of MOSIC:
\begin{problem}
    \begin{align*}\label{prob:minmax}
        &\underset{\bm{\theta}}{\min} \;\underset{\bm{\lambda}\geq 0}{\max} \; L(\bm{\lambda},\bm{\theta}),\problemtag\\
        L(\bm{\lambda},\bm{\theta}) = -f(\bm{\theta}) &+ \bm{\lambda}^T ReLU(\bm{g}(\bm{\theta})) - \frac{\beta}{2}\bm{\lambda}^2,
    \end{align*}
\end{problem}
\noindent where
$\bm{\lambda} \in \mathbb{R}^{n+1}_+,$
$\bm{g}(\bm{\theta})=(S(\bm{x_1};\bm{\theta})h(x_1;\alpha),\dots, S(\bm{x_n};\bm{\theta})h(\bm{x_n};\alpha), c - \frac{1}{n}\sum_{i=1}^n S(\bm{x_i};\bm{\theta})^\top$.
For notational convenience, we write $\bm{g}(\bm{\theta})$ without explicitly indicating its dependence on fixed constraint values $c$ and $\alpha$ (or those for any additional constraint, if present). 

\begin{algorithm}[t]
\caption{$\gamma$-Gradient Descent Ascent ($\gamma$-GDA)}\label{alg:GDA_algo}
\begin{algorithmic}[1]
\STATE \textbf{Input:} step size $\eta$,  decay rate $\zeta$, objective function $L(\bm{\theta},\bm{\lambda})$. 
\STATE Initialize $\bm{\lambda}_0= 0$; Initialize $\bm{\theta}_0$ randomly
\FOR{$t = 0, 1, \ldots T$}
    \STATE If converged, output $\bm{\theta}^*=\bm{\theta}_t$
    \STATE $\gamma \leftarrow (1 + t)^\zeta$
    \STATE Update $\bm{\theta}_t$ using gradient descent with learning rate $\eta/\gamma$:
    $\bm{\theta}_{t+1} \leftarrow \bm{\theta}_t - \left(\frac{\eta}{\gamma}\right) \nabla_{\bm{\theta}} L(\bm{\theta}_t, \lambda_t).$
    \STATE Update $\bm{\lambda}_t$ using gradient ascent with learning rate $\eta$:
$\bm{\lambda}_{t+1} \leftarrow \bm{\lambda}_t + \eta \nabla_{\bm{\lambda}} L(\bm{\theta}_{t+1}, \bm{\lambda}_t).$
\ENDFOR
\STATE \textbf{Output:} $\bm{\theta}^T$
\end{algorithmic}
\end{algorithm}

\vspace{-5pt}
\subsection{Feasibility Guarantees}
\label{subsec:modified_minmax}

With \ref{prob:minmax}, we establish that if Algorithm~\ref{alg:GDA_algo} converges, it reaches a \textit{strict local minmax point} (Proof in 
Appendix B.2): 

\begin{restatable}[Local Optimality]{proposition}{LocalOptimality}\label{local_minmax}
    Let $(\bm{\theta}',\bm{\lambda}')$ be a \emph{linearly stable point} (formally defined in 
    Definition 2)
    of Algorithm~\ref{alg:GDA_algo}. Then,  $(\bm{\theta}',\bm{\lambda}')$ must be a \emph{strict local minmax point}.
\end{restatable}

Building on this, 
the following result, Lemma~\ref{main_lemma}, shows that upon convergence, all constraints are approximately satisfied within a small tolerance. This guarantee applies not only to the group size constraint but also to general linear constraints in $S$ (though not linear in $\bm{\theta}$), which include the overlap constraint.

\begin{restatable}{lemma}{MainLemma}\label{main_lemma}
Suppose the constraints include a group size constraint,  $g^{\text{size}}(\bm{\theta}) = c - \frac{1}{n} \sum_{i=1}^n S(x_i; \bm{\theta}),$ and $K$ ($K\geq 0$) additional constraints linear in $S$, $g^k(\bm{\theta}) = a^k + \sum_{i=1}^n b^k_i S(x_i; \bm{\theta}), $ where $a^k, b^k_i \in \mathbb{R}$, and,  $\forall k, |\sum_{i=1}^n b^k_i|>0$.
Together, they define the constraint vector $\bm g(\bm{\theta})=(g^1(\bm{\theta}),\dots,g^K(\bm{\theta}),g^{\text{size}}(\bm{\theta}))^\top$.

Define $\xi>0$ as the tolerance, $\phi_{\max}=\max_i\hat\phi(x_i,a_i,y_i)$,  
$L=\sup|\partial S(\cdot;\bm{\theta})/\partial \theta_j|$ as the coordinate-wise Lipschitz constant,  
and $\mu_\Delta=\E[\partial S(x_i;\bm{\theta})/\partial \theta_j]$ for $j=\arg\max_j |\mu_\Delta|/L$.

Let $(\bm{\theta}^*, \bm{\lambda}^*)$ be a strict local min-max point obtained by Algorithm~\ref{alg:GDA_algo}. If 
$$
\beta < \frac{\xi(c - \xi)|\mu_\Delta|}{2\, \phi_{\max} L},
$$
then either the model collapses ($\frac{1}{n} \sum_{i=1}^n S(x_i; \bm{\theta}^*) < \xi$) or all constraints are approximately satisfied: 

    \begin{align*}
        &g^{\text{size}}(\bm{\theta}^*) \leq \xi, \text{and furthermore, with probability at least } 1-\delta, \\
        &g^k(\bm{\theta}^*) \leq \frac{\xi}{|\sum_{i=1}^nb_i^k|(1+\frac{L}{|\mu_{\Delta}|\sqrt{n}}\sqrt{\log\frac{2}{\delta}})}, \forall k\in [K]
    \end{align*}
    
\end{restatable}

The proof of Lemma~\ref{main_lemma} 
(Appendix B.3) 
analyzes each constraint at \textit{strict local minmax points} and establishes that, with $\beta$ properly chosen,\footnote{
Constraint violation is insensitive to $\beta$; $\beta \in [10^{-5}, 10^{-2}]$ works well.}
the constraints are either fully satisfied or violated up to a small tolerance error, if the model does not collapse. The proof further shows that, as long as the feasible region is non-negligible, the model is unlikely to collapse.\footnote{If the model collapses to a near-zero subgroup, we restart with a new seed; persistent collapse indicates the constraint-defined feasible region should be re-evaluated.}

\begin{remark}[Implication on the overlap constraint]
When $n \rightarrow \infty$, the bound on constraint violation is governed by $|\sum_{i=1}^n b_i^k|$. 
For the overlap constraint on sample $j$, this term reduces to $h(x_j;\alpha)$ since $b^k_i = \mathds{1}(i = j)h(x_j;\alpha)$.
When $h(\bm{x_j};\alpha)$ is small, this denominator inflates the bound, suggesting potentially high violation. However, a small $h(\bm{x_j};\alpha)$ indicates that the corresponding violation of the overlap condition is itself negligible. Thus, such constraints can be safely ignored in practice.
\end{remark}

\begin{remark}[Extension to ratio constraints]\label{remark:ratio_constraint}
Because the model maximizes the subgroup ATE, it naturally favors smaller subgroups by excluding samples with relatively low estimated CATE, while the group size constraint enforces a group size lower bound $c$. As a result, the group size typically converges to the size threshold $c$, i.e., $\sum_{i=1}^n S(x_i;\bm{\theta})\approx c$ at termination. This observation allows us to extend the linear constraint to ratio constraints
$$g^k(\bm{\theta})=a^k + \frac{\sum_{i=1}^nb_i^kS(\bm{x_i};\bm{\theta})}{\sum_{i=1}^nS(\bm{x_i};\bm{\theta})}.$$
To retain compatibility with Lemma~\ref{main_lemma}, we block gradient flow through the denominator during backpropagation, effectively treating it as a constant.
\end{remark}

The linear and ratio families capture many practical constraints relevant to healthcare applications. Linear constraints include the overlap constraint and budget constraints (each patient incurs a treatment cost under limited resources). Ratio constraints cover safety constraints (e.g., risk levels associated with patients) and certain fairness constraints (e.g., conditional statistical parity metric). Section~\ref{subsec:ablation_more_constraint} assesses MOSIC's extendability to these constraints. Due to the nonconvexity of both the objective and the feasible region, extending the guarantee to more complex constraints remains an open direction for future work.
\begin{algorithm}[t]
\caption{MOSIC}\label{alg:overall}
\begin{algorithmic}[1]
\STATE \textbf{Input:} $\{(\bm{x_i}, a_i, y_i )\}_{i=1}^n$, constraint-related values $c$ and $\alpha$, learning rate $\eta$, decay rate $\zeta$
\STATE Estimate $\hat \mu_{0}(\bm{X})$, $\hat \mu_{1}(\bm{X})$, $\hat e(\bm{X})$ \textit{\% Estimate nuisance functions}
\STATE Compute $\hat\phi_(\bm{x}_i, a_i, y_i)$ using nuisance functions \textit{\% Pseudo-outcomes}
\STATE Construct $L(\bm{\theta};\bm{\lambda};c,\alpha)$
\STATE $\bm{\theta}^* \leftarrow \gamma$-GDA($\eta$, $\zeta$, $L(\bm{\theta};\bm{\lambda};c,\alpha)$) 
\textit{\% Solve \ref{prob:minmax} using Algorithm~\ref{alg:GDA_algo}}
\STATE \textbf{Output:} Parametric surrogate model $S(\bm{X};\bm{\theta}^*)$
\end{algorithmic}
\end{algorithm}

\begin{figure*}[t]
    \centering
    \subfigure[Unconfounded($\tilde{\omega}=0$)]{
        \includegraphics[width=0.48\linewidth]{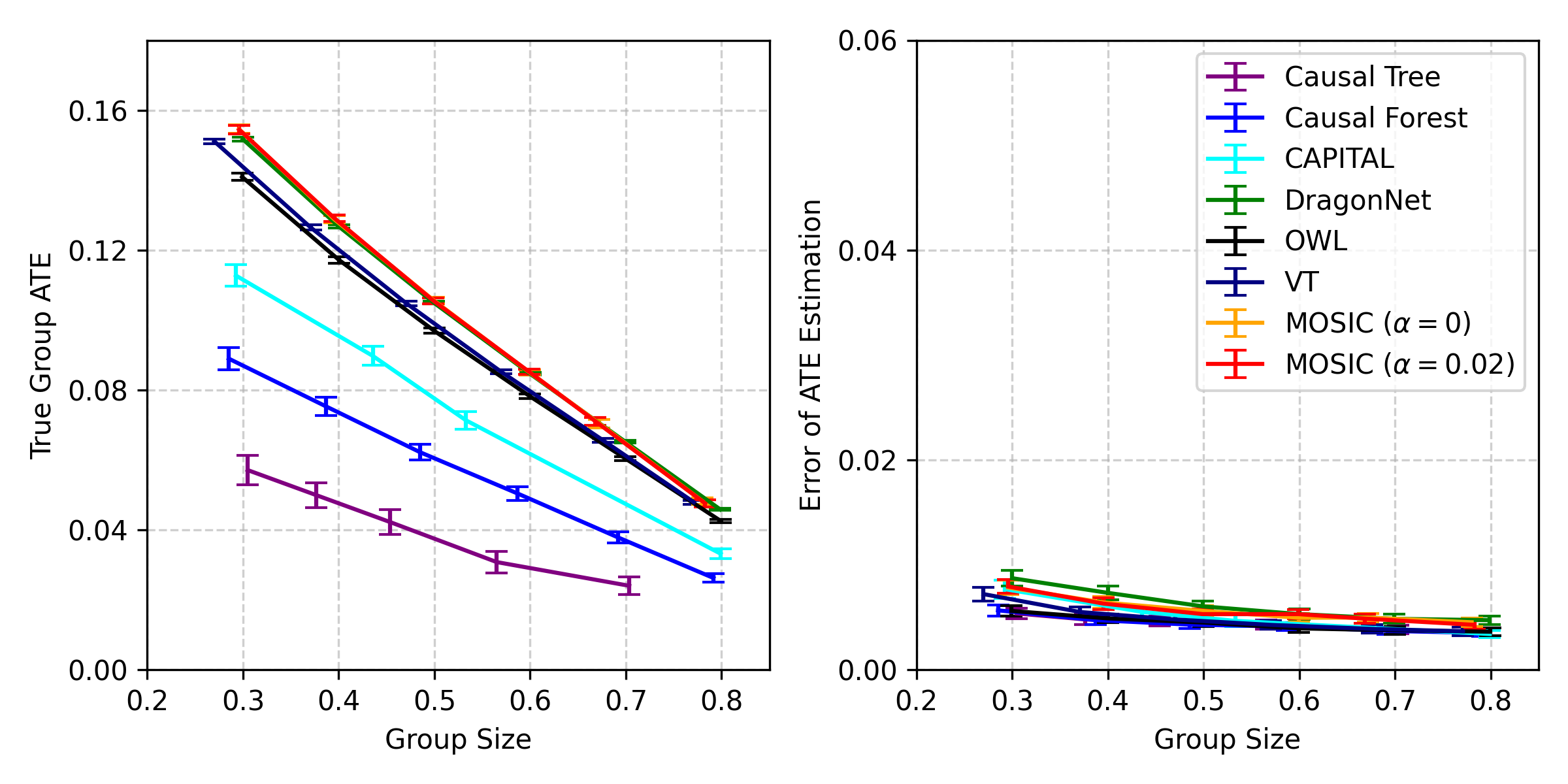}\label{fig:syn-wo-cs}
    }
    \subfigure[Confounded($\tilde{\omega}=5$)]{
        \includegraphics[width=0.48\linewidth]{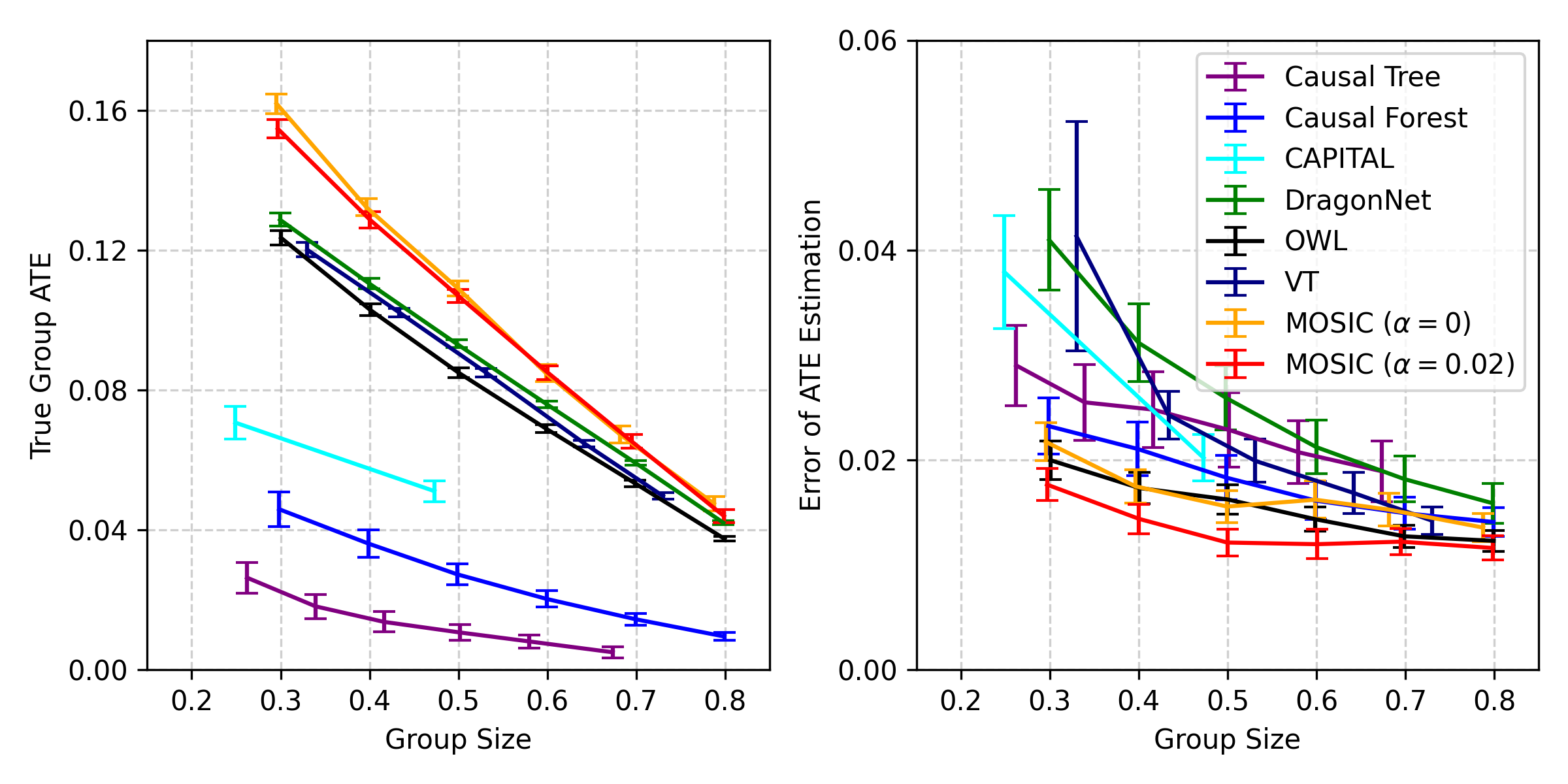}\label{fig:syn-w-cs}
    }\hfill
    \vspace{-5pt}
    \caption{True ATE and estimation error across different group sizes on synthetic data}
    \vspace{-5pt}
    \label{fig:syn}
\end{figure*}

Finally, we summarize our overall framework, MOSIC, in Algorithm~\ref{alg:overall}. It first estimates the nuisances and computes the pseudo-outcomes for each sample, which are then fed into the objective function (\ref{prob:minmax}). The optimization is subsequently performed using the $\gamma$-GDA algorithm.

\vspace{-5pt}
\section{Experiments}\label{sec:experiments}

We evaluate \methodname\ on both synthetic and real-world data. Section~\ref{subsec:setting} outlines the experimental setup. 
Section~\ref{subsec:results} demonstrates that MOSIC achieves high ATE while maintaining comparable covariate balance, or vice versa.
Section~\ref{subsec:ablation} further examines the robustness, interpretability, extensibility, and computational properties of MOSIC.

\vspace{-5pt}
\subsection{Setup}\label{subsec:setting}
\paragraph{Datasets} We evaluate MOSIC on synthetic and real-world datasets:
\begin{enumerate}[noitemsep, topsep=0pt, leftmargin=2em]
    \item 
    We generate synthetic data following a procedure adapted from \citet{assaad2021counterfactual} (details in 
    Appendix C.2). 
    We introduce an \emph{imbalance parameter} $\tilde{\omega} \geq 0$ to determine the strength of confounding bias.
In particular, we generated two datasets of size $n=5{,}000$ with $d=10$ covariates and the continuous outcome $Y$: (1) one with no confounding bias ($\tilde{\omega} = 0$) and (2) one with high confounding bias ($\tilde{\omega} = 5$).
    \item 
    We use two de-identified datasets from intensive care units (ICU): 
    The eICU dataset ($n=13{,}361$, $d=23$ covariates)~\citep{pollard2018eicu} and the MIMIC-IV ($n=6{,}516$, same covariates)~\citep{johnson2023mimic}. In both datasets, treated patients ($A = 1$) received an initial Glucocorticoids dose of 160mg within 10 hours before to 24 hours after ICU admission. The outcome $Y$ represents 7-day survival, with $Y=1$ indicating survival and $Y=0$ otherwise. The covariates $\bm{X}$ include lab test results, vital signs, and sequential organ failure assessment (SOFA) scores~\citep{vincent1996sofa}.
\end{enumerate}

\begin{figure*}[t]
    \centering
    \subfigure[eICU\label{fig:real-world-eicu}]{%
        \includegraphics[width=0.48\linewidth]{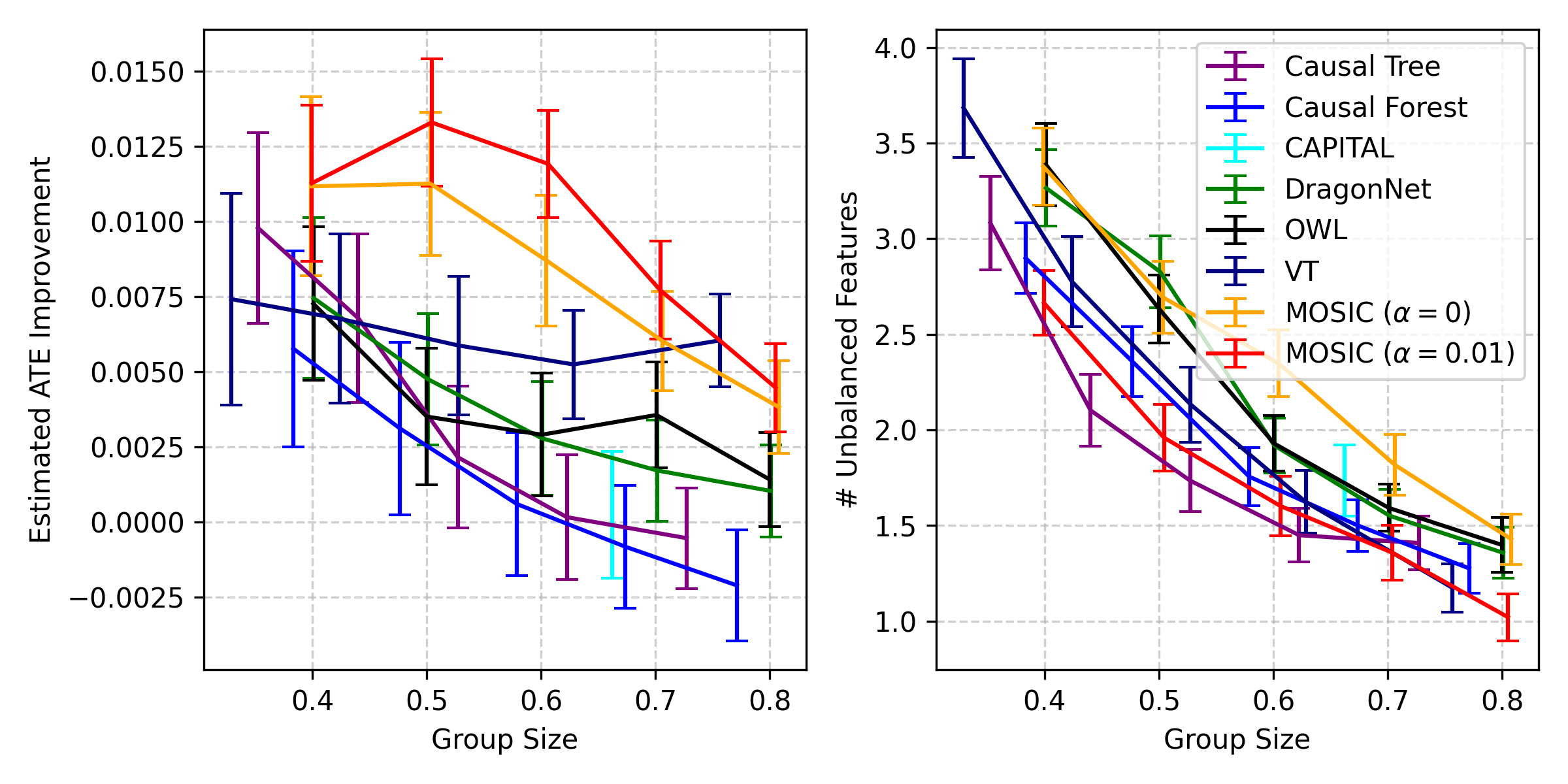}
    }\hfill
    \subfigure[MIMIC-IV\label{fig:real-world-mimic}]{%
        \includegraphics[width=0.48\linewidth]{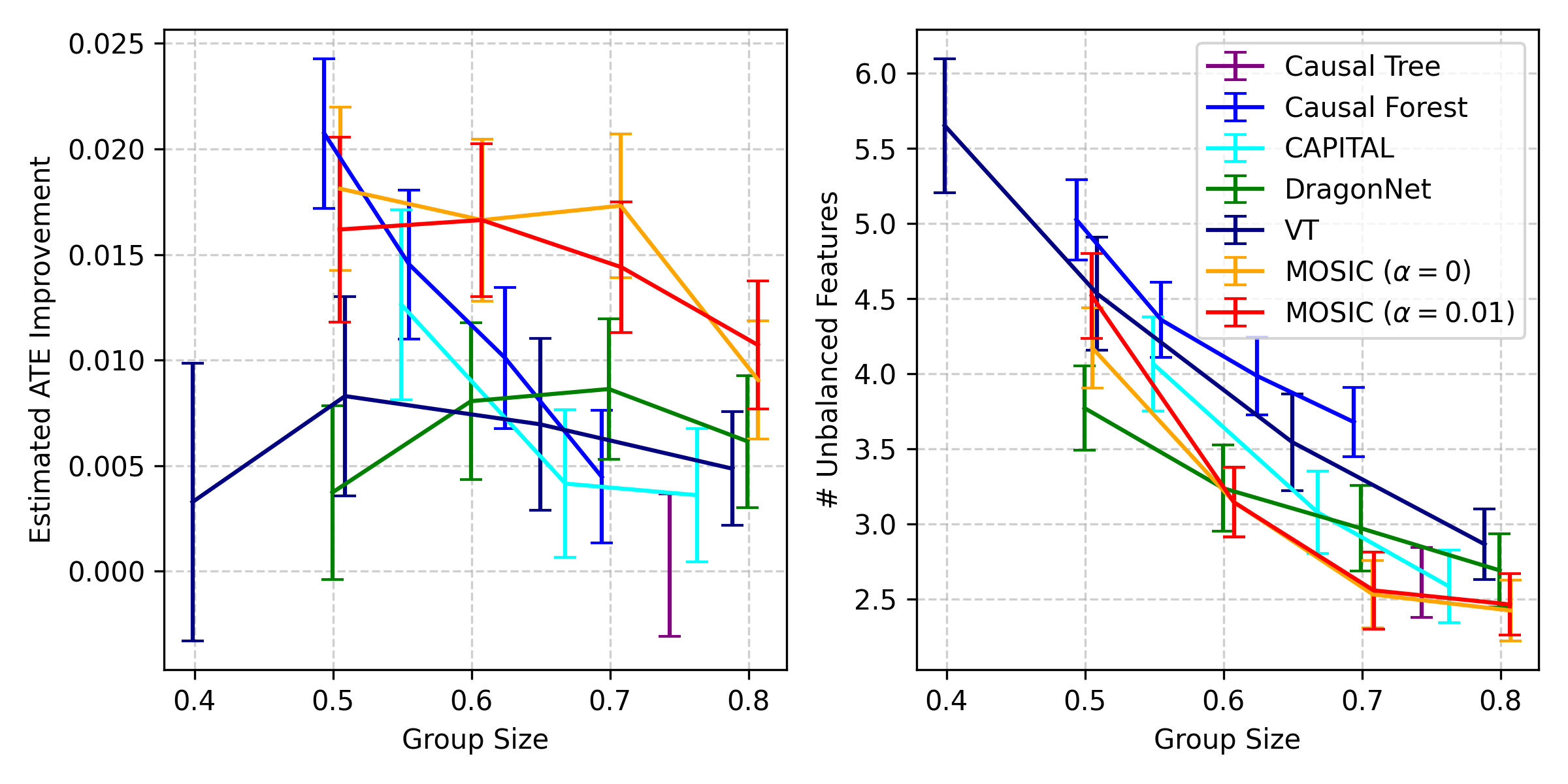}
    }
    \vspace{-5pt}
    \caption{Estimated ATE and the number of unbalanced features on real-world datasets.\protect\footnotemark}
    \vspace{-5pt}
    \label{fig:real-world}
\end{figure*}
\footnotetext{The OWL method failed to converge on the MIMIC-IV dataset and was therefore omitted from the figure.}

\paragraph{Baselines}
We compare \methodname\ with two categories from Section~\ref{sec:lit}: those designed for subgroup identification and those adapted from CATE estimation algorithms.

\begin{enumerate}[noitemsep, topsep=0pt, leftmargin=2em]
    \item Dedicated subgroup identification methods: We evaluate CAPITAL, OWL, and VT, modifying them to incorporate a group size constraint via thresholding (Details in 
    Appendix C.3). 
    These methods prioritize interpretability, making it unclear whether performance limits stem from this trade-off or poor CATE estimation. To address this, we consider the next category.
    \item CATE estimation algorithms: We evaluate three methods: CT, CF, and Dragonnet in the main results. We additionally compare with DR-learner, R-learner, BART, and an overlap-weighted variant of MOSIC (MOSIC-OW) in 
    Appendix E.4.
    In these baselines, patients are ranked by the estimated CATE values, and the top subgroup of the desired size is selected~\citep{vanderweele2019selecting}. While not designed for subgroup selection, they provide a natural baseline. If CATE estimation were accurate, this approach would identify the optimal subgroup, allowing us to separate the impact of CATE estimation reliability from the subgroup interpretability trade-off.
\end{enumerate}

\paragraph{Evaluation Metrics} 
We assess performance using two metrics: \textbf{(1) Subgroup ATE}, which measures whether identified subgroups achieve high ATE at the desired size. On synthetic data, we compute the ground‑truth ATE; on real-world data, we use the difference between the subgroup and overall AIPTW estimates. \textbf{(2) ATE Estimation Reliability}. On synthetic data, it is measured by the AIPTW estimation error. On real-world data, where the true ATE is unobserved, it is measured by the number of unbalanced features. 
A feature is considered unbalanced \citep{cohen2013statistical} if its standardized mean difference (SMD) $>0.2$ after IPTW reweighting~\citep{austin2009using, zang2023high}. More unbalanced features indicate greater estimation uncertainty and lower subgroup reliability. To validate imbalance as a proxy for estimation error, we also report the correlation of estimation error and feature imbalance on synthetic data 
(Appendix E.1). 
To assess how well MOSIC enforces the overlap constraint, we further measure the proportion of test-set samples that violate the overlap constraint on the real-world data.

\paragraph{Implementation Details}
For all datasets, we perform 100 random splits of the training and test sets. For each split, we first conduct a 5-fold cross-validation on the training set to determine the optimal hyperparameters 
(Appendix C.5). 
The model is then retrained on the entire training set using the selected hyperparameters and evaluated on the corresponding test set. The final results are reported using the mean and standard deviations across all 100 evaluations.

We implement the subgroup identification model ($S$) using MLP and DT. In the next sections, `\methodname' refers to \methodname-MLP, unless stated otherwise. For DTs, we adopt the neural-network representation of ~\cite{marton2024gradtree}. All these models are trained using Algorithm~\ref{alg:GDA_algo}, with L1 regularization applied in the loss function. For nuisance function estimation, we use LR for the propensity score model $\hat e(\cdot)$, and Dragonnet for the outcome models $\hat \mu_{a}(\cdot)$ 
(Appendix C.4).
\footnote{Empirically, estimating propensity scores with Dragonnet results in a large number of unbalanced features; we therefore adopt LR, demonstrating the flexibility of \methodname.}
Our implementation shares the training data for nuisance estimation and subgroup selection. But the final comparison on the test set is not affected by such sharing because the test set is never used for either step. 

\vspace{-5pt}
\subsection{Results}\label{subsec:results}

\paragraph{Synthetic Data}
To assess the impact of overlap constraints, we compare \methodname\ with ($\alpha = 0.02$) and without ($\alpha = 0$) them. 
Figure~\ref{fig:syn} demonstrates that \methodname\ consistently identifies subgroups with the highest true subgroup ATE across all group sizes (Figure~\ref{fig:syn-wo-cs} and~\ref{fig:syn-w-cs}, left);
and it achieves the lowest ATE estimation errors (Figure~\ref{fig:syn-wo-cs} and~\ref{fig:syn-w-cs}, right). 
We also investigate the statistical properties of MOSIC in 
Appendix F and G.

\paragraph{Real-World Data}
Since the real-world datasets contain a large portion of samples with propensities outside [0.05, 0.95] 
(Figure E.5), 
we relax the overlap constraint threshold to be $\alpha=0.01$. Figures~\ref{fig:real-world} and 
Figure E.4
demonstrate that MOSIC consistently outperforms other methods, and 
Table E.1
quantifies the statistical significance of the performance differences shown in Figure~\ref{fig:real-world}. We highlight the importance of jointly evaluating subgroup ATE and covariate balance when assessing performance. A large ATE alone is insufficient if covariate imbalance persists, as it may indicate unreliable estimates. As shown in Figures~\ref{fig:real-world-eicu} and \ref{fig:real-world-mimic} (left), \methodname\ achieves higher subgroup ATEs at comparable group sizes in most cases. Even when the ATE advantage is not statistically significant (e.g., MOSIC ($\alpha=0.01$) vs. CF at $c= 0.6$ on MIMIC, p = 0.68), MOSIC delivers significantly better covariate balance (p = 5.1E-04; Figures~\ref{fig:real-world-eicu} and \ref{fig:real-world-mimic}, right). 

To quantify the uncertainty of the ATE, we additionally compute 95\% confidence intervals for the subgroup ATEs and conduct a sensitivity analysis of unmeasured confounding 
(Appendix E.3). 
For completeness, feature imbalance with SMD$>$ 0.1 as the threshold is also presented in 
Appendix E.7.

Figure \ref{fig:overlap_evaluation_alpha} (right) shows that MOSIC with the overlap constraint successfully limits the number of test-set samples falling outside the allowable propensity range, demonstrating effective enforcement of the constraint (additional results are provided in 
Appendix E.6). 
This reduction in extreme-propensity samples leads to improved feature balance on eICU (Figures \ref{fig:real-world-eicu}, right). 

\paragraph{Alternative Outcome Definition.}
To evaluate robustness to the outcome specification, we additionally consider ventilation-free days within 7 days of ICU admission, assigning 0 ventilation-free days to patients who die within 7 days. 
Table~\ref{tab:vfd_results} shows that the subgroup ATEs identified by MOSIC and the baselines are largely comparable. However, MOSIC consistently achieves substantially better covariate balance, reducing the number of unbalanced features by 25--40\% across subgroup sizes. Additional balance diagnostics in 
Appendix E.8
confirm this trend. These results are consistent with our main findings and suggest that the benefit of MOSIC is robust to the choice of outcome definition.

\begin{table}[t]
\centering
\caption{Alternative outcome definition: ventilation-free days within 7 days.}
\label{tab:vfd_results}
\resizebox{\columnwidth}{!}{
\begin{tabular}{lccc|ccc}
\toprule
& \multicolumn{3}{c|}{Subgroup ATE}
& \multicolumn{3}{c}{\# Unbalanced Features} \\
\cmidrule(lr){2-4} \cmidrule(lr){5-7}
Method & $c=0.6$ & $c=0.7$ & $c=0.8$
& $c=0.6$ & $c=0.7$ & $c=0.8$ \\
\midrule
OWL
& 0.05$\pm$0.01 & 0.04$\pm$0.00 & 0.03$\pm$0.00
& 2.59$\pm$0.18 & 2.28$\pm$0.17 & 1.73$\pm$0.14 \\

VT
& 0.04$\pm$0.01 & 0.04$\pm$0.01 & 0.04$\pm$0.01
& 2.67$\pm$0.19 & 2.20$\pm$0.17 & 1.88$\pm$0.16 \\

CT
& 0.04$\pm$0.00 & 0.04$\pm$0.00 & 0.04$\pm$0.00
& 3.06$\pm$0.23 & 2.23$\pm$0.19 & 1.74$\pm$0.15 \\

CF
& 0.04$\pm$0.00 & 0.03$\pm$0.00 & 0.04$\pm$0.00
& 2.65$\pm$0.18 & 2.09$\pm$0.18 & 1.85$\pm$0.17 \\

Dragonnet
& 0.05$\pm$0.00 & 0.05$\pm$0.00 & 0.04$\pm$0.00
& 2.65$\pm$0.18 & 2.14$\pm$0.15 & 1.89$\pm$0.17 \\

\midrule

MOSIC-MLP
& 0.05$\pm$0.00 & 0.05$\pm$0.00 & 0.04$\pm$0.00
& 1.88$\pm$0.16 & 1.48$\pm$0.12 & 1.55$\pm$0.13 \\

MOSIC-DT
& 0.04$\pm$0.00 & 0.04$\pm$0.00 & 0.04$\pm$0.00
& 1.52$\pm$0.13 & 1.50$\pm$0.13 & 1.43$\pm$0.13 \\

MOSIC-Forest
& 0.04$\pm$0.00 & 0.04$\pm$0.00 & 0.04$\pm$0.00
& 1.64$\pm$0.14 & 1.48$\pm$0.14 & 1.27$\pm$0.14 \\

\bottomrule
\end{tabular}
}
\end{table}

\begin{figure}
    \centering
    \includegraphics[width=0.9\linewidth]{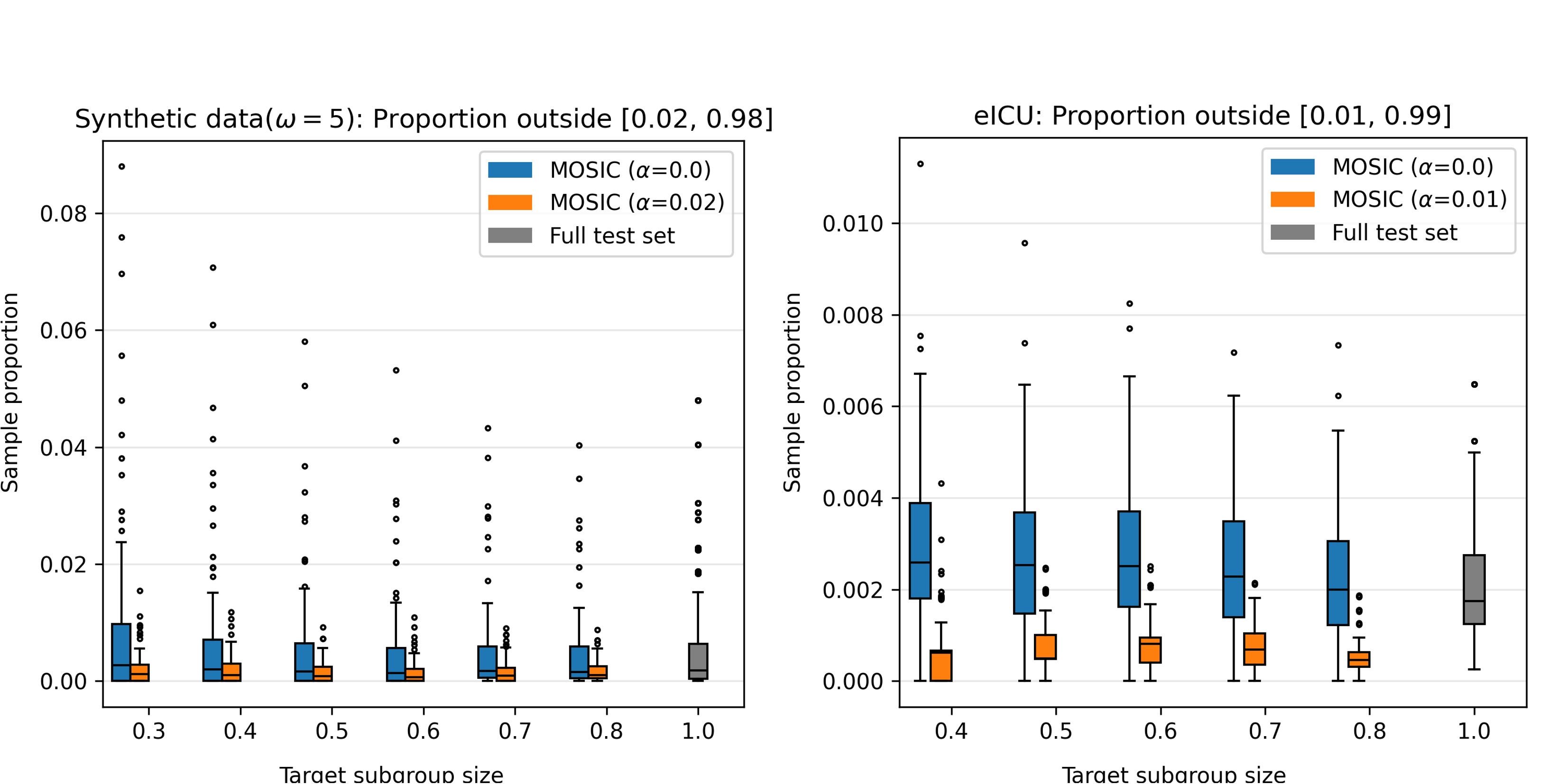}
    \caption{Test set sample proportion that violates the corresponding overlap constraints on confounded synthetic data (left) and eICU (right)}\label{fig:overlap_evaluation_alpha}
    \vspace{-10pt}
\end{figure}


\begin{table*}[!t]
\centering
\caption{Performance on synthetic data ($\tilde{\omega}=5$) under multiple additional constraints. Constraints: 1) Safety: Average Risk $\leq 0.05$; 2) Budget: Total Cost $\leq 1250$; 3) Fairness: $|$Sensitive Group Ratio - 0.5$|$ $\leq 0.01$.}
\resizebox{0.8\linewidth}{!}{%
\begin{tabular}{lcccc}
\toprule
\textbf{Metric} & \shortstack{\textbf{Group Size \& Overlap}} 
                & \shortstack{\textbf{Group Size \& Overlap}\\ \textbf{\& Safety}} 
                & \shortstack{\textbf{Group Size \& Overlap}\\ \textbf{\& Safety \& Budget}} 
                & \shortstack{\textbf{Group Size \& Overlap}\\ \textbf{\& Safety \& Budget \& Fairness}} \\
\midrule
Group Size & 0.50 ± 0.01 & 0.50 ± 0.05 & 0.48 ± 0.05 & 0.49 ± 0.02 \\
\# Unbalanced Features & 0.14 ± 0.40 & 0.20 ± 0.45 & 0.23 ± 0.57 & 0.23 ± 0.45 \\
True CATE & 0.10 ± 0.01 & 0.10 ± 0.01 & 0.09 ± 0.01 & 0.09 ± 0.01 \\
ATE Error & 0.01 ± 0.01 & 0.01 ± 0.01 & 0.01 ± 0.01 & 0.01 ± 0.01 \\
Average Risk & 0.08 ± 0.01 & \textbf{0.05 ± 0.01} & \textbf{0.05 ± 0.01} & \textbf{0.05 ± 0.01} \\
Total Cost & 1377.41 ± 41.57 & 1365.72 ± 144.95 & \textbf{1252.78 ± 132.59} & \textbf{1268.40 ± 42.44} \\
Sensitive Group Ratio & 0.45 ± 0.03 & 0.45 ± 0.03 & 0.46 ± 0.03 & \textbf{0.49 ± 0.02} \\
\bottomrule
\end{tabular}\label{tab:syn_more_constraints}
}
\vspace{-0.5\baselineskip}
\end{table*}

\subsection{Additional Analyses}\label{subsec:ablation}

\paragraph{Incorporating Sample-Splitting}
To assess whether internal sample-splitting improves subgroup learning, we implemented 5-fold cross-fitting on the training data: nuisances were trained on 4 folds and used to generate CATEs on the held-out fold, and the subgroup model was trained on these cross-fitted CATEs. Evaluation remained on the untouched test set, where nuisances were refit on the full training set before computing AIPTW. The results (Figure \ref{fig:sample_splitting}) are similar to our original pipeline, suggesting that the benefit of decoupling nuisance and subgroup estimation is limited in our setting, likely because further splitting of a modest dataset weakens nuisance estimation.

\begin{figure}[t]
    \centering
    \subfigure[Confounded Synthetic Data($\tilde{\omega}=5$)]{%
        \includegraphics[width=0.85\linewidth]{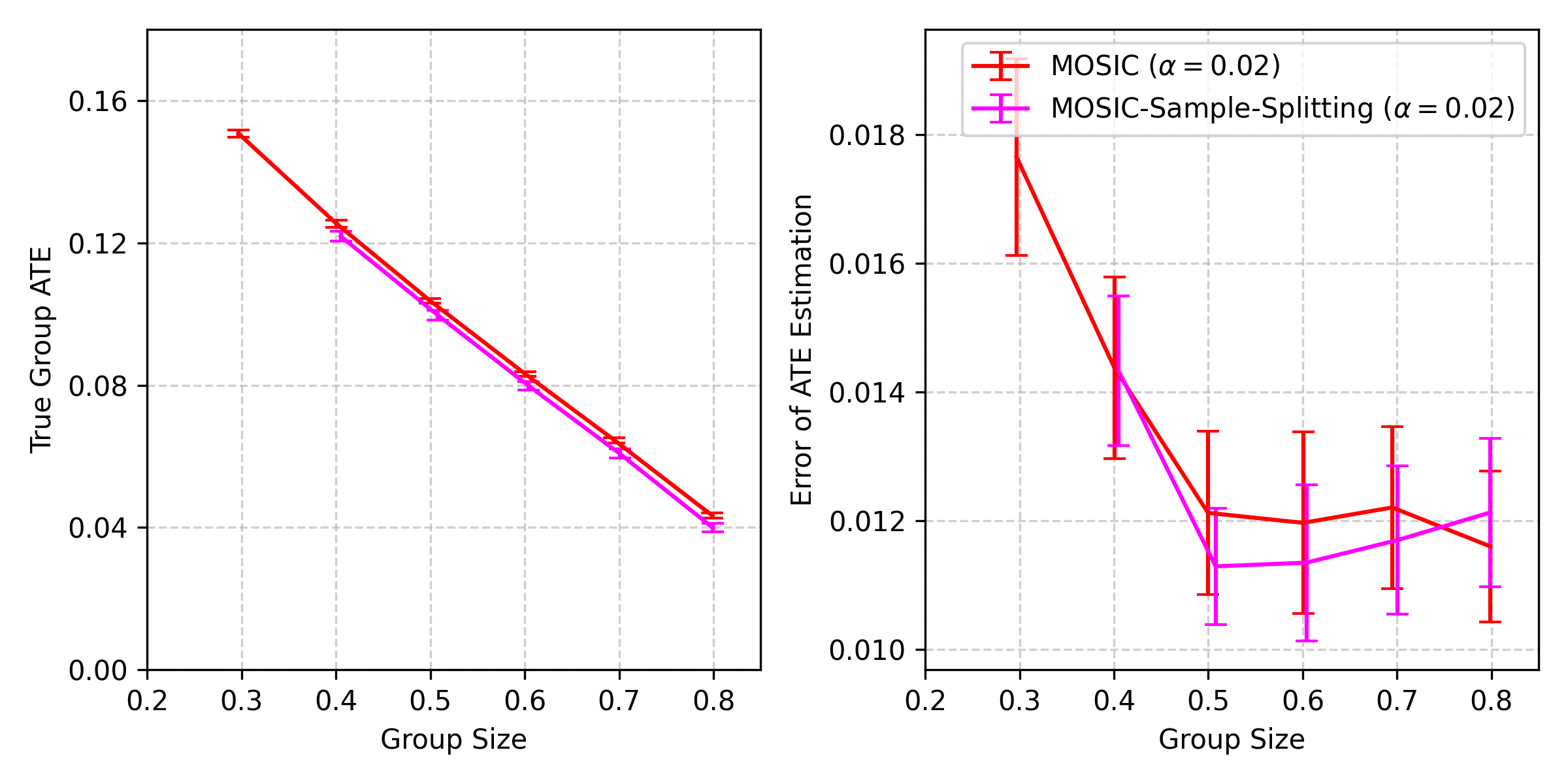}
    }\hfill
    \subfigure[eICU]{%
        \includegraphics[width=0.85\linewidth]{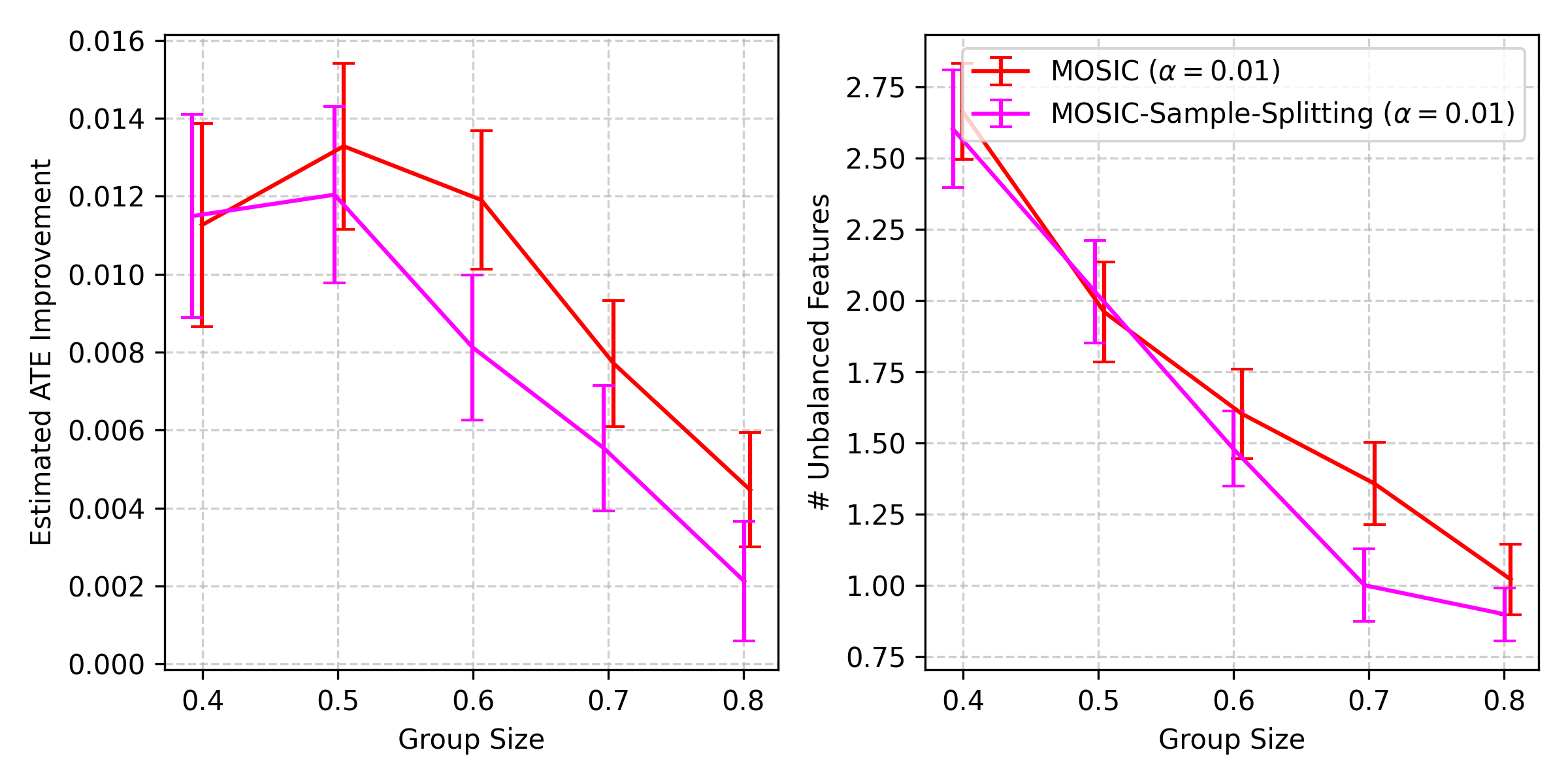}
    }
    \vspace{-5pt}
    \caption{MOSIC w/ and w/o cross-fitted CATE estimates.}
    \vspace{-5pt}\label{fig:sample_splitting}
\end{figure}

\paragraph{Decision Tree as Backbone}\label{subsec:dt_based}
While MOSIC-MLP outperforms baselines, its black-box nature limits interpretability. In contrast, decision trees are favored in clinical settings for their transparency
~\citep{cai2022capital}.
We therefore compare \methodname\ with DT backbone to DT-based baselines (CT, CF) on eICU. VT is excluded due to its high ATE estimation error on synthetic data (Figure~\ref{fig:syn-w-cs}). 

To impose a fair comparison, we match model capacity: we fix the tree depth to 5 for both MOSIC (denoted as MOSIC-DT) and CT, and use ensembles of 3 trees
of depth 5 for MOSIC (denoted as MOSIC-Forest) and CF. As shown in Figure~\ref{fig:eicu_dt}, MOSIC 
consistently outperforms CT and CF under the same interpretability requirement. Notably, despite the limited model capacity, MOSIC effectively enforces both group size and overlap constraints, leading to improved covariate balance.

\begin{figure}[t]	
    \centering
    \includegraphics[width=1\columnwidth]{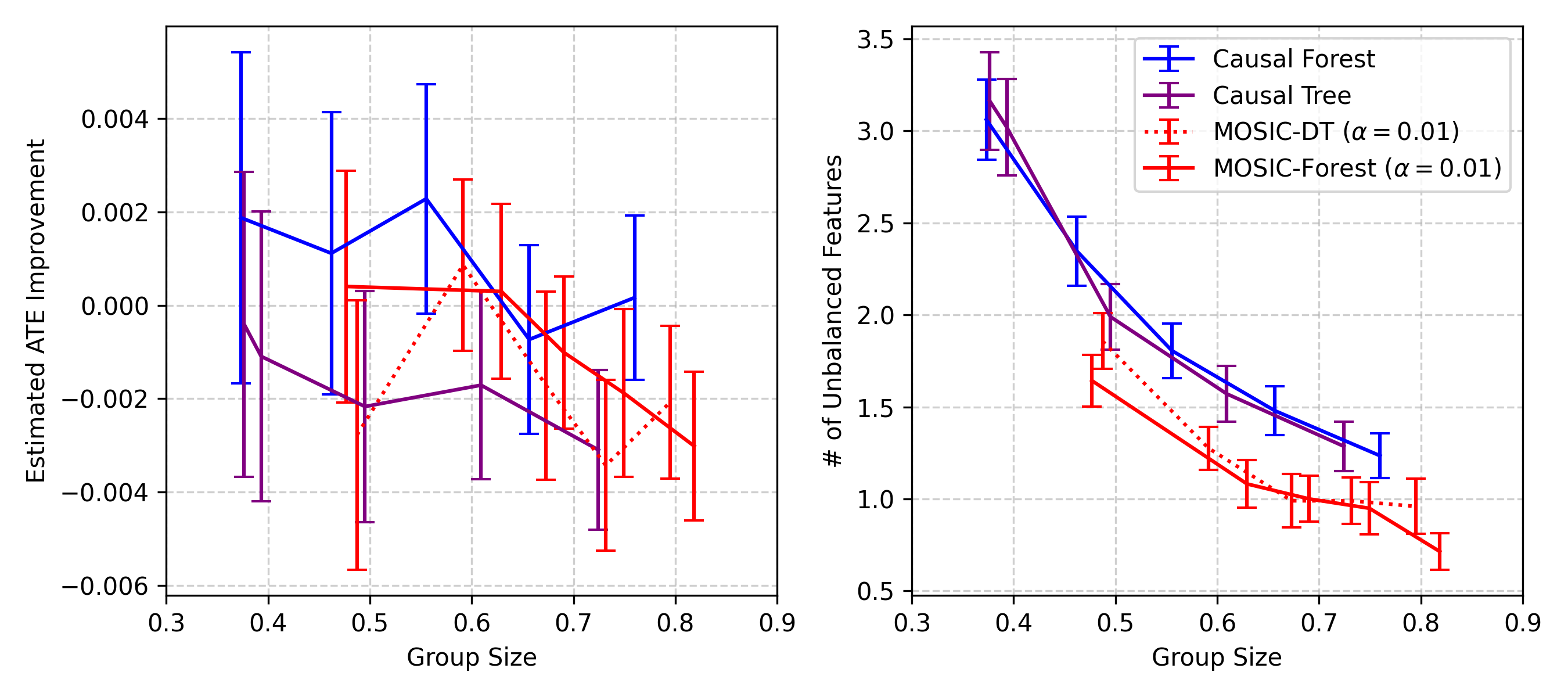}
    \caption{Results on eICU: MOSIC with DT backbone vs. DT-based baselines.
    }\label{fig:eicu_dt}
    \vspace{-10pt}
\end{figure}

\paragraph{Extension to Additional Constraints}\label{subsec:ablation_more_constraint}
In addition to the group size and overlap constraint introduced earlier, MOSIC readily extends to other constraints. We demonstrate this capability on both synthetic and real-world data. On synthetic data, we progressively impose further requirements on top of the size and overlap constraints: first adding a safety constraint, then safety and budget constraints, and finally safety, budget, and fairness constraints. Table~\ref{tab:syn_more_constraints} demonstrates that MOSIC can effectively satisfy all of them (Detailed specification of these constraints is provided in 
Appendix E.9).

On eICU, in addition to the size ($c=0.4$) and overlap constraint ($\alpha=0.01$), we introduced a safety constraint 
motivated by evidence that glucocorticoids may exacerbate neural damage~\citep{hill2021glucocorticoids}. 
In particular, we required that the proportion of patients with a Glasgow Coma Scale (GCS) score $<6$
remain below 0.05. 
GCS measures the level of consciousness, with lower scores indicating more severe dysfunction. We targeted patients with GCS $<6$ because this threshold represents the most severe central nervous system dysfunction in the SOFA, the most commonly used ICU metric~\cite{singer2016third}. Without a constraint on the GCS score, we observed that the selected subgroup contained a nontrivial number of patients with GCS$<6$. Since observational studies have suggested that glucocorticoids may exacerbate neural system damage~\cite{hill2021glucocorticoids}, such a subgroup raises a safety concern. 
We therefore limit, rather than strictly exclude, such patients, acknowledging that a small subset with severe neurological impairment may still derive benefit from treatment.

\begin{figure*}[!t]
    \centering

    \begin{minipage}[t]{0.3\linewidth}
        \centering
        \subfigure[MOSIC-DT without GCS constraint]{
            \includegraphics[width=0.75\linewidth]{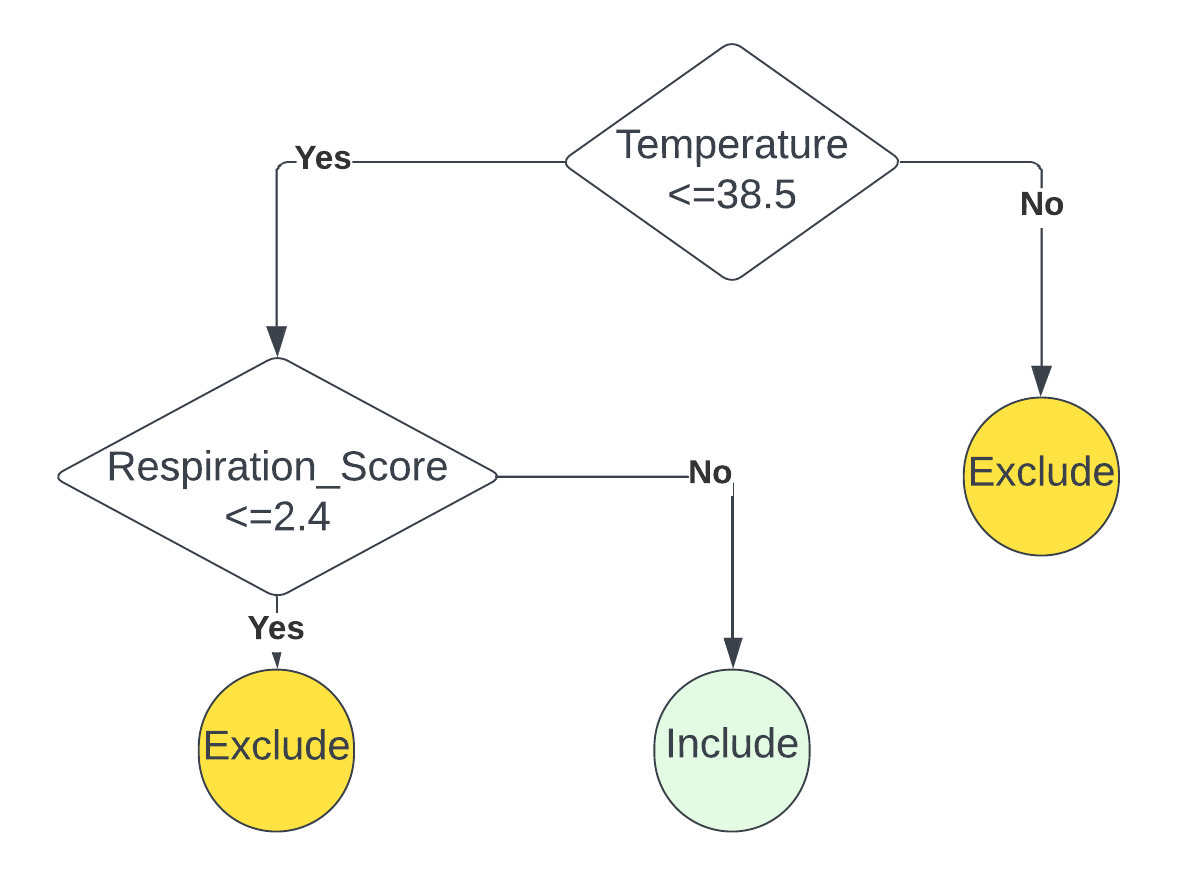}
            \label{subfig:eicu_dt_wo_cns}
        }

        \vspace{-1em}

        \subfigure[MOSIC-DT with GCS constraint]{
            \includegraphics[width=0.9\linewidth]{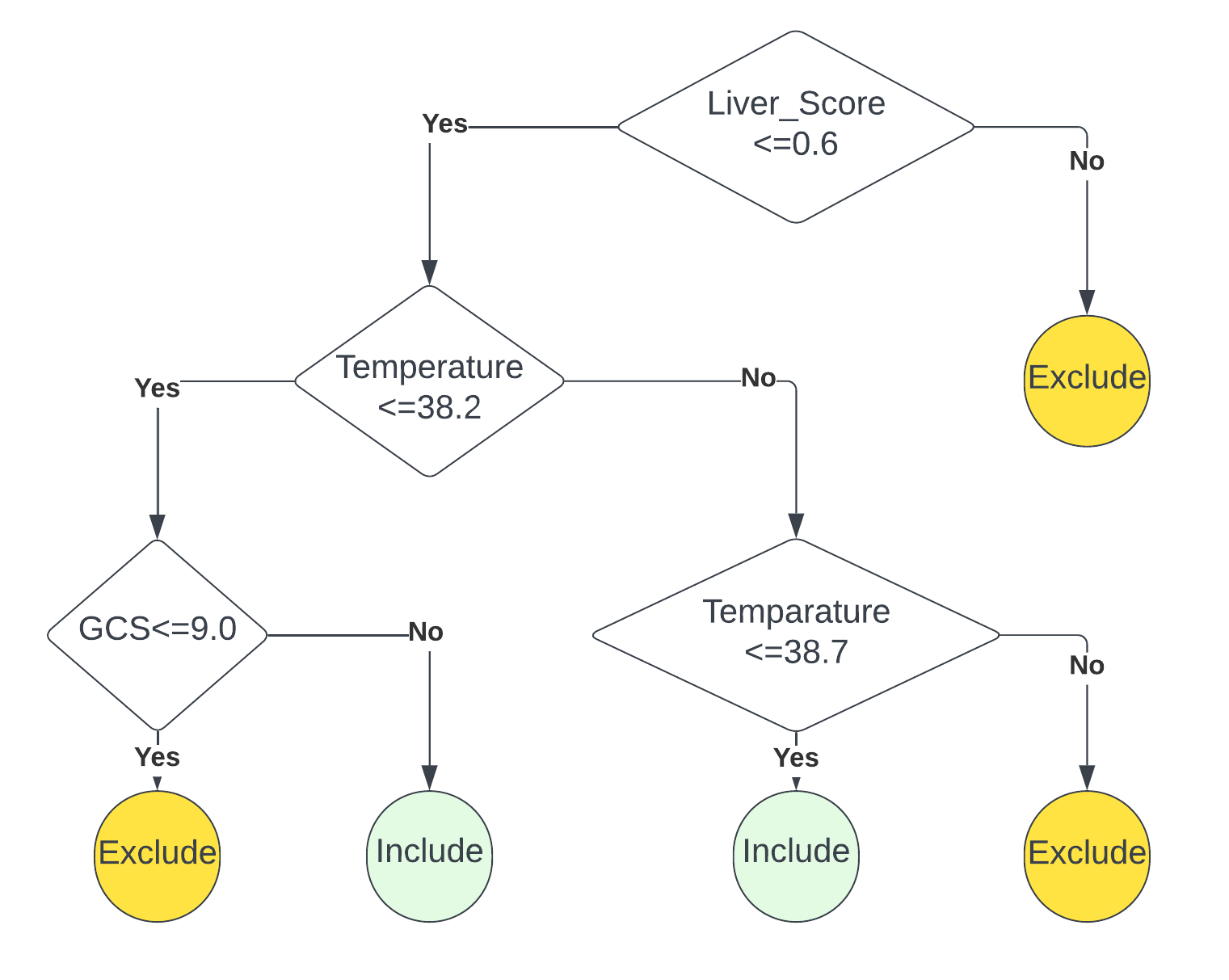}
            \label{subfig:eicu_dt_w_cns}
        }
    \end{minipage}
    \hfill
    \begin{minipage}[t]{0.32\linewidth}
        \centering
        \subfigure[SHAP analysis of MOSIC-MLP without GCS constraint]{
            \includegraphics[width=\linewidth]{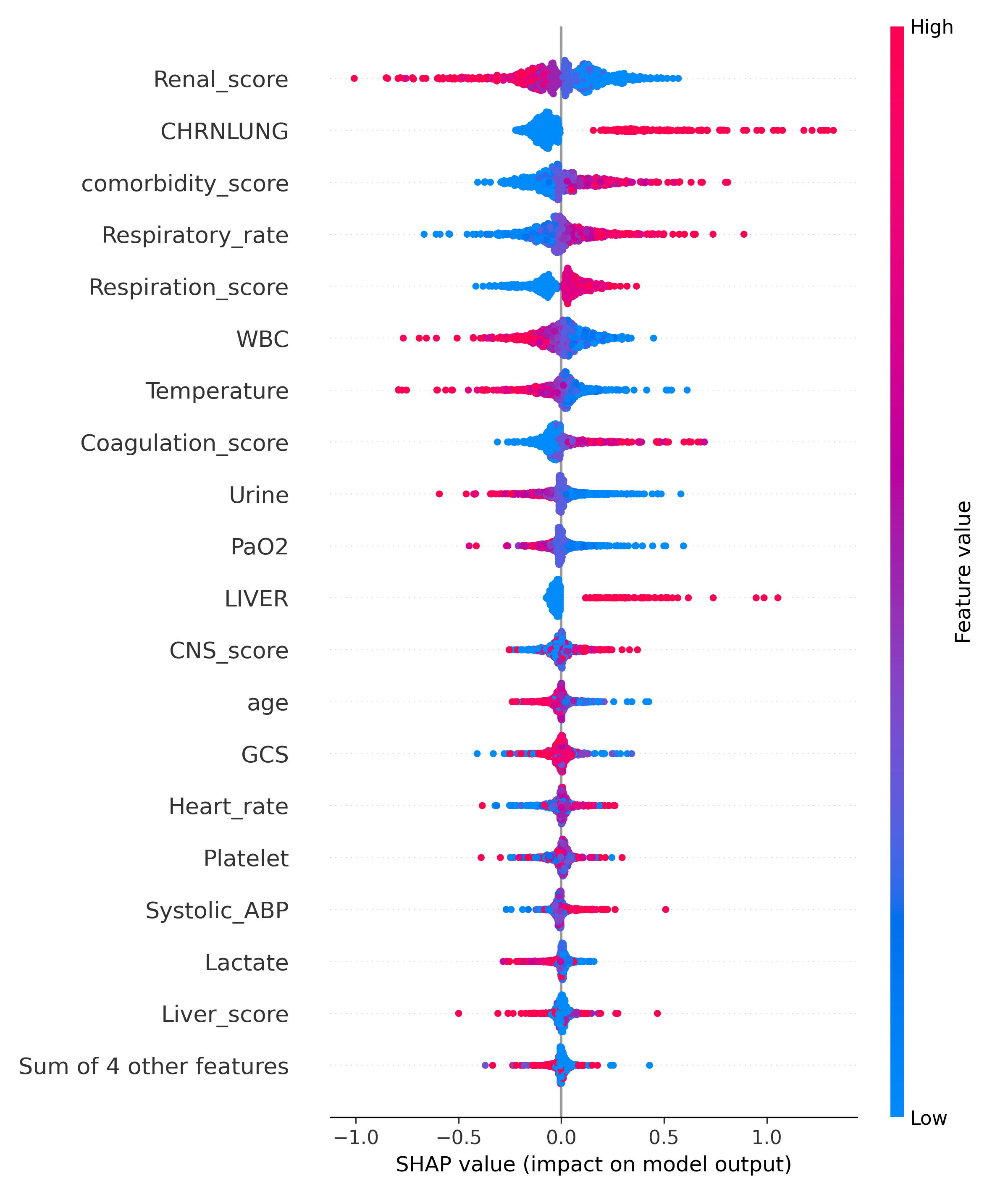}\label{fig:eicu_shap_wo_cns_constraint}
        }
    \end{minipage}
    \hfill
    \begin{minipage}[t]{0.32\linewidth}
        \centering
        \subfigure[SHAP analysis of MOSIC-MLP with GCS constraint]{
            \includegraphics[width=\linewidth]{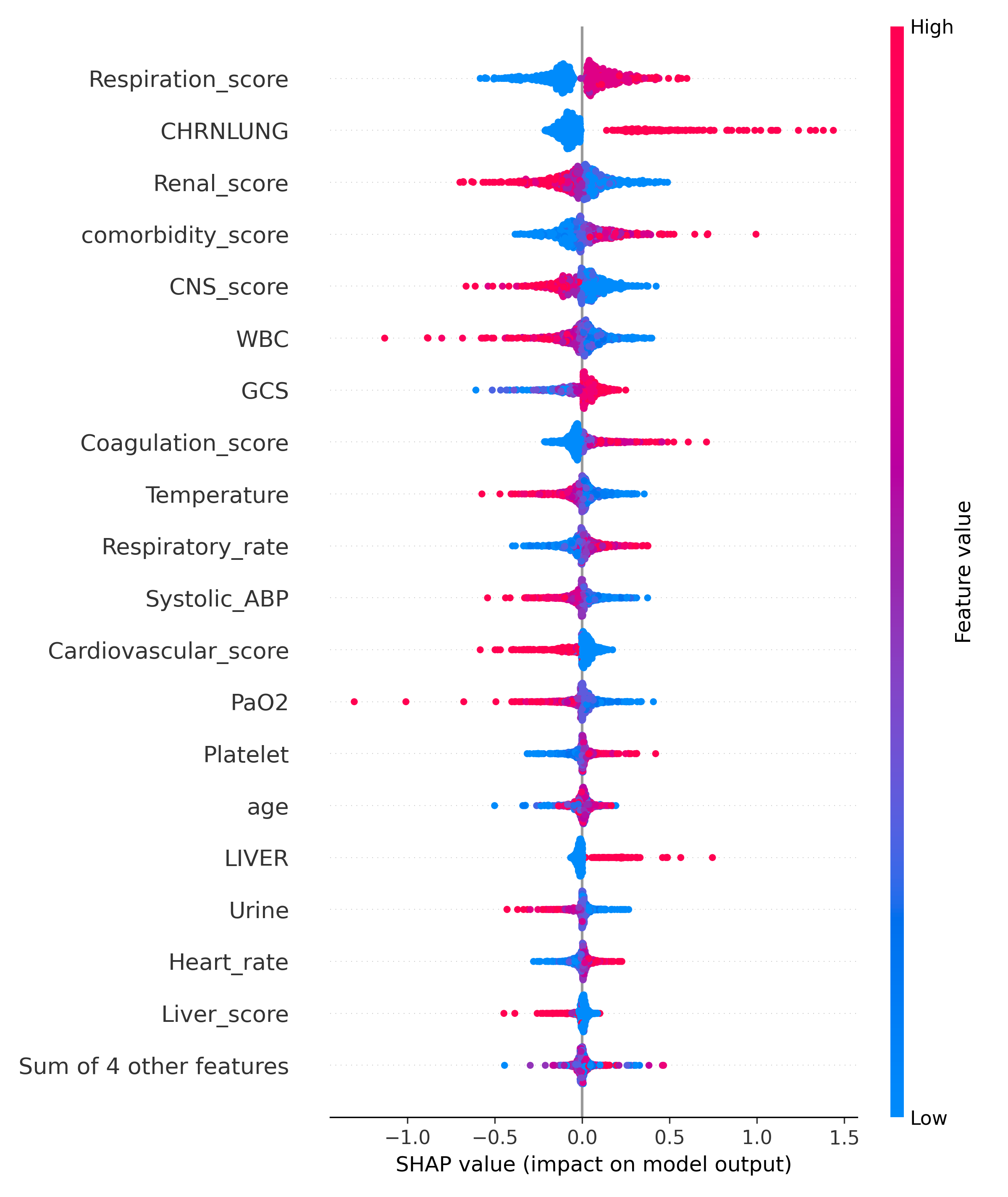}\label{fig:eicu_shap_w_cns_constraint}
        }
    \end{minipage}

    \vspace{-6pt}
    \caption{Resulting DT example of MOSIC-DT (left) and SHAP analyses of MOSIC-MLP (middle/right) on eICU with and without the additional GCS$<6$ constraint. The two DTs are trained from the same split for comparison.}
    \label{fig:eicu_cns_constraint}
    \vspace{-10pt}
\end{figure*}

\begin{table}[t]
\centering
\caption{Results on eICU with the additional constraint requiring the proportion of patients with GCS$<6$ to remain below 0.05. We report the mean $\pm$ standard error over 100 random splits.}
\label{table:eicu_cns4}
\resizebox{\linewidth}{!}{%
\begin{tabular}{lcc}
\toprule
& \multicolumn{2}{c}{\textbf{Constraint}} \\
\cmidrule(lr){2-3}

\textbf{Metric} & \textbf{Size \& Overlap} & \textbf{Size \& Overlap \& GCS} \\
\midrule
\rowcolor{gray!3}
Group Size            & $0.46 \pm 0.10$   & $0.44 \pm 0.10$   \\
\# Unbalanced Features & $2.0 \pm 1.72$      & $2.2 \pm 2.28$      \\
\rowcolor{gray!3}
ATE Improvement         & $0.02 \pm 0.02$ & $0.02 \pm 0.02$ \\
Proportion of GCS $<6$   & \bm{$0.15 \pm 0.05$}  & \bm{$0.03 \pm 0.03$} \\
\bottomrule
\end{tabular}
}
\vspace{-10pt}
\end{table}

We evaluate both MOSIC-DT and MOSIC-MLP with this additional constraint (See 
Appendix E.10 
for details). Table~\ref{table:eicu_cns4} shows that MOSIC-DT can additionally satisfy this safety constraint even under the interpretability restriction. The resulting DT examples show that MOSIC-DT indeed introduces an explicit rule to exclude low-GCS patients (Figure~\ref{subfig:eicu_dt_w_cns}), in contrast to the result before adding this constraint (Figure~\ref{subfig:eicu_dt_wo_cns}).
Since DTs may be sensitive to randomness, we additionally run MOSIC-MLP with and without the GCS$<6$ constraints and run post-hoc SHAP-value analysis to evaluate the feature contributions. Figure~\ref{fig:eicu_cns_constraint} confirms that the dominant features identified by SHAP are generally consistent with the rule structure revealed by the decision tree. In particular, the SHAP values of the GCS features were ambiguous before we enforced the constraint that avoids patients with GCS $<6$ (Figure~\ref{fig:eicu_shap_wo_cns_constraint}), but they became clearly separated after we enforced it (Figure~\ref{fig:eicu_shap_w_cns_constraint}), aligning well with our intention. 

\paragraph{Discussion on an identified subgroup rule}
To provide insight into the characteristics of the identified subgroup, we examine one learned subgroup rule.
Fig.~\ref{subfig:eicu_dt_wo_cns} shows that MOSIC-DT primarily selects patients with higher respiratory severity and lower temperature. The respiratory component is clinically plausible, as corticosteroids are known to benefit patients with severe pulmonary inflammation, such as ARDS and pneumonia~\cite{chaudhuri2021corticosteroids, torres2015effect}. The SHAP analysis of MOSIC-MLP (Fig.~\ref{fig:eicu_shap_wo_cns_constraint}) reveals a similar pattern, assigning higher importance to chronic lung disease, elevated respiratory rate, and indicators of respiratory severity. In contrast, the temperature signal is less straightforward. While hypothermia in sepsis is associated with worse outcomes and dysregulated host response~\cite{bhavani2019identifying}, it is not an established marker of corticosteroid benefit. Overall, some subgroup characteristics align with existing clinical knowledge, whereas others warrant further investigation.

\begin{table}[!t]
\centering
\caption{Runtime (seconds) of each method across different sample sizes. MOSIC shares the same first-stage nuisance fitting as Dragonnet; reported times reflect MOSIC’s second-stage optimization only.}
\label{tab:runtime}
\resizebox{\linewidth}{!}{%
\begin{tabular}{llcccc}
\toprule
& & \multicolumn{4}{c}{\textbf{Sample Size}} \\
\cmidrule(lr){3-6}

Method & Constraint & 1000 & 3000 & 10000 & 30000 \\
\midrule
Dragonnet & - & 15.77 & 25.86 & 58.22 & 151.32 \\
CT & - & 0.00 & 0.01 & 0.03 & 0.06 \\
CF & - & 0.02 & 0.05 & 0.16 & 0.41 \\
VT & - & 0.07 & 0.16 & 0.37 & 0.67 \\
OWL & - & 0.20 & 0.20 & 0.73 & 1.83 \\
CAPITAL & Size & 1.37 & 7.29 & 25.33 & 89.57 \\
MOSIC & Size + Overlap & 1.64 & 1.97 & 1.67 & 2.24 \\
MOSIC & Size + Overlap + Safety & 1.50 & 1.78 & 1.90 & 2.53 \\
MOSIC & Size + Overlap + Safety + Budget & 2.34 & 1.78 & 1.74 & 2.43 \\
MOSIC & Size + Overlap + Safety + Budget + Fairness & 2.35 & 2.04 & 1.64 & 2.44 \\
\bottomrule
\end{tabular}
}
\vspace{-1.5em}
\end{table}

\paragraph{Runtime Analysis and Training Dynamics} 


Gradient-based methods such as GDA scale well in high-dimensional settings, with per-iteration cost linear in the number of parameters. This makes MOSIC more scalable than combinatorial approaches such as CAPITAL. Although MOSIC involves many constraints, its overhead is modest: each constraint requires a single evaluation per update, resulting in a cost of $O(nm)$ for $m$ constraints and $n$ samples. Moreover, the $n$ overlap constraints contribute only $O(n)$ operations in total, since each involves a single multiplication. Consequently, MOSIC maintains a low per-iteration cost despite the large number of constraints. Figure~\ref{fig:eicu_training_dynamic} further confirms stable optimization under multiple constraints.

We also provide an empirical runtime comparison 
(Appendix H). 
To reflect resource-limited clinical environments, all experiments are conducted on CPUs. Table~\ref{tab:runtime} shows that the optimization overhead is small relative to nuisance estimation, suggesting that computation is unlikely to be a practical deployment bottleneck.

\begin{figure}[t]
    \centering
    \includegraphics[width=1\linewidth]{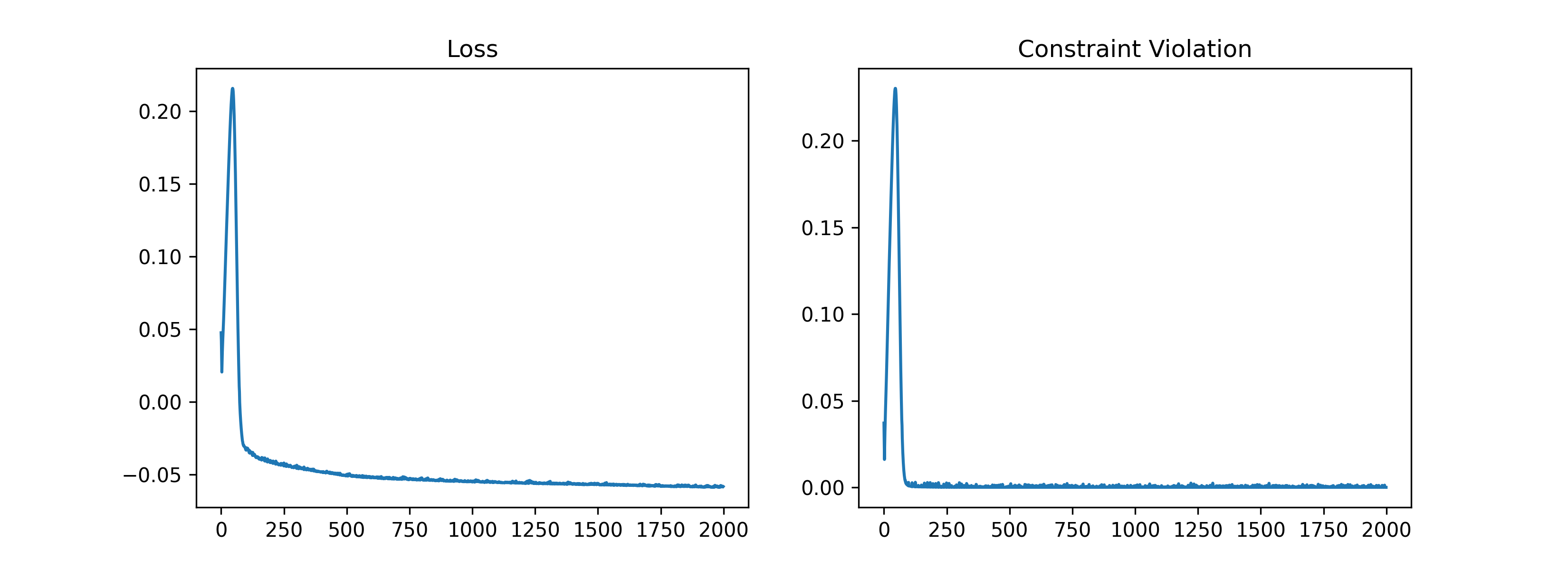}
    \caption{Training dynamics on eICU under size and overlap constraints. Left: training loss. Right: total constraint violation (positive values indicate violations; 0 indicates all constraints are satisfied)} \label{fig:eicu_training_dynamic}
\end{figure}

\vspace{-10pt}
\section{Conclusion}
We propose MOSIC, a model-agnostic framework for optimal subgroup identification that handles multiple constraints with feasibility guarantees. It demonstrates strong empirical performance under group size and overlap constraints, flexibly extends to additional clinical constraints, and supports diverse models to deliver interpretable solutions.

Finally, we acknowledge that the real-world ICU datasets used in this study have known limitations, including potential immortal-time bias and unmeasured confounding. Therefore, subgroup rules obtained from these datasets (e.g., Fig.~\ref{fig:eicu_cns_constraint}) should be interpreted as hypothesis-generating rather than clinically prescriptive. These intrinsic limitations affect all subgroup identification and CATE-based approaches equally, 
and do not affect the conclusion of our comparison.
In addition, our ablation study using a GCS-based safety constraint demonstrates how domain knowledge can be integrated to refine subgroup definitions when needed.

\section{GenAI Disclosure}
We used generative AI tools solely for language polishing (e.g., grammar and clarity edits). All scientific content, including problem formulation, methodology, theory, experiments, results, and conclusions, was developed and verified by the authors.

\section{Limitations and Ethical Considerations}
This study uses publicly available, de-identified ICU datasets (eICU and MIMIC-IV) released for research under established data use agreements; no re-identification was attempted. The analyses do not involve direct interaction with human subjects and are conducted in accordance with the datasets’ policy. Methodologically, results are based on observational data and standard causal assumptions (e.g., unconfoundedness and overlap), so identified subgroups and estimated effects should be interpreted as exploratory rather than clinical guidance. The proposed method is intended for research and decision support, not for direct clinical deployment without further validation and oversight.

\begin{acks}
We thank Suraj Rajendran and Zhenxing Xu for their contributions to MIMIC-IV data processing. This work utilized computing resources provided by AWS. We would also like to acknowledge the support from NSF 2212175, NIH RF1AG072449, RF1AG084178, R01AG080991, R01AG080624, R01AG076448, R01AG076234, and R01NS140142.

\end{acks}

\bibliographystyle{ACM-Reference-Format}
\balance
\bibliography{ref}

\appendix

\newpage

\counterwithin{figure}{section}
\counterwithin{table}{section}
\renewcommand{\thefigure}{\thesection.\arabic{figure}}
\renewcommand{\thetable}{\thesection.\arabic{table}}

\section{Preliminary Definitions and Theorems}\label{appx:def}
\subsection{Local minmax point}
\begin{definition}[Local minmax point]\label{def_local_minmax}
    A point \((\mathbf{\theta}^\star, \mathbf{\lambda}^\star)\) is said to be a \textit{local minmax point} of \(L\), if there exists \(\delta_0 > 0\) and a continuous function \(h\) satisfying \(h(\delta) \to 0\) as \(\delta \to 0\), such that for any \(\delta \leq \delta_0\), and any \((\mathbf{\theta}, \mathbf{\lambda})\) satisfying 
\[
\|\mathbf{\theta} - \mathbf{\theta}^\star\| \leq \delta \quad \text{and} \quad \|\mathbf{\lambda} - \mathbf{\lambda}^\star\| \leq h(\delta),
\]
we have
\[
L(\mathbf{\theta}^\star, \mathbf{\lambda}) \leq L(\mathbf{\theta}^\star, \mathbf{\lambda}^\star) \leq \max_{\mathbf{\lambda}' : \|\mathbf{\lambda}' - \mathbf{\lambda}^\star\| \leq h(\delta)} L(\mathbf{\theta}, \mathbf{\lambda}').
\]
If a point $(\theta^\star,\lambda^\star)$ satisfies 
$$[\nabla^2_{\theta\theta}L - \nabla^2_{\theta\lambda}L(\nabla^2_{\lambda\lambda}L)^{-1}\nabla^2_{\lambda\theta}L] \succ 0,$$
we call it a \textit{strict local minmax point}~\cite{pmlr-v119-jin20e}.
\end{definition}

\subsection{Linearly Stable Point}
\begin{definition}[Linearly Stable Point]\label{def_stable_points}For a differentiable dynamical system \(\mathbf{w}\), a fixed point \(\mathbf{z}^\star\) is a \textit{linearly stable point} of \(\mathbf{w}\) if its Jacobian matrix \(\mathbf{J}(\mathbf{z}^\star) := \left(\frac{\partial \mathbf{w}}{\partial \mathbf{z}}\right)(\mathbf{z}^\star)\) has spectral radius \(\rho(\mathbf{J}(\mathbf{z}^\star)) \leq 1\).
\end{definition}

\subsection{Convergence of the \texorpdfstring{$\gamma$-GDA} \; algorithm}
\begin{theorem}\label{theorem_local_minimax}
(Jin et al \cite{pmlr-v119-jin20e}, Theorem 26): Given an objective: $\underset{x \ in \mathcal{X}}{\min} \underset{y \in \mathcal{Y}}{\max} f(x,y)$. For any twice-differentiable f, the strict linearly stable limit points of the $\gamma$-GDA flow are \{strict local minmax points\} $\cup$ \{($\theta$,$\lambda$) | ($\theta$,$\lambda$) is stationary and $\nabla^2_{\lambda\lambda}f(\theta,\lambda)$ is degenerate\} as $\gamma \rightarrow \infty$.
\end{theorem}

\section{Proofs}\label{appx:proof}

\subsection{Proof of Lemma~\ref{lemma:surrogate_overlap}}\label{appx:proof_surrogate_overlap}
\LemmaSurrogateOverlap*

\begin{proof}
We proceed by proving both directions separately.

\textbf{Forward Direction:} Suppose $S(x_i;\theta)h(x_i) \leq 0$. We aim to show that this implies $\alpha \leq e(x_i) \leq 1 - \alpha$.

\begin{align*}
    & S(x_i;\theta)h(x_i)\leq 0\\
    \implies & h(x_i) \leq 0 \quad\quad \hfill{\text{(Since } S(x_i) >0 \text{)}} \\
    \implies & 1 - \frac{e(x_i)(1 - e(x_i))}{\alpha (1 - \alpha)}\leq 0 \quad\quad \hfill{\text{(Substituting } h(x_i) \text{)}}\\
    \implies & \alpha (1 - \alpha) - e(x_i)(1 - e(x_i))\leq 0\\
    \implies & e(x_i) - \alpha - \left(e(x_i)^2 -\alpha^2 \right) \geq 0\\
    \implies & (e(x_i) - \alpha)\left( 1 - (e(x_i) + \alpha)\right) \geq  0\\
    \implies & (e(x_i) - \alpha)\left( 1 - \alpha - e(x_i) \right) \geq  0\\
    \implies & \alpha \leq e(x_i) \leq 1 - \alpha.
\end{align*}

\textbf{Backward Direction:} Suppose $\alpha \leq e(x_i) \leq 1 - \alpha$. We need to show that this implies $S(x_i;\theta)h(x_i) \leq 0$.

Define the auxiliary function $q(z) = z(1-z)$, whose derivative is given by:
\[
q'(z) = 1 - 2z.
\]
Thus, $q(z)$ attains its maximum at $z = 0.5$. 


When $0 < \alpha \leq e(x_i) \leq 0.5$, we have $q'(e(x_i)) = 1 - 2 e(x_i) \geq 0$, which implies that $q(e(x_i))$ is non-decreasing on $[0,0.5]$. Consequently, from the assumption $\alpha \leq e(x_i)$, we obtain:
  \begin{align*}
      q(e(x_i)) \geq q(\alpha) \implies e(x_i)(1 - e(x_i)) \geq \alpha(1 - \alpha).
  \end{align*}
- Similarly, for $0.5 < e(x_i) \leq 1 - \alpha < 1$, we have $e(x_i)(1 - e(x_i)) \geq \alpha(1 - \alpha)$.

By combining both cases, we conclude that:
\[
e(x_i)(1 - e(x_i)) \geq \alpha(1 - \alpha).
\]
Dividing both sides by $\alpha(1 - \alpha)$ yields:
\begin{align*}
    1 - \frac{e(x_i)(1 - e(x_i))}{\alpha(1 - \alpha)} = h(x_i) \leq 0.
\end{align*}
Since $S(x_i) > 0$, it follows that:
\[
S(x_i)h(x_i) \leq 0.
\]
This completes the proof.
\end{proof}

\subsection{Proof of Proposition~\ref{local_minmax}}\label{appx:proof_local_minmax}
\LocalOptimality*

\begin{proof}
The derivation of Theorem~\ref{theorem_local_minimax} relies on the Jacobian matrix and requires twice differentiability. However, our objective function incorporates ReLU activations, introducing non-differentiability at the origin. This prevents us from directly applying Theorem~\ref{theorem_local_minimax} in its standard form. Nonetheless, since the probability of encountering exact zero inputs to ReLU is negligible, our objective function remains effectively differentiable in practice when using gradient-based optimization.

Thus, the key arguments of Theorem~\ref{theorem_local_minimax} extend to our objective (\ref{prob:minmax}), implying that its stable limit points must either be local minmax points or points where the second derivative is degenerate.

Further, we compute the first and second derivatives of $L$ with respect to $\lambda$:

\begin{equation}\label{fist_derivative}
    \frac{\partial L}{\partial \lambda_i} = \operatorname{ReLU}(g_i(\theta)) - \beta \lambda_i,
\end{equation}
\begin{equation}\label{second_derivative}
\frac{\partial^2L}{\partial \lambda_i \lambda_j} =
\begin{cases}
0 & \text{if } i \neq j, \\
-\beta & \text{if } i = j.
\end{cases}
\end{equation}

Since $\beta > 0$, the Hessian matrix $\frac{\partial^2L}{\partial \lambda_i \lambda_j}$ is never degenerate. Applying Theorem~\ref{theorem_local_minimax}, we conclude that 
the strict linearly stable limit points of the $\gamma$-GDA flow are precisely the set of local minmax points.
\end{proof}

\subsection{Proof of Lemma~\ref{main_lemma}} \label{appx:proof_main_lemma}
\MainLemma*
  
\begin{proof}

\noindent At convergence, the following condition holds:

\begin{equation}\label{eq:first_condition_raw}
    \frac{\partial L}{\partial \bm{\lambda}} \Big|_{\bm{\lambda} = \bm{\lambda^*}} = ReLU(\bm{g}(\bm{\theta^*})) - \beta\bm{\lambda^*} = \bm{0},
\end{equation}
\begin{equation}\label{eq:second_condition_raw}
    \frac{\partial L}{\partial \bm{\theta}}\Big|_{\bm{\theta}=\bm{\theta^*}}=
    -\frac{\partial f}{\partial \bm{\theta}}\Big|_{\bm{\theta}=\bm{\theta^*}} + \bm{\lambda^*}\frac{\partial ReLU}{\partial \bm{g}}\cdot\frac{\partial \bm{g}}{\partial \bm{\theta}}\Big|_{\bm{\theta}=\bm{\theta^*}}=\bm{0}
\end{equation}

\noindent Since Eq.~\eqref{eq:first_condition_raw} applies component-wise to $\bm{g}(\bm{\theta^*})$, we analyze each individual component $g(\bm{\theta^*})$ separately. Due to ReLU, we distinguish between two cases: $g(\bm{\theta^*}) \leq 0$ and $g(\bm{\theta^*}) > 0$. When $g(\bm{\theta^*}) \leq 0$, Lemma~\ref{main_lemma} holds trivially. Therefore, we focus on the remaining case where $g(\bm{\theta^*}) > 0$, which means the constraint is violated and Eq.~\eqref{eq:first_condition_raw} and Eq.\eqref{eq:second_condition_raw} simplifies to 
\begin{equation}\label{eq:first_condition}
    g(\bm{\theta^*}) = \beta\lambda^*,
\end{equation}
\begin{equation}\label{eq:second_condition}
    -\frac{\partial f}{\partial \bm{\theta}}\Big|_{\bm{\theta}=\bm{\theta^*}} + \lambda^*\frac{\partial g}{\partial \bm{\theta}}\Big|_{\bm{\theta}=\bm{\theta^*}}=\bm{0}
\end{equation}
\noindent Substituting Eq.\eqref{eq:first_condition} into Eq.\eqref{eq:second_condition}, we obtain
\begin{equation}\label{eq:lambda_plugged_in}
    -\frac{\partial f}{\partial \bm{\theta}}|_{\bm{\theta}=\bm{\theta^*}} + \frac{g(\bm{\theta^*})}{\beta}\frac{\partial g}{\partial \bm{\theta}}|_{\bm{\theta}=\bm{\theta^*}}=0.
\end{equation}
\noindent We note that
$g(\bm{\theta^*})$ is a scalar, meaning that
each element of $\frac{\partial f}{\partial \bm{\theta}}|_{\bm{\theta}=\bm{\theta^*}}$ is proportional to the corresponding elements in $\frac{\partial g}{\partial \bm{\theta}}|_{\bm{\theta}=\bm{\theta^*}}$ by the same constant $\frac{g(\bm{\theta^*})}{\beta}>0$. This allows us to pick any element of $\bm{\theta}$ to analyze $g(\bm{\theta^*})$. (Note that when $\theta\in \mathbb{R}^d$ is high-dimensional, Eq.\eqref{eq:lambda_plugged_in} is challenging to achieve, because we need one scalar to satisfy all $d$ equations. In practice, we observe small oscillations when the algorithm converges, which implies the algorithm finds it hard to exactly satisfy all equations.)

Denote $\phi_i = \phi(\bm{x}_i, a_i, y_i)$, $\Delta_i = \frac{\partial S(\bm{x}_i; \bm{\theta})}{\partial \theta_j}\Big|_{\bm{\theta}=\bm{\theta^*}} \leq L$, $\mu_\Delta = \mathbb{E}\Delta_i$, where $j = \underset{j}{\arg\max}\frac{\mu_\Delta}{L}$. The definition of function $f$ (Eq.\eqref{eq:subgroup_ate}) results in
\begin{align*}
\Big|\frac{\partial f}{\partial \theta_j} \Big|_{\bm{\theta} = \bm{\theta^*}} \Big|
&= \Big|\frac{ \sum_{i=1}^{n} \phi_i \Delta_i \sum_{s=1}^{n} S(\bm{x_s}; \bm{\theta})
    }{\left( \sum_{i=1}^{n} S(\bm{x}_i; \bm{\theta}) \right)^2 }- \\
    &\frac{\sum_{i=1}^{n} \Delta_i \sum_{s=1}^{n} \phi_s S(\bm{x_s}; \bm{\theta})}{\left( \sum_{i=1}^{n} S(\bm{x}_i; \bm{\theta}) \right)^2}\Big| \\
&= \Big|\frac{ \sum_{i=1}^{n}\sum_{s=1}^{n}\Delta_i S(\bm{x_s}; \bm{\theta})(\phi_i - \phi_s)}{\left( \sum_{i=1}^{n} S(\bm{x}_i; \bm{\theta}) \right)^2}\Big|
\end{align*}
\noindent With $\phi_{max}=\max_i|\hat{\phi}(x_i,a_i,y_i)|$, we have, 
\begin{align}
\left| \frac{\partial f}{\partial \theta_j} \Big|_{\bm{\theta} = \bm{\theta^*}} \right| &\leq 2\phi_{max}\Big|\frac{ \sum_{i=1}^{n}\sum_{s=1}^{n}\Delta_i S(\bm{x_s}; \bm{\theta})}{\left( \sum_{i=1}^{n} S(\bm{x}_i; \bm{\theta}) \right)^2}\Big|\notag\\
&= 2\phi_{max}\Big|\frac{ \sum_{i=1}^{n}\Delta_i\sum_{s=1}^{n} S(\bm{x_s}; \bm{\theta})}{\left( \sum_{i=1}^{n} S(\bm{x}_i; \bm{\theta}) \right)^2}\Big|\notag\\
&= 2\phi_{max}\frac{ \Big|\sum_{i=1}^{n}\Delta_i\Big|}{\sum_{i=1}^{n} S(\bm{x}_i; \bm{\theta})}\label{eq:f_partial_bound}
\end{align}
Substituting Eq.~\eqref{eq:f_partial_bound} into Eq.~\eqref{eq:lambda_plugged_in} and taking absolute values, we obtain
\begin{equation}\label{eq:g_partial_bound}
    g(\bm{\theta^*}) \leq \beta\left(\left|\frac{\partial g}{\partial \bm{\theta}}|_{\bm{\theta}=\bm{\theta^*}}\right|\right)^{-1}\frac{
    2\phi_{max}\left|\sum_{i=1}^{n} \Delta_i\right|
    }{\sum_{i=1}^{n} S(\bm{x}_i; \bm{\theta})}
\end{equation}

Next, we will prove the upper bound of the group size constraint violation and the general linear constraints violation by analyzing $\frac{\partial g}{\partial \bm{\theta}}|_{\bm{\theta}=\bm{\theta^*}}$.

\noindent\textbf{Case 1:} $g(\bm{\theta^*})=c - \frac{1}{n}\sum_{i=1}^nS(\bm{x}_i;\bm{\theta^*}) > 0$: the group size constraint is violated.

\begin{align}
        &\frac{\partial g}{\partial \theta_j}|_{\bm{\theta}=\bm{\theta^*}} = -\frac{1}{n}\sum_{i=1}^n \Delta_i\notag\\
        &\text{Substituting this into Eq.~\eqref{eq:g_partial_bound}:}\notag\\
        & g(\bm{\theta^*}) \leq \beta\left(\left|-\frac{1}{n}\sum_{i=1}^n \Delta_i\right|\right)^{-1}\frac{
    2\phi_{max}\left|\sum_{i=1}^{n} \Delta_i\right|
    }{\sum_{i=1}^{n} S(\bm{x}_i; \bm{\theta})}\notag\\
    &=\beta\frac{2\phi_{max}}{\frac{1}{n}\sum_{i=1}^{n} S(\bm{x}_i; \bm{\theta})}\label{eq:f_g_partial_pluged}
\end{align}
Substitute $\beta\leq \frac{\xi(c-\xi)\mu_{\Delta}}{2\phi_{max}L}$ into Eq.~\eqref{eq:f_g_partial_pluged}, we obtain
$$g(\bm{\theta^*}) \leq \frac{\xi(c-\xi)\mu_{\Delta}}{\frac{1}{n}\sum_{i=1}^{n} S(\bm{x}_i; \bm{\theta})L} \leq \frac{\xi(c-\xi)}{\frac{1}{n}\sum_{i=1}^{n} S(\bm{x}_i; \bm{\theta})} = \frac{\xi(c-\xi)}{c - g(\bm{\theta^*})}$$

\noindent Solving for $g(\bm{\theta^*})$, we obtain $g(\bm{\theta^*})\leq \xi$ (group size constraint satisfied) or $g(\bm{\theta^*})\geq c-\xi$ (the model collapses, group size is near zero).

     Next, we show that the model is improbable to collapse if the feasible region is non-negligible.
     
     Define the collapsed region as
$$
\bm{\Theta}_{\text{collapse}} := \left\{ \bm{\theta} \in \Theta : \frac{1}{n} \sum_{i=1}^n S(\bm{x}_i; \bm{\theta}) < \xi \right\},
$$
and the feasible region as
\begin{align*}
\bm{\Theta}_{\text{feasible}} := 
\Big\{ \bm{\theta} \in \Theta : \;
&\frac{1}{n} \sum_{i=1}^n S(\bm{x}_i; \bm{\theta}) \geq c \\
&\text{and other constraints are met} \Big\}
\end{align*}
for some thresholds \( 0 < \xi \ll c < 1 \). Suppose that the volume of the feasible region dominates that of the collapsed region, i.e.,
$$
|\bm{\Theta}_{\text{collapse}}| \ll |\bm{\Theta}_{\text{feasible}}|.
$$

Given $\forall i, 0\leq S(\bm{x}_i;\bm{\theta})\leq 1$, using Hoeffding inequality, we obtain
$$P\left(\sum_{i=1}^nS(\bm{x}_i;\bm{\theta}) - \mathbb{E}\sum_{i=1}^nS(\bm{x}_i;\bm{\theta}) \leq -t\right) \leq \exp(-\frac{2t^2}{n}).$$
Let $t = \sqrt{\frac{n\log(1/\delta)}{2}}$, with probability at least $1-\delta$, we have
$$\sum_{i=1}^nS(\bm{x}_i;\bm{\theta}) - \mathbb{E}\sum_{i=1}^nS(\bm{x}_i;\bm{\theta}) \geq -\sqrt{\frac{n\log(1/\delta)}{2}}$$
Assume $\bm{\theta}$ is sampled uniformly from $\Theta_{\text{collapse}}\cup\Theta_{\text{feasible}}$, $|\Theta_{collapse}| \ll |\Theta_{feasible}| \Rightarrow P(\bm{\theta}\in \Theta_{\text{collapse}}) \ll P(\bm{\theta}\in \Theta_{\text{feasible}})$. Then,
    \begin{align*}
        \mathbb{E}\sum_{i=1}^nS(\bm{x}_i;\bm{\theta}) &= P(\bm{\theta}\in \Theta_{collapse})\cdot \sum_{i=1}^nS(\bm{x}_i;\bm{\theta}) \\
        &+ P(\bm{\theta}\in \Theta_{feasible})\cdot \sum_{i=1}^nS(\bm{x}_i;\bm{\theta})
    \end{align*}
Given that the contribution from $\Theta_{\text{collapse}}$ is negligible, we approximate:
    $$\mathbb{E}\sum_{i=1}^nS(\bm{x}_i;\bm{\theta})\approx \sum_{i=1}^nS(\bm{x}_i;\bm{\theta}) |\{\bm{\theta}\in \Theta_{feasible}\}\geq nc$$
Therefore, with probability at least $1-\delta$, we have
    \begin{align*}
        \sum_{i=1}^nS(\bm{x}_i;\bm{\theta}) &\geq \mathbb{E}\sum_{i=1}^nS(\bm{x}_i;\bm{\theta})-\sqrt{\frac{n\log(1/\delta)}{2}}\\
        &\geq nc - \sqrt{\frac{n\log(1/\delta)}{2}}\\
        \implies \frac{1}{n}\sum_{i=1}^nS(\bm{x}_i;\bm{\theta}) &\geq c -\sqrt{\frac{\log(1/\delta)}{2n}},
    \end{align*}
i.e. the model is improbable to collapse if the feasible region is non-negligible.

\noindent\textbf{Case 2 :} $g(\bm{\theta^*})=a^k + \bm{b}^{k\top}\bm{S}(\bm{x}; \bm{\theta})>0$:

Similarly, we have
\begin{equation}
    g(\bm{\theta^*}) \leq \beta\left(\left|\sum_{i=1}^n b_i^k\Delta_i\right|\right)^{-1}\frac{
    2\phi_{max}\left|\sum_{i=1}^{n} \Delta_i\right|
    }{\sum_{i=1}^{n} S(\bm{x}_i; \bm{\theta})}.\label{eq:linear_constraint_partial_pluged}
\end{equation}
Let $Z = \sum_{i=1}^nb_i^k\Delta_i$, then $\mathbb{E}Z = \sum_{i=1}^nb_i^k\mathbb{E}\Delta_i=\mu_{\Delta}\sum_{i=1}^nb_i^k.$

By the triangle inequality, we have $\Big||Z| - |\mathbb{E}Z|\Big|\leq |Z - \mathbb{E}Z|$. So, by the Hoeffding inequality, we obtain
\begin{align*}
    P(\Big||Z| - |\mathbb{E}Z|\Big|\geq t) &\leq P(|Z - \mathbb{E}Z|\geq t)\\
    &\leq 2\exp(-\frac{2t^2}{4L^2\sum_{i=1}^n(b_i^k)^2})
\end{align*}
Let $t = L\sqrt{2 \sum_{i=1}^n(b_i^k)^2 \log\frac{2}{\delta}}$. Then with probability at least $1-\delta$,
$$|Z| \geq |\mathbb{E}Z| + t = |\mu_{\Delta}\sum_{i=1}^nb_i^k| + t.$$

Substituting this into Eq.~\eqref{eq:linear_constraint_partial_pluged}, we obtain
\begin{align*}
    g(\bm{\theta^*}) &\leq \beta\frac{
    2\phi_{max}\left|\sum_{i=1}^{n} \Delta_i\right|
    }{\left(|\mu_{\Delta}\sum_{i=1}^nb_i^k| + t\right)\sum_{i=1}^{n} S(\bm{x}_i; \bm{\theta})}\\
    &\leq \beta\frac{
    2\phi_{max}\sum_{i=1}^{n} \left|\Delta_i\right|
    }{\left(|\mu_{\Delta}\sum_{i=1}^nb_i^k| + t\right)\sum_{i=1}^{n} S(\bm{x}_i; \bm{\theta})}\\
    &\leq \beta\frac{2\phi_{max}L}{\left(|\mu_{\Delta}\sum_{i=1}^nb_i^k| + t\right)\frac{1}{n}\sum_{i=1}^{n} S(\bm{x}_i; \bm{\theta})}
\end{align*}

As analyzed in \textbf{Case 1}, when the feasible region is non-negligible, $\frac{1}{n}\sum_{i=1}^{n} S(\bm{x}_i; \bm{\theta})>c-\xi$ with high probability. Therefore,
\begin{align*}
    g(\bm{\theta^*}) &\leq \beta\frac{2\phi_{max}L}{\left(|\mu_{\Delta}\sum_{i=1}^nb_i^k| + t\right)(c-\xi)}\\
    &\text{Plug in }\beta\leq \frac{\xi(c-\xi)|\mu_\Delta|}{2\phi_{max}L},\\
    &\leq \frac{\xi|\mu_\Delta|}{|\mu_{\Delta}\sum_{i=1}^nb_i^k| + t}\\
    &\leq \frac{\xi|\mu_\Delta|}{|\mu_{\Delta}\sum_{i=1}^nb_i^k| + L\sqrt{2 \sum_{i=1}^n(b_i^k)^2 \log\frac{2}{\delta}}}\\
    &\leq \frac{\xi}{|\sum_{i=1}^nb_i^k| + \frac{L}{|\mu_{\Delta}|}\sqrt{2 \sum_{i=1}^n(b_i^k)^2 \log\frac{2}{\delta}}}\\
\end{align*}
By Cauchy-Schwarz inequality, we have $ \sqrt{\sum_{i=1}^n(b_i^k)^2} \geq \frac{|\sum_{i=1}^nb_i^k|}{\sqrt{n}}$, therefore,
\begin{align*}
    g(\bm{\theta^*}) &\leq \frac{\xi}{|\sum_{i=1}^nb_i^k| + \frac{L}{|\mu_{\Delta}|}\frac{|\sum_{i=1}^nb_i^k|}{\sqrt{n}}\sqrt{2 \log\frac{2}{\delta}}}\\
    &=\frac{\xi}{|\sum_{i=1}^nb_i^k| \left( 1 + \frac{L}{|\mu_{\Delta}|\sqrt{n}}\sqrt{\log\frac{2}{\delta}}\right)}
\end{align*}
    

\noindent Put together \textbf{Case 1} and \textbf{Case 2}, we finish the proof.

\end{proof}

\section{Implementation Details}\label{appx:implement}
\subsection{Selection of $\hat{\phi}$} \label{sec:phi}

$\frac{\sum_{i=1}^n S(\bm{x}_i;\theta) \hat{\phi}(\bm{x}_i, a_i, y_i)}{\sum_{i=1}^n S(\bm{x}_i;\theta)}$ can be used to represent the ATE estimated by different methods. And the following $\hat\phi_\mathrm{iptw}$ and $\hat\phi_\mathrm{aiptw}$,
correspond to the ATE estimated using IPTW and AIPTW, respectively.
\emph{inverse probability of treatment weighting} (IPTW) and \emph{augmented inverse probability of treatment weighting} (AIPTW), respectively. 
To clarify, $\phi(\bm{x}_i,a_i, y_i)$ does not depend on the parameter $\bm{\theta}$, these estimates are separated from the parametric surrogate model $S$.
\begin{align*}
    \hat\phi_\mathrm{iptw}(\bm{x}_i, a_i, y_i) =& \frac{a_i}{\hat e(\bm{x}_i)}y_i - \frac{1-a_i}{1-\hat e(\bm{x}_i)} y_i, \\
    \hat\phi_\mathrm{aiptw}(\bm{x}_i, a_i, y_i) = & \hat \mu_1(\bm{x}_i) - \hat \mu_0(\bm{x}_i) + \frac{a_i}{\hat e(\bm{x}_i)}(y_i - \hat \mu_1(\bm{x}_i)) \notag\\
    &- \frac{1-a_i}{1-\hat e(\bm{x}_i)}(y_i - \hat \mu_0(\bm{x}_i)).
\end{align*}

\subsection{Synthetic Data Generation}\label{subsec:synthetic}
Let
    $\sigma_X,\sigma_Y, \rho
    \in \mathbb{R}$ be fixed constants, 
    and let $\beta_1,\beta_\tau,\omega \in \mathbb{R}^p$.
Draw $\bm{X}$ according to a multi-variate normal distribution, and $A, Y(0), Y(1)$ as follows:
        \begin{align}
        &\bm{X} \sim \mathcal{MVN}(0,\sigma_X^2 [(1-\rho)I_p + \rho 1_p1_p^T]),\\
        &A|\bm{X} \sim \textup{Bernoulli}(\sigma(\bm{X}^T\omega)),\\
        &\epsilon \sim \mathcal{N}(0,\sigma_Y^2),\hspace{0.1in}\\
        &Y(0) = (\sin(10 * \bm{X} ) + 5*\bm{X}^2)^T\beta_1+ \epsilon,\\
        &Y(1) = (\sin(10 * \bm{X} ) + 5*\bm{X}^2)^T\beta_1 + \bm{X}^T\beta_\tau + \epsilon.
        \end{align}
We set the parameters as follows:
\begin{align*}
    p &= 10, \quad \sigma_X = \sigma_Y = 0.1, \quad \rho = 0.3, \\
    \beta_1 &= [0, 0, 0, 0, 2, 0, 0, 0, 0, 0], \\
    \beta_\tau &= [0.5, 0.5, 0.5, 0.5, 0, 0, 0, 0, 0, 0], \\
    \omega &= [0, -1 \cdot \tilde{\omega}, -1 \cdot \tilde{\omega}, 1 \cdot \tilde{\omega}, 1 \cdot \tilde{\omega}, -2 \cdot \tilde{\omega}, 0, 0, 0, 0].
\end{align*}

The \emph{imbalance parameter}, $\tilde{\omega} \geq 0$, scales the magnitude of the treatment assignment weight vector $\omega$. In particular, we generated two synthetic datasets:  (1) no confounding bias ($\tilde{\omega} = 0$); and (2) high confounding bias ($\tilde{\omega} = 5$). As there is no limitation to generate synthetic data, we set the total sample size for synthetic data to 5,000 to study the model performance under finite samples and use a balanced 50/50 train-test split. As for real-world datasets, we use a 70/30 train-test split.

\subsection{Baseline Implementation}\label{subsec:baseline}
\noindent\textbf{CAPITAL} identifies a subgroup by maximizing its size while ensuring the CATE exceeds a predefined threshold. Since subgroup size is not directly controlled in this setting, we vary the CATE threshold and construct the group size vs. ATE curve for comparison. 

\noindent\textbf{OWL} is originally designed for individual treatment rule estimation but can be viewed as a subgroup identification method without interpretability constraints. It assigns scores between 0 and 1, which we threshold to obtain subgroups of a desired size. We implement OWL using the DTRlearn2 R package.

\noindent\textbf{VT} The original VT supports both binary and continuous outcomes, but the commonly used R package aVirtualTwins handles only binary outcomes. Since our experiments require both, we implemented our own VT following \cite{foster2011subgroup}: first estimating treatment effects with a random forest, and then fitting a regression tree to assign subgroup scores. Subgroups of different sizes are obtained by thresholding these scores.

\noindent \textbf{Dragonnet \& CT \& CF} We adopt the implementations from the Python package causalml.

\begin{table}
	\centering
	\begin{tabular}{ccc}
		\toprule
		Dataset & Dragonnet & LR \\
		\hline
		Synthetic & 0.10 & 0.10 \\
		MIMIC-IV & 2.32 & 1.57 \\
		eICU & 1.08 & 1.10 \\
		\bottomrule
	\end{tabular}
    \caption{Number of unbalanced features after reweighting by propensity scores obtained by propensity models}\label{table:summary_ps}
\end{table}

\subsection{Model Selection for Nuisance Function Estimation}\label{subsec:nuisance}
We consider LR and Dragonnet for estimating the propensity score model. We follow the literature~\cite{zang2023high} to use the number of unbalanced features after inverse propensity score weighting as the metric (the lower the better) to select the propensity score model. The results are shown in Table \ref{table:summary_ps}, which shows that LR performs as well as or better than Dragonnet. Thus we select LR as the propensity score model for all datasets.

For the potential outcome model, we consider Dragonnet, CF, and CT, as they are the CATE estimators we are comparing against in subgroup identification. Theoretically, CF improves upon CT by reducing bias and producing smoother decision boundaries through aggregating CTs. As for Dragonnet and CF, prior work \cite{kiriakidou2022evaluation} has shown that Dragonnet outperforms CF.

While many other methods exist for nuisance function estimation, our goal is to demonstrate \methodname’s flexibility rather than prescribe a specific estimator. If more effective models are developed or if a particular method performs better in a given setting, we encourage adapting those for nuisance function estimation.

\subsection{Hyperparameter Tuning}\label{subsec:hyperparam}
CAPITAL optimizes the policy tree depth with $max\_depth \in \{2, 3\}$. VT selects decision tree depth for subgroup identification from $max\_depth \in \{3, 5, 7, 10\}$. Similarly, CT uses $max\_depth \in \{3, 5, 7, 10\}$, while CF considers both tree depth $max\_depth \in \{3, 5, 7, 10\}$ and the number of trees $num\_tree \in \{5, 10, 20, 50, 100\}$. Moreover, Dragonnet tunes the hidden layer size with $hidden\_size \in \{50, 100, 200\}$. Finally, the hyperparameters for MOSIC include $\beta \in \{10^{-2},10^{-3},10^{-4},10^{-5}\}$ and those depending on the implementation of the subgroup identification model. For \methodname-MLP, we tune the hidden layer size with $hidden\_size \in \{50, 100, 200\}$.

\section{Real-world observational study setup}
We provide details of our real-world observational study setup using the target trial emulation framework in Table~\ref{tab:TTE_setup_mortality} and Table~\ref{tab:TTE_setup_ventFree}.

\begin{table*}[t]
\centering
\caption{Target Trial Emulation for 7-day mortality}
\label{tab:TTE_setup_mortality}
\resizebox{\linewidth}{!}{%
\begin{tabular}{p{3cm} p{5.5cm} p{5.5cm}}
\toprule
Protocol component & Target trial specification & Target trial emulation \\
\midrule
Enrollment window & Within 24 hours after ICU admission & Within 24 hours after ICU admission \\
\midrule
Eligibility criteria & Adult;
Patient has sepsis at enrollment window; No sepsis history before enrollment
window; No corticosteroid prescription more than 10 hrs before enrollment window & Adult; Patients meet Sepsis-3 criteria at enrollment window; No sepsis history before enrollment window; No corticosteroid prescription more than 10 hrs before enrollment window\\
\midrule
Treatment strategies & Treated: Initiation of hydrocortisone at a dose of 160 mg at enrollment window; \newline\newline Control: No initiation of any corticosteroid drug at enrollment window & 
Treated = receiving >= 160mg per day hydrocortisone-equivalent corticosteroid; 
\newline\newline Control = Otherwise. 
\newline\newline We consider prednisolone, prednisone, hydrocortisone, dexamethasone, and methylprednisolone. \newline
Conversions: \newline
1 mg Prednisolone = 4 mg Hydrocortisone \newline
1 mg Prednisone = 4 mg Hydrocortisone \newline
1 mg Dexamethasone = 25 mg Hydrocortisone \newline
1 mg Methylprednisolone = 5 mg Hydrocortisone\\
\midrule
Treatment assignment & Patients are randomly assigned to
either treatment strategy on the first
day after ICU admission and are
aware of the strategy they are
assigned to. &  Randomization is emulated at enrollment by adjustment for baseline confounders. \newline \newline Covariates include Demographics (Age, gender); Labs/vitals (GCS, Lactate, MAP, PaO2, Platelet, Respiratory rate, Systolic ABP, Temperature, Urine, WBC); Comorbidity (Elixhauser Comorbidity Index, chronic lung disease, liver disease, AIDS); SOFA score and its organ-specific subscores. \\
\midrule
Follow-up & We follow each patient from their baseline until their death, loss of follow-up, or discharge, whichever occurs first. & We follow each patient from their baseline until their death, loss of follow-up, or discharge, whichever occurs first.\\
\midrule
Causal contrasts & Intention-to-treat effect. & Observational analog of intention-to-treat effect.\\
\midrule
Statistical analysis & Mean difference in 7-day mortality & Mean difference in 7-day mortality estimated using AIPW with both outcome and propensity models.\\
\bottomrule
\end{tabular}
}
\end{table*}

\begin{table*}[t]
\centering
\caption{Target Trial Emulation for ventilation-free days within 7 days}
\label{tab:TTE_setup_ventFree}
\resizebox{\linewidth}{!}{%
\begin{tabular}{p{3cm} p{5.5cm} p{5.5cm}}
\toprule
Protocol component & Target trial specification & Target trial emulation \\
\midrule
Enrollment window & Within 24 hours after ICU admission & Within 24 hours after ICU admission \\
\midrule
Eligibility criteria & Adult;
Patient has sepsis at enrollment window; No sepsis history before enrollment
window; No corticosteroid prescription more than 10 hrs before enrollment window & Adult; Patients meet Sepsis-3 criteria at enrollment window; No sepsis history before enrollment window; No corticosteroid prescription more than 10 hrs before enrollment window\\
\midrule
Treatment strategies & Treated: Initiation of hydrocortisone at a dose of 160 mg at enrollment window; \newline\newline Control: No initiation of any corticosteroid drug at enrollment window & 
Treated = receiving >= 160mg per day hydrocortisone-equivalent corticosteroid; 
\newline\newline Control = Otherwise. 
\newline\newline We consider prednisolone, prednisone, hydrocortisone, dexamethasone, and methylprednisolone. \newline
Conversions: \newline
1 mg Prednisolone = 4 mg Hydrocortisone \newline
1 mg Prednisone = 4 mg Hydrocortisone \newline
1 mg Dexamethasone = 25 mg Hydrocortisone \newline
1 mg Methylprednisolone = 5 mg Hydrocortisone\\
\midrule
Treatment assignment & Patients are randomly assigned to
either treatment strategy on the first
day after ICU admission and are
aware of the strategy they are
assigned to. &  Randomization is emulated at enrollment by adjustment for baseline confounders. \newline \newline Covariates include Demographics (Age, gender); Labs/vitals (GCS, Lactate, MAP, PaO2, Platelet, Respiratory rate, Systolic ABP, Temperature, Urine, WBC); Comorbidity (Elixhauser Comorbidity Index, chronic lung disease, liver disease, AIDS); SOFA score and its organ-specific subscores. \\
\midrule
Follow-up & We follow each patient from their baseline until their death, loss of follow-up, or discharge, whichever occurs first. & We follow each patient from their baseline until their death, loss of follow-up, or discharge, whichever occurs first.\\
\midrule
Causal contrasts & Intention-to-treat effect. & Observational analog of intention-to-treat effect.\\
\midrule
Statistical analysis & Mean difference in ventilation-free days within 7 days & Mean difference in ventilation-free days within 7 days estimated using AIPW with both outcome and propensity models.\\
\bottomrule
\end{tabular}
}
\end{table*}

\section{Additional Experiment Results}\label{sec:more_experiment}

\subsection{Correlation of ATE estimation error and feature imbalance}\label{appx:subsec_corr_imbalance}
Figure \ref{fig:syn-error} numerically verifies that a higher number of unbalanced features is correlated with larger estimation error, justifying feature imbalance as a proxy for estimation reliability in real-world experiments. Further, a higher overlap threshold $\alpha$ reduces the subgroup ATE estimation error, suggesting that enforcing stronger overlap improves estimation reliability. 

\begin{figure}[t]
	\centering
    \includegraphics[width=0.9\columnwidth]{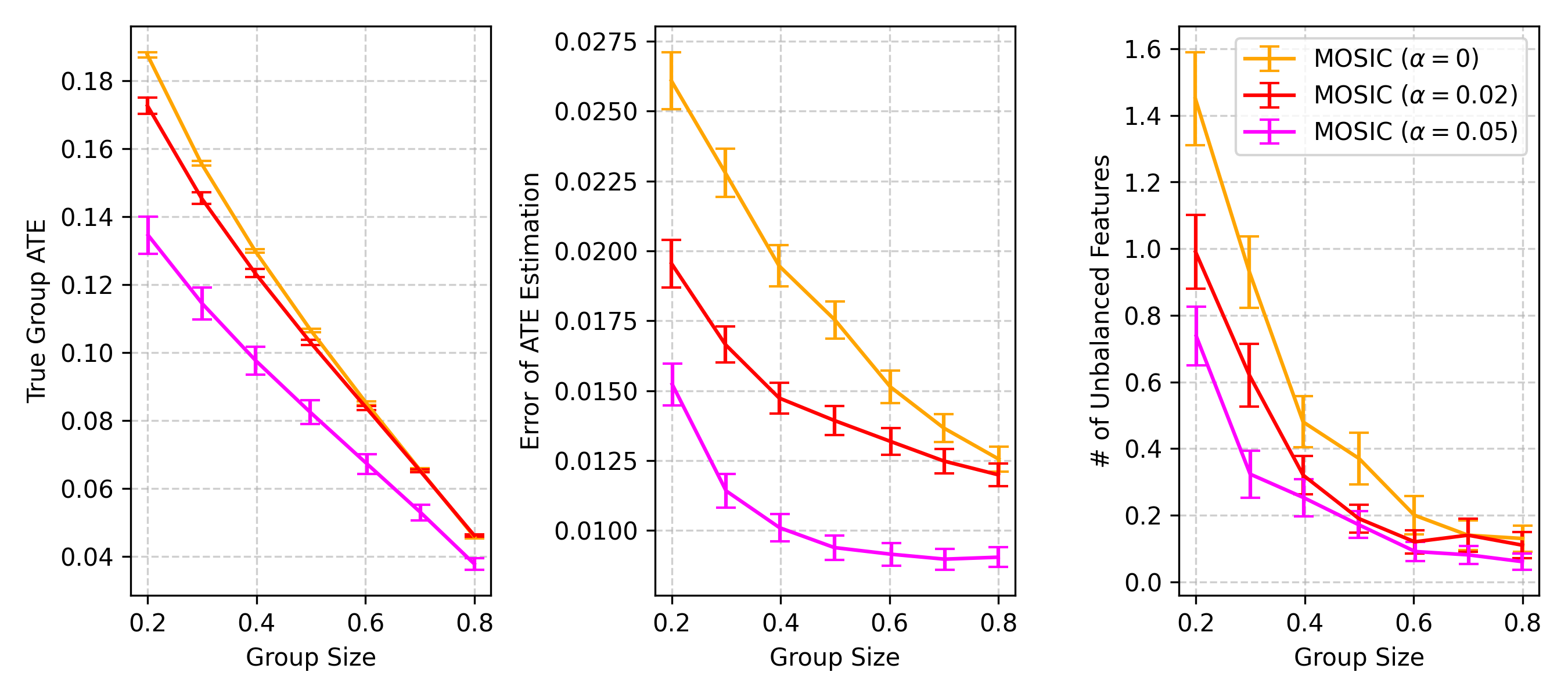}
	\caption{Correlation of ATE estimation error and the number of unbalanced features on confounded synthetic dataset ($\tilde{\omega} = 5$), obtained using \methodname\ with varying $\alpha$.}
    \label{fig:syn-error}
    \vspace{-10pt}
\end{figure}

\subsection{Significance Tests on Real-world data}\label{sec:significance_tests}

Table~\ref{tab:combined_sig_test} presents the p-values comparing \methodname\ ($\alpha=0.01$) against baseline methods, corresponding to the hypothesis test that \methodname\ differs from the baselines. These results directly support the performance trends shown in Figure~\ref{fig:real-world}. For instance, in eICU experiments with c=0.5:
\begin{itemize}
    \item \methodname\ achieves significantly better ATE than Dragonnet (p=0.0057)
    \item \methodname\ maintains significantly fewer unbalanced features than Dragonnet (p=9.2E-04)
\end{itemize}

\subsection{Uncertainty quantification}\label{appx:uncertainty}
To quantify uncertainty, we computed the asymptotic 95\% confidence intervals based on the closed-form influence function of the AIPTW estimator. Figure~\ref{fig:ci-width} reports the distribution of CI widths across 100 independent train–test splits, showing the anticipated increase in uncertainty as subgroup size decreases.

Additionally, we report the Risk Ratio (RR) and E-values to measure the robustness of our subgroup results with respect to unmeasured confounders. 

\begin{figure}[t]
    \centering
    \subfigure[Unconfounded]{%
        \includegraphics[width=0.48\linewidth]{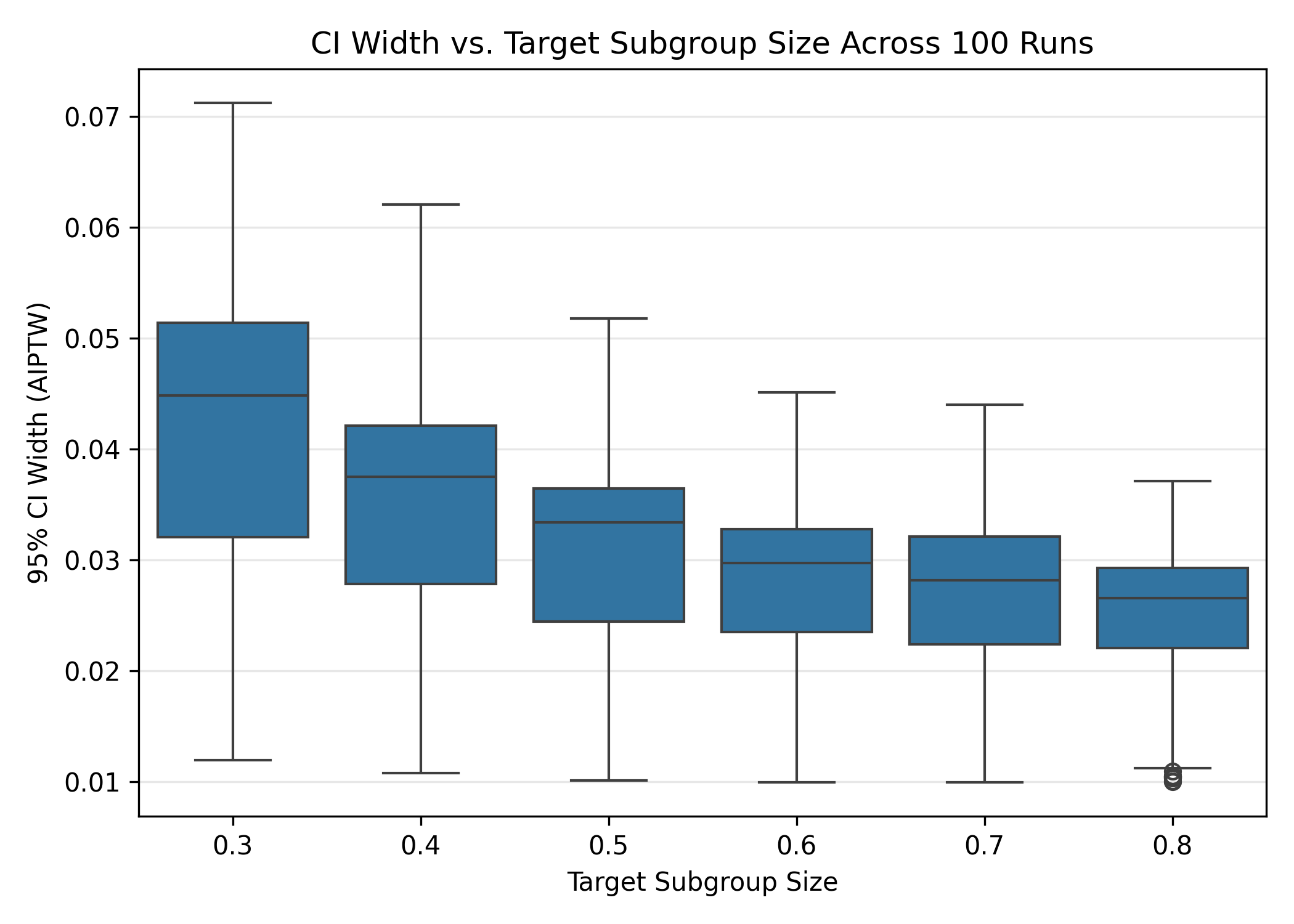}
    }\hfill
    \subfigure[Confounded]{%
        \includegraphics[width=0.48\linewidth]{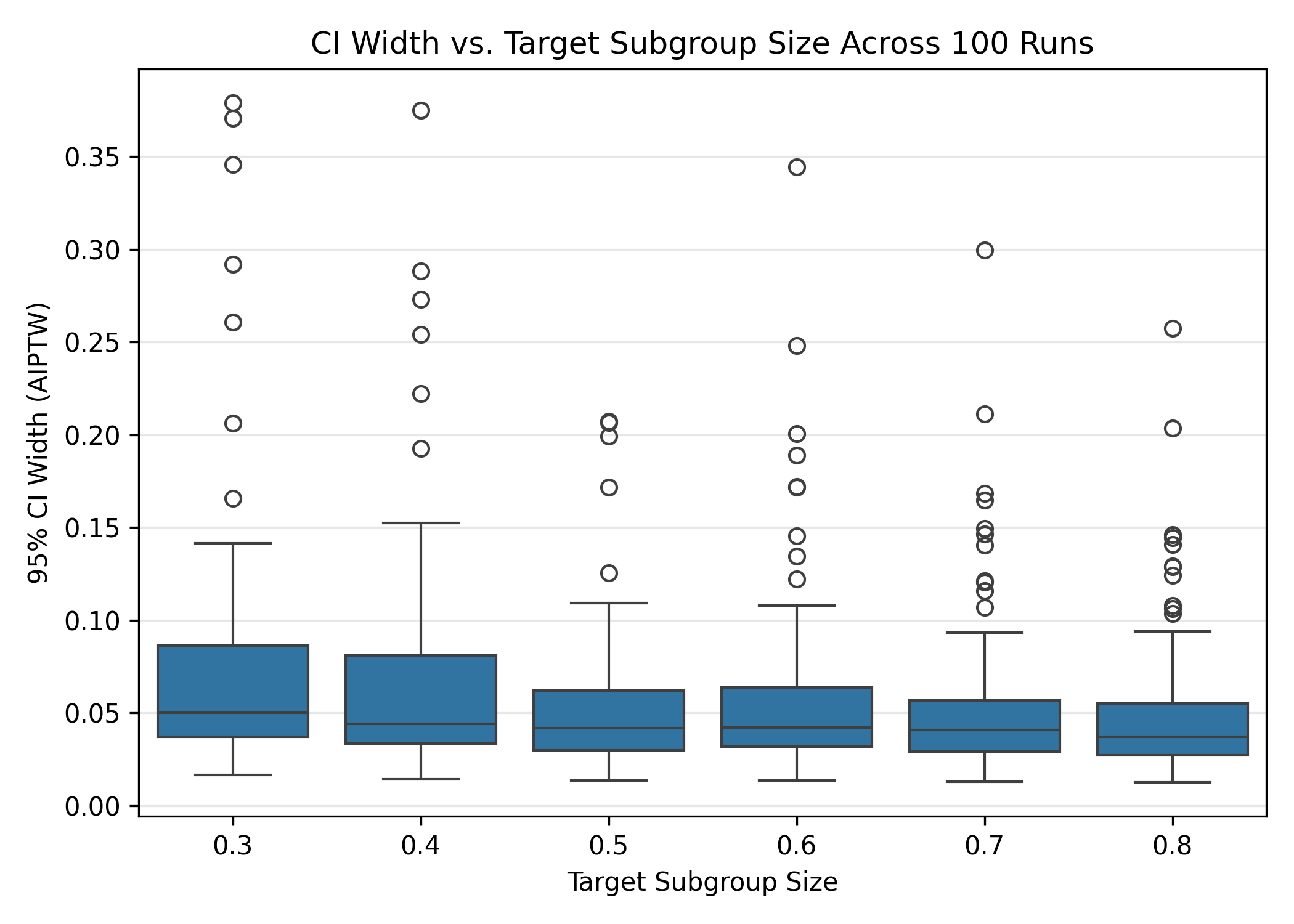}
    }

    \subfigure[eICU]{%
        \includegraphics[width=0.48\linewidth]{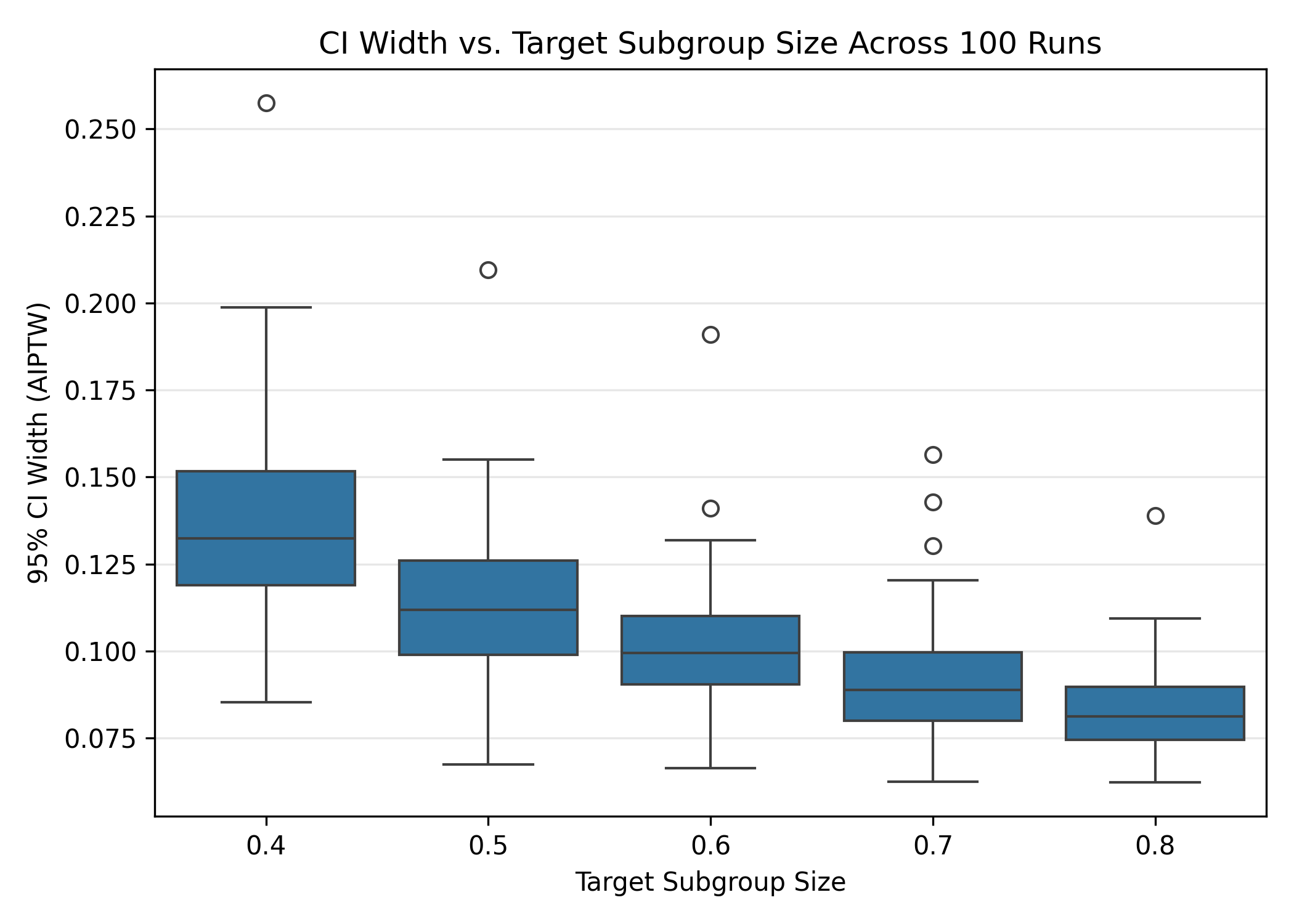}
    }\hfill
    \subfigure[MIMIC-IV]{%
        \includegraphics[width=0.48\linewidth]{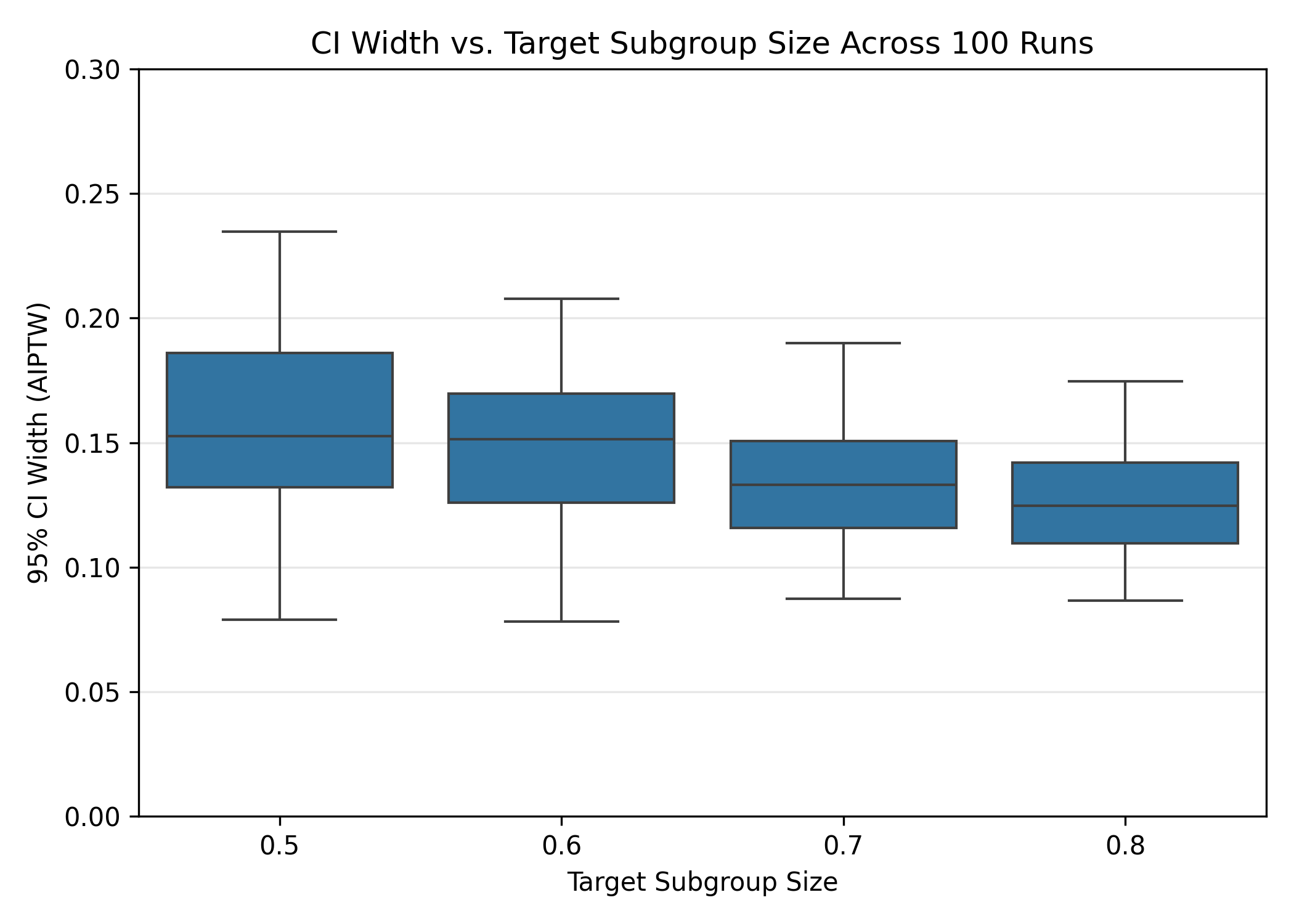}
    }
    \vspace{-5pt}
    \caption{CI width.}
    \vspace{-5pt}\label{fig:ci-width}
\end{figure}

\begin{figure}[t]
    \centering
    \subfigure[Risk Ratio]{%
        \includegraphics[width=0.45\linewidth]{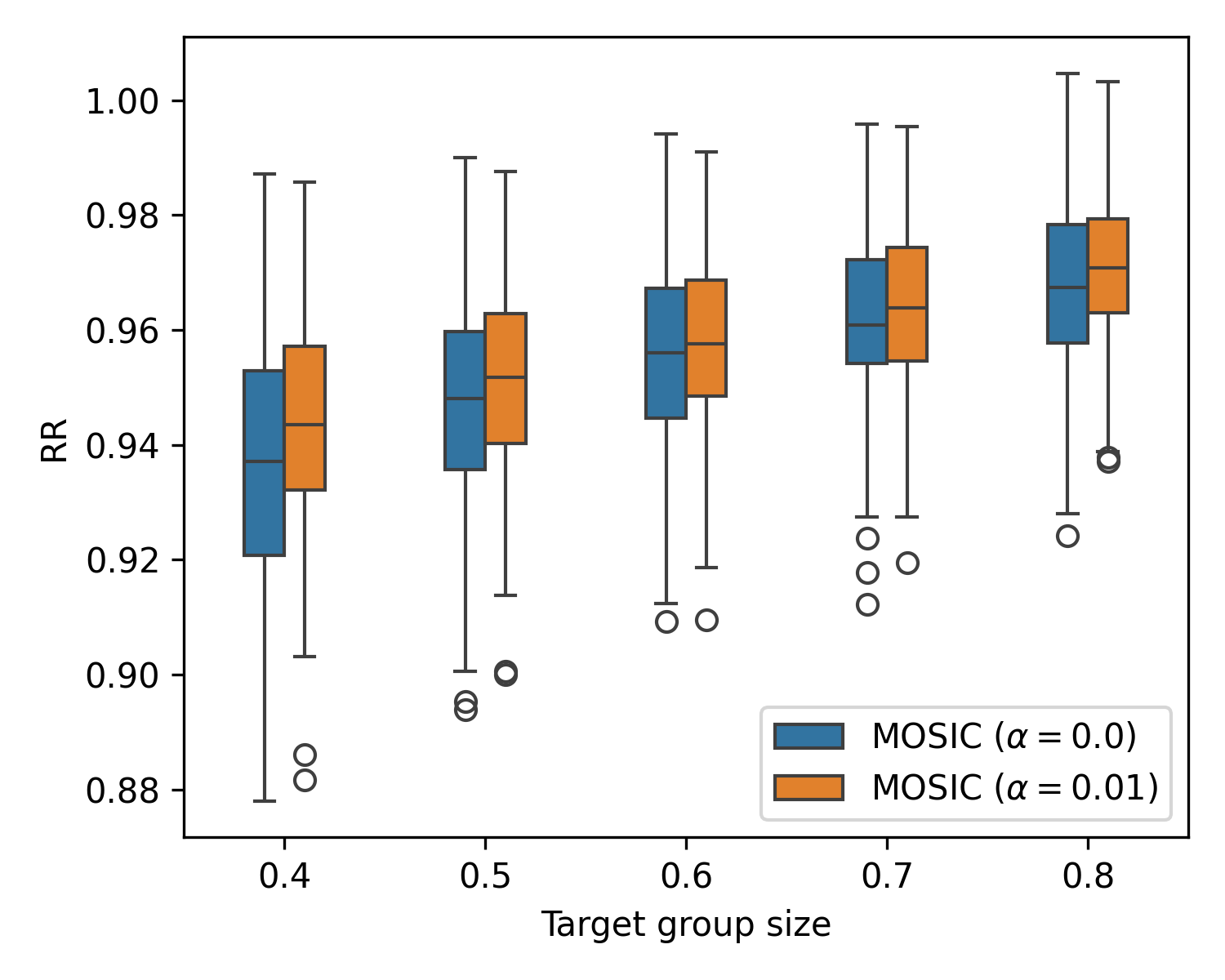}
    }\hfill
    \subfigure[E-value]{%
        \includegraphics[width=0.45\linewidth]{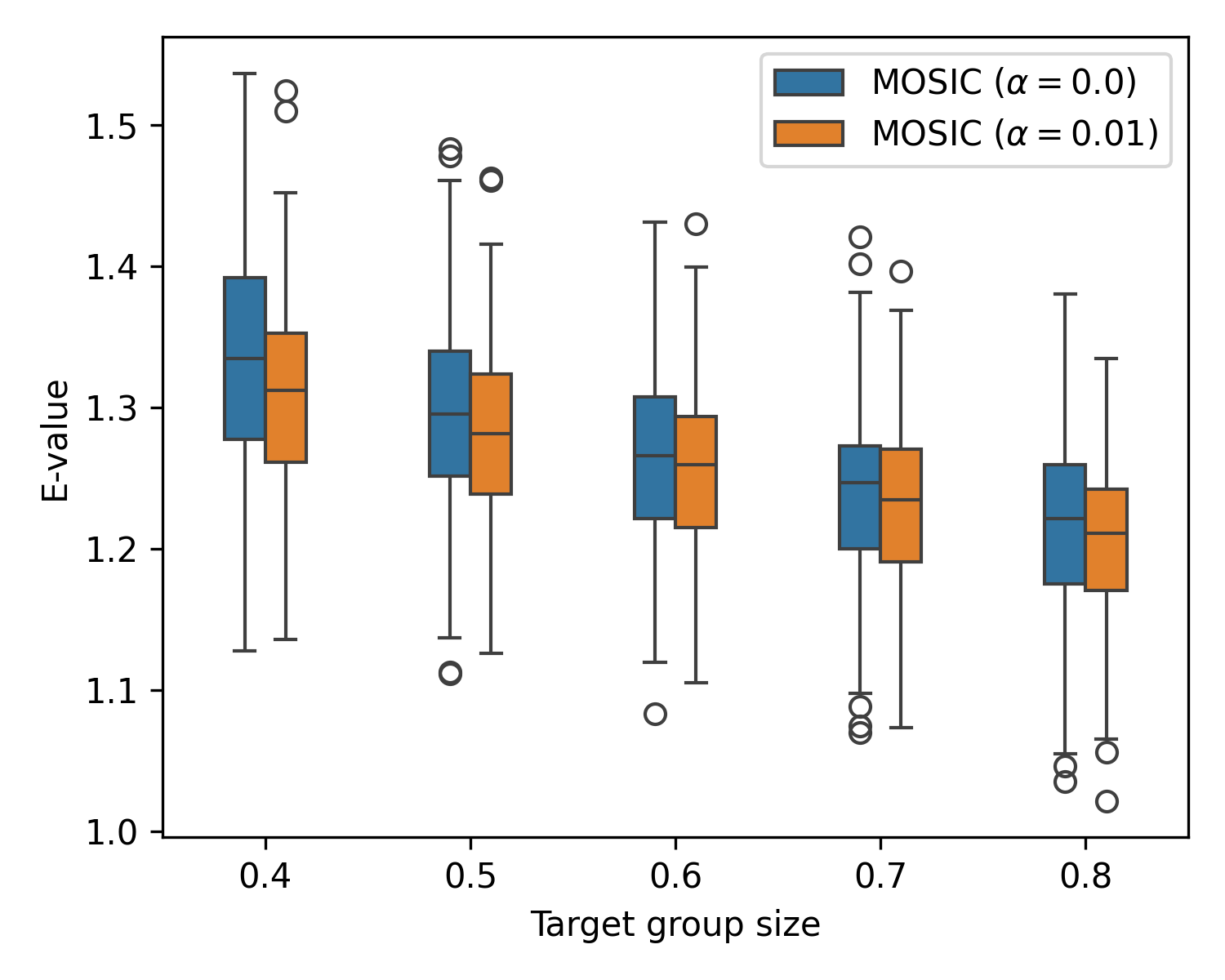}
    }
    \vspace{-5pt}
    \caption{Sensitivity analysis on unmeasured confounding}
    \vspace{-5pt}\label{fig:evalue}
\end{figure}


\subsection{Additional Baselines}\label{appx:more_baslines}
We further compare against DR-learner, R-learner, BART, and an overlap-weighted variant of MOSIC (MOSIC-OW). We adopted the Python package causalml to implement DR-learner and R-learner, using Random Forest as their base learner. For BART, we adopted the R package bartCause. For hyperparameters, DR-learner and R-learner consider $max\_depth \in \{3, 5, 7, 10\}$, and BART considers the number of prior standard deviations $k \in \{1, 2, 3\}$. The hyperparameter tuning procedures are the same as described in Section~\ref{subsec:setting}.  

\begin{figure*}[!t]
    \centering
    \subfigure[Unconfounded($\tilde{\omega}=0$)]{
        \includegraphics[width=0.48\linewidth]{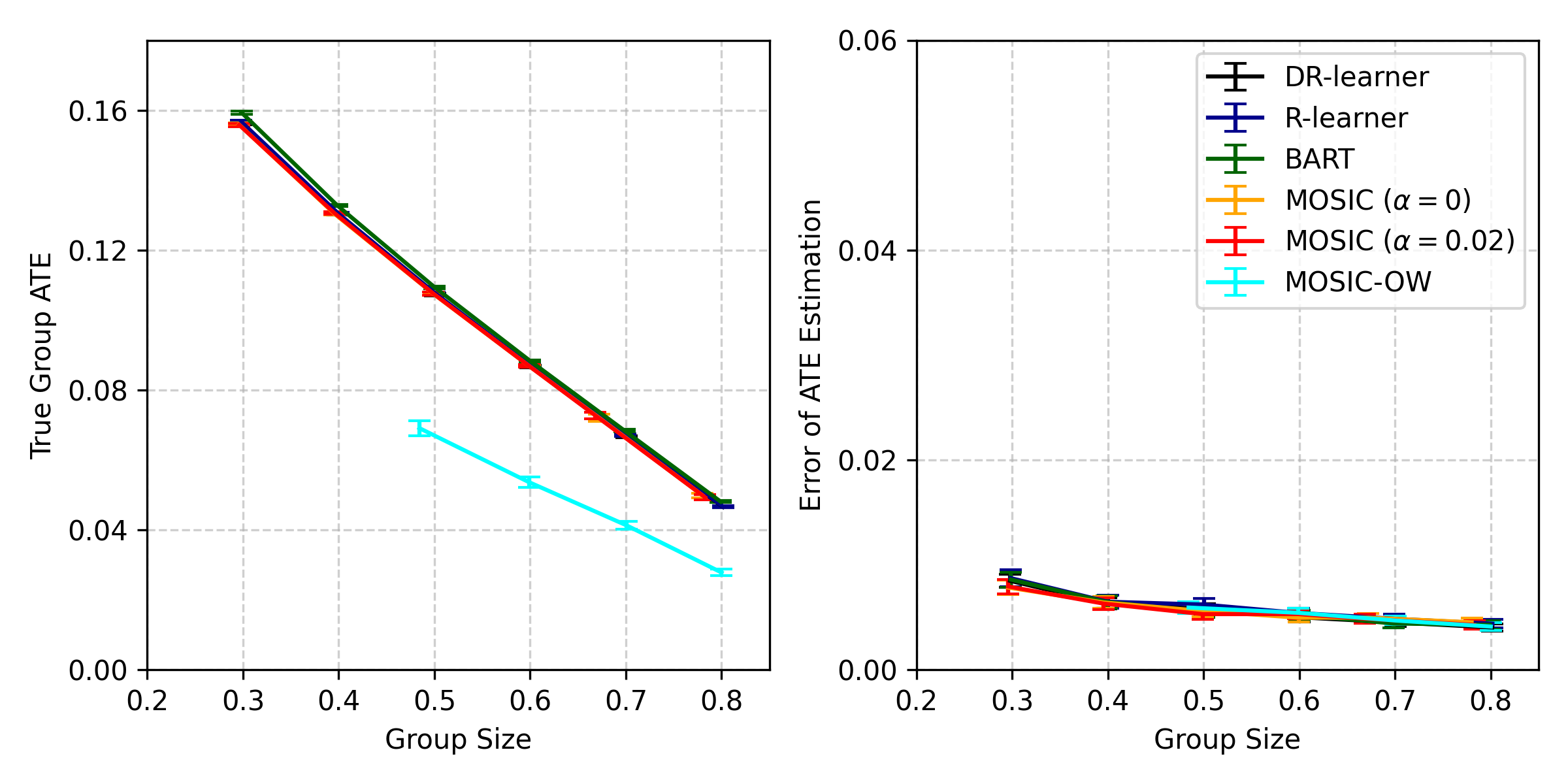}
    }
    \subfigure[Confounded($\tilde{\omega}=5$)]{
        \includegraphics[width=0.48\linewidth]{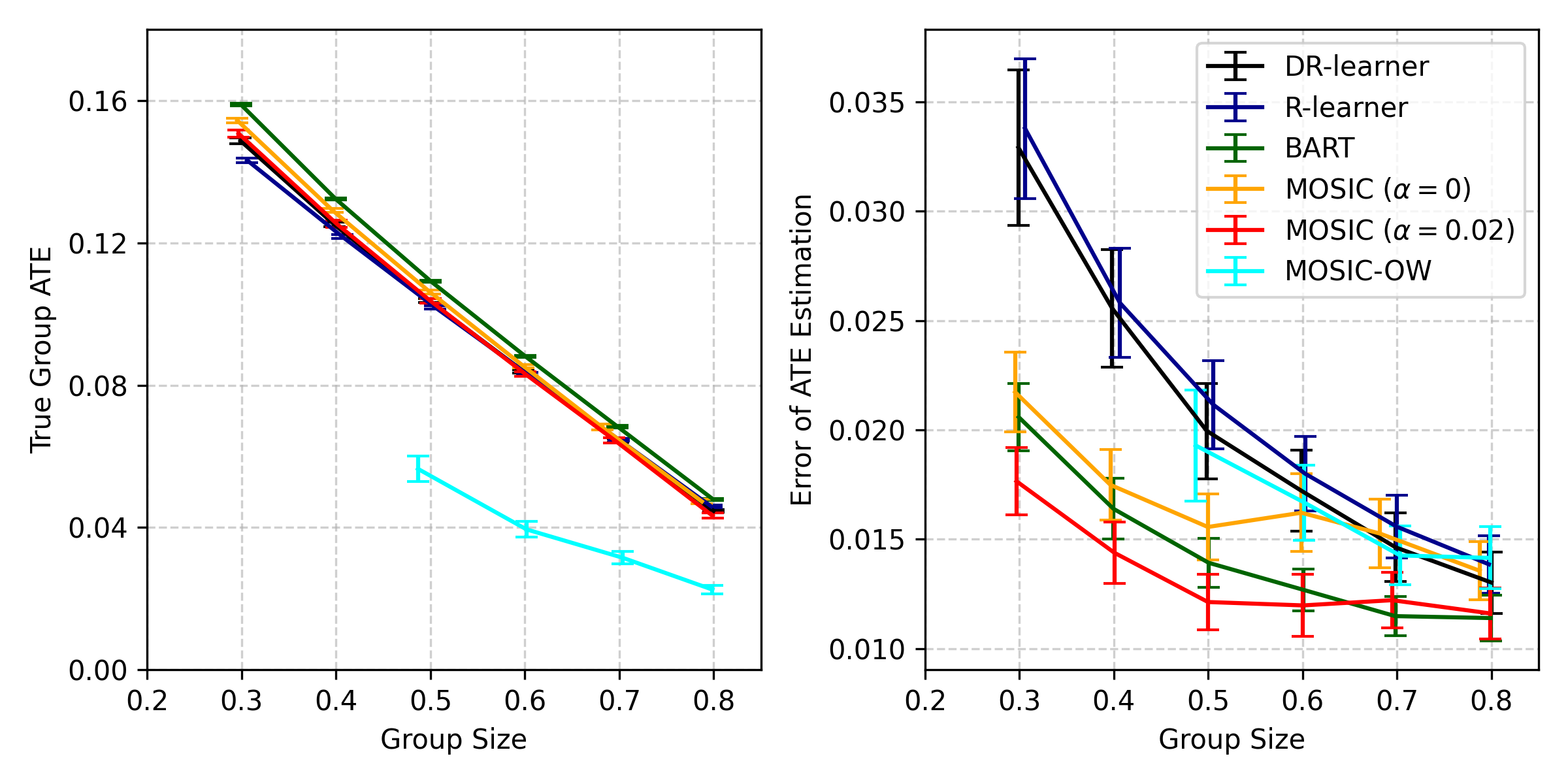}
    }\hfill

    \subfigure[eICU]{
        \includegraphics[width=0.48\linewidth]{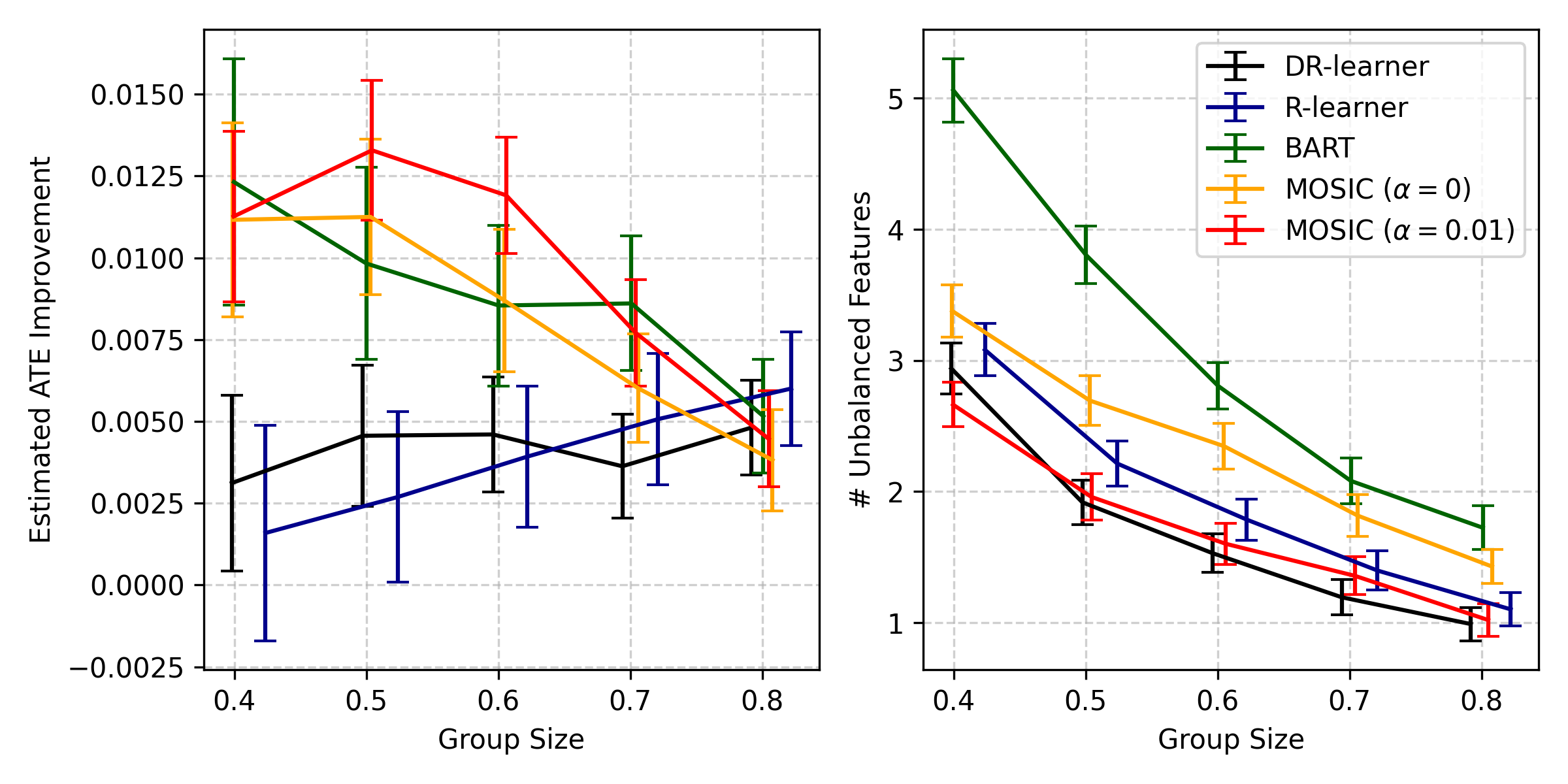}
    }\hfill
    \subfigure[MIMIC-IV]{%
        \includegraphics[width=0.48\linewidth]{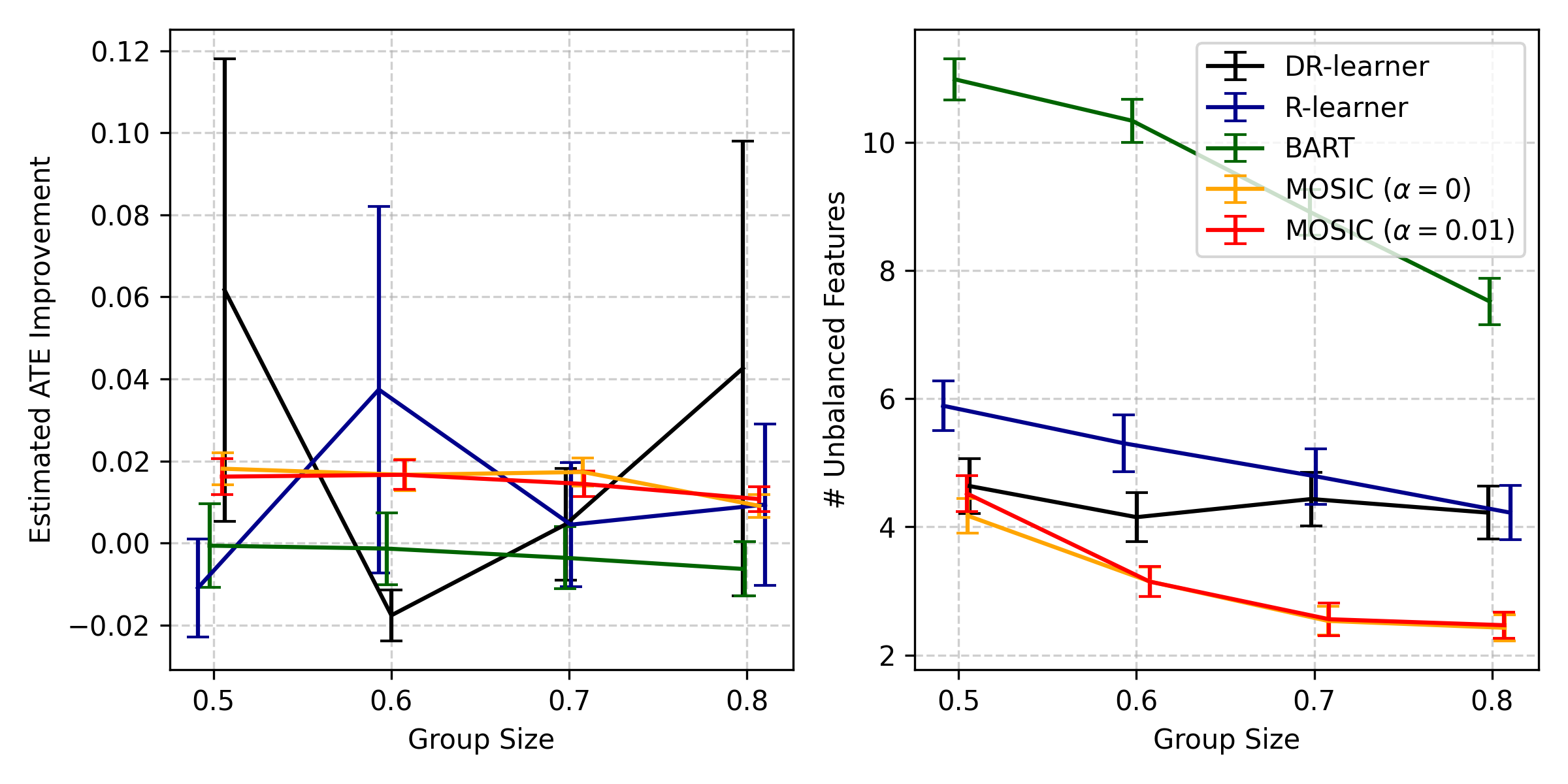}
    }
    \vspace{-5pt}
    \caption{Estimated ATE and the number of unbalanced features on real-world datasets. MOSIC-OW performs substantially worse in synthetic experiments, so we omit it from real-world comparisons.}
    \vspace{-5pt}
    \label{fig:additional_baseline}
\end{figure*}

\subsection{Overlap constraint on synthetic data}
We numerically verify that \methodname\ can indeed accommodate the overlap constraint (Figure \ref{subfig:syn-constraint}). 
Without overlap constraints, \methodname\ selects patients with the highest ITEs, including those in the non-overlap region. With overlap constraints (Figure~\ref{subfig:syn-constraint}, left), \methodname\ continues to select patients with large ITEs but systematically excludes those who violate the overlap constraint ($\hat{e}(x)<0.05$ or $\hat{e}(x)>0.95$).


\subsection{Test Set Overlap on Real-world Data}\label{subsec:additional_overlap}
We evaluated overlap on the held-out test sets and report the proportions of samples with estimated propensity scores falling outside [0.01,0.99], [0.02,0.98], and [0.05,0.95] within the subgroups selected on the test sets (Figure \ref{fig:overlap_evaluation_0.05} and ~\ref{fig:overlap_evaluation_alpha}). When no overlap constraint is imposed ($\alpha=0$), the proportion of low-overlap samples in the selected subgroup closely matches that of the full test set. In contrast, when overlap constraints are activated ($\alpha>0$), these proportions decrease substantially across all thresholds, demonstrating that MOSIC’s overlap constraint effectively improves overlap in the selected subgroups during evaluation as well as training.

\begin{figure}[h!]
	\centering
	\subfigure[Confounded Synthetic Data($\tilde{\omega}=5$)]{
		\centering
		\includegraphics[width=0.48\columnwidth]{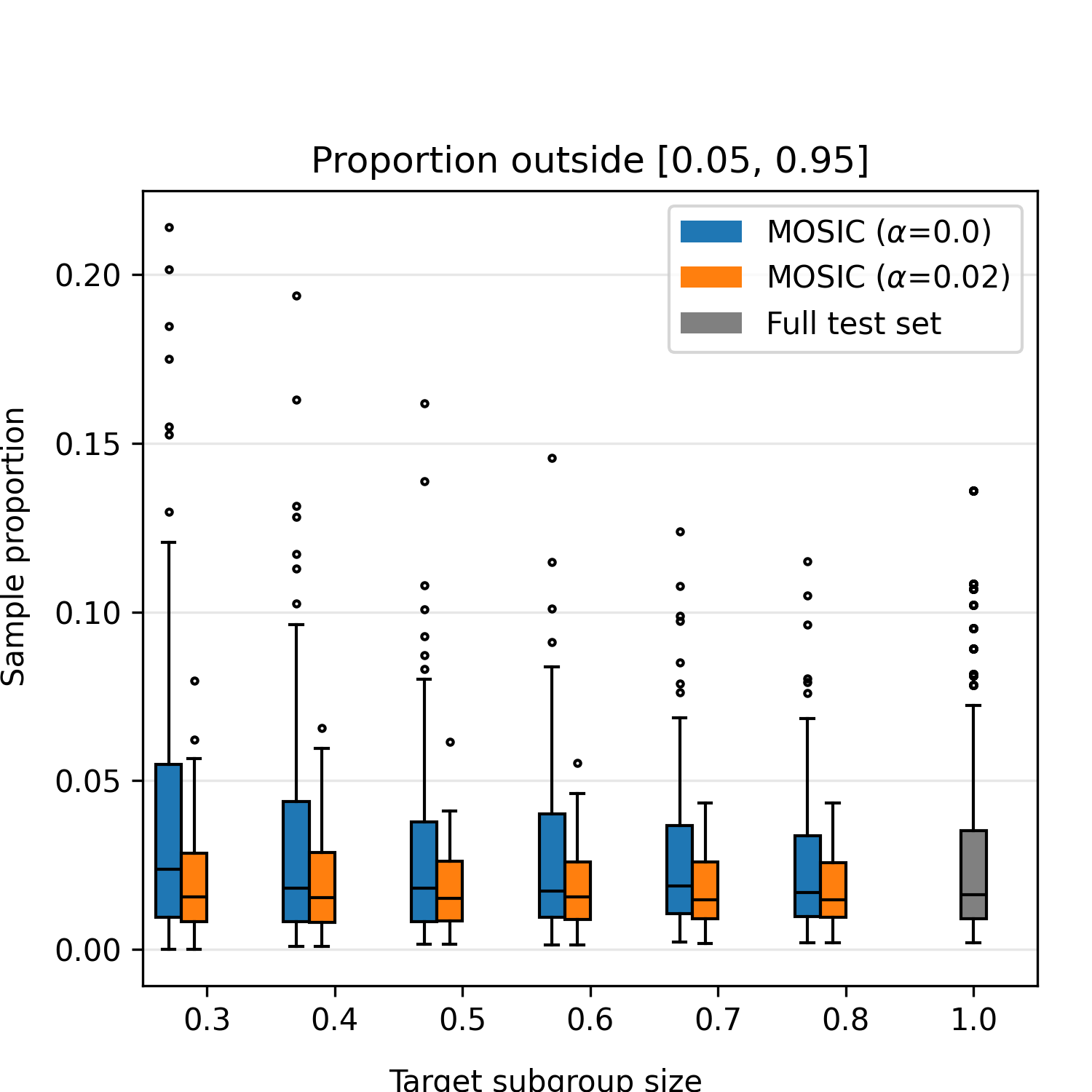}
	}\hfill
	\subfigure[eICU]{
		\centering
		\includegraphics[width=0.48\columnwidth]{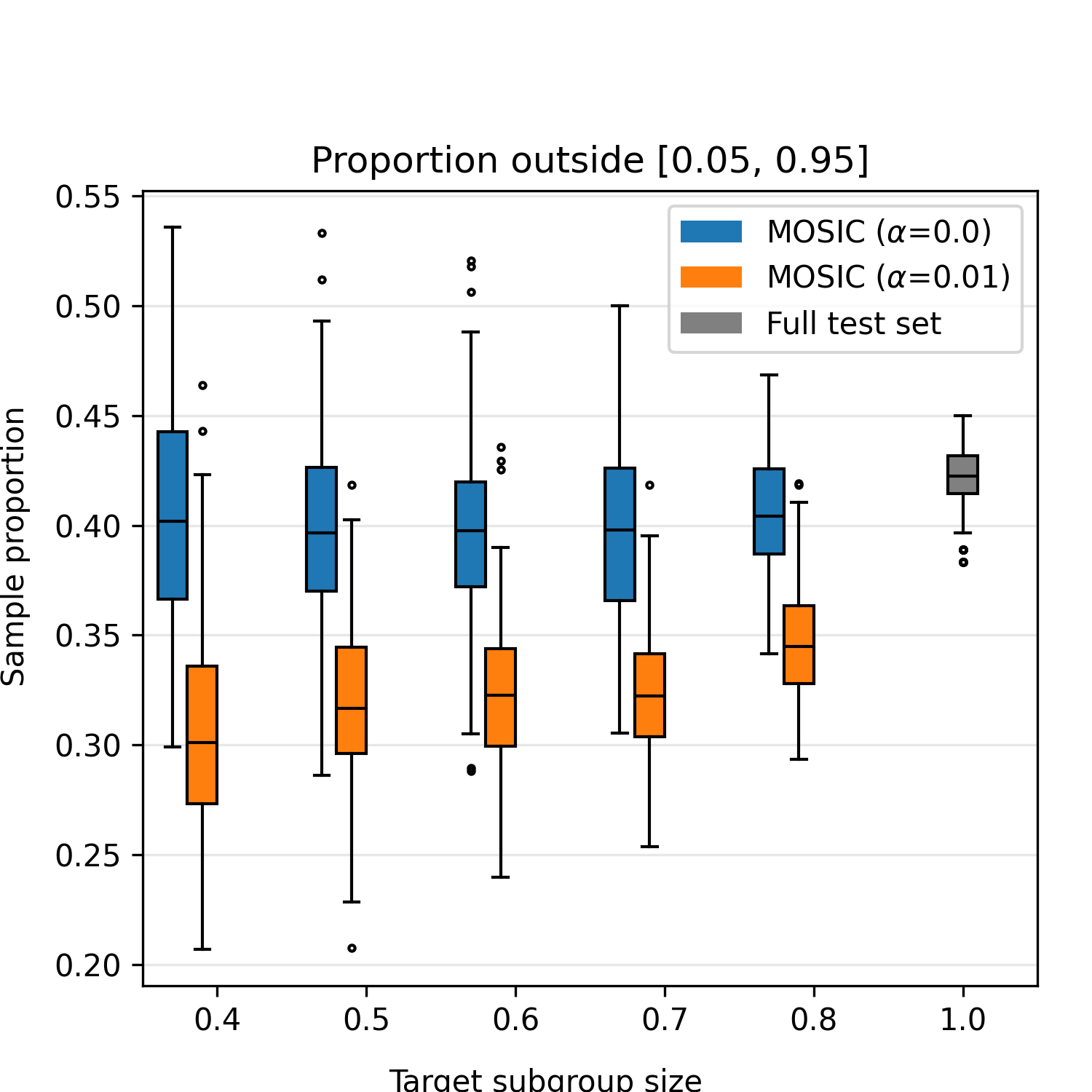}
	}%
	\caption{Overlap Evaluation on Test Set.}
    \label{fig:overlap_evaluation_0.05}
\end{figure}

\begin{figure}[!t]
    \centering
    \includegraphics[width=1.0\columnwidth]{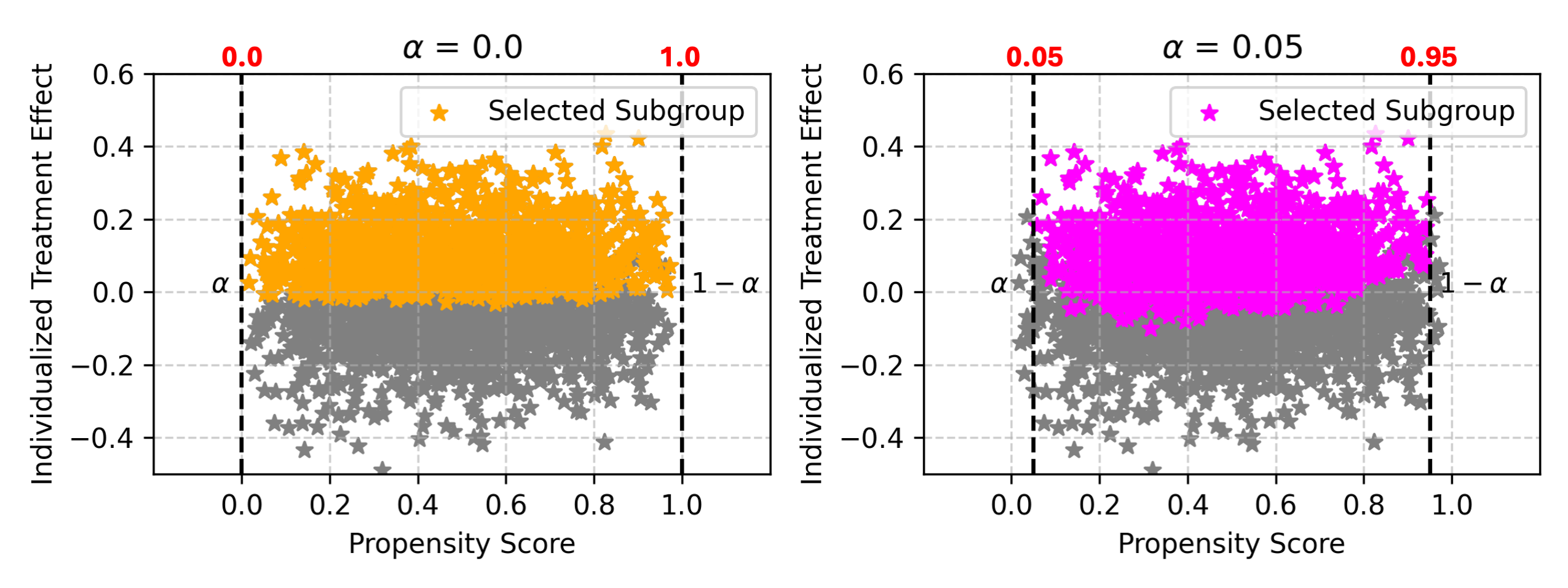}
	\caption{Results on synthetic data with confounding bias ($\tilde{\omega} = 5$). Each dot represents an individual, with the x-axis denoting the estimated propensity score, the y-axis representing the true individual treatment effect (ITE), and the vertical lines indicating the desired overlap threshold.}
    \label{subfig:syn-constraint}
\end{figure}

\subsection{\# Unbalanced features on real-world data using SMD>0.1 as threshold}\label{appx:subsec_smd01}

SMD$>$0.1 is a stricter threshold to evaluate feature imbalance and commonly used in epidemiology studies. However, Austin (2009) notes that "For modest sample sizes, one could expect standardized differences that exceed 0.20 (20 percent) even when the propensity-score model was correctly specified." Given that both of our real-world cohorts are relatively small (13,361 patients in eICU and 6,516 in MIMIC), we chose SMD $>$ 0.2 as the primary threshold to reduce the risk of flagging spurious imbalance driven by limited sample size. This threshold is also frequently used in epidemiology studies.

We now report the results using the stricter SMD $>$ 0.1 criterion, and it does not alter our main conclusions. Although the absolute number of unbalanced covariates increases across all methods due to inherent data size limitations, the relative comparison remains the same: MOSIC achieves comparable balance while yielding substantially higher subgroup ATEs (Figure~\ref{fig:n_unbalance_smd_01}, Table~\ref{tab:smd01_balance_sig} and Figure~\ref{fig:real-world}). (While Dragonnet demonstrates significantly lower feature imbalance, our method obtains significantly higher subgroup ATE improvement than Dragonnet, as shown in Figure 2b.)
\begin{figure}[t]
    \centering
    \subfigure[eICU]{%
        \includegraphics[width=0.48\linewidth]{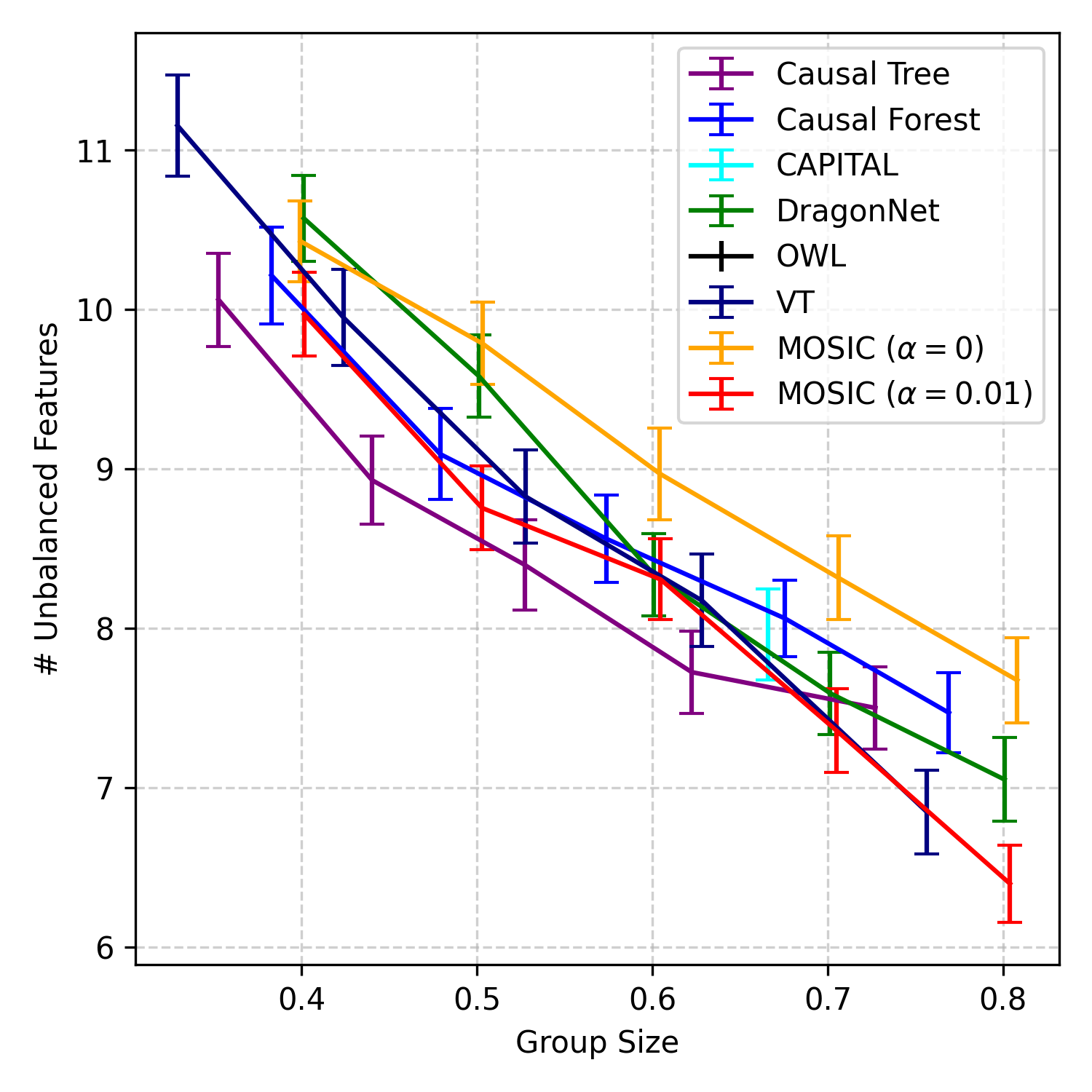}
    }\hfill
    \subfigure[MIMIC]{%
        \includegraphics[width=0.48\linewidth]{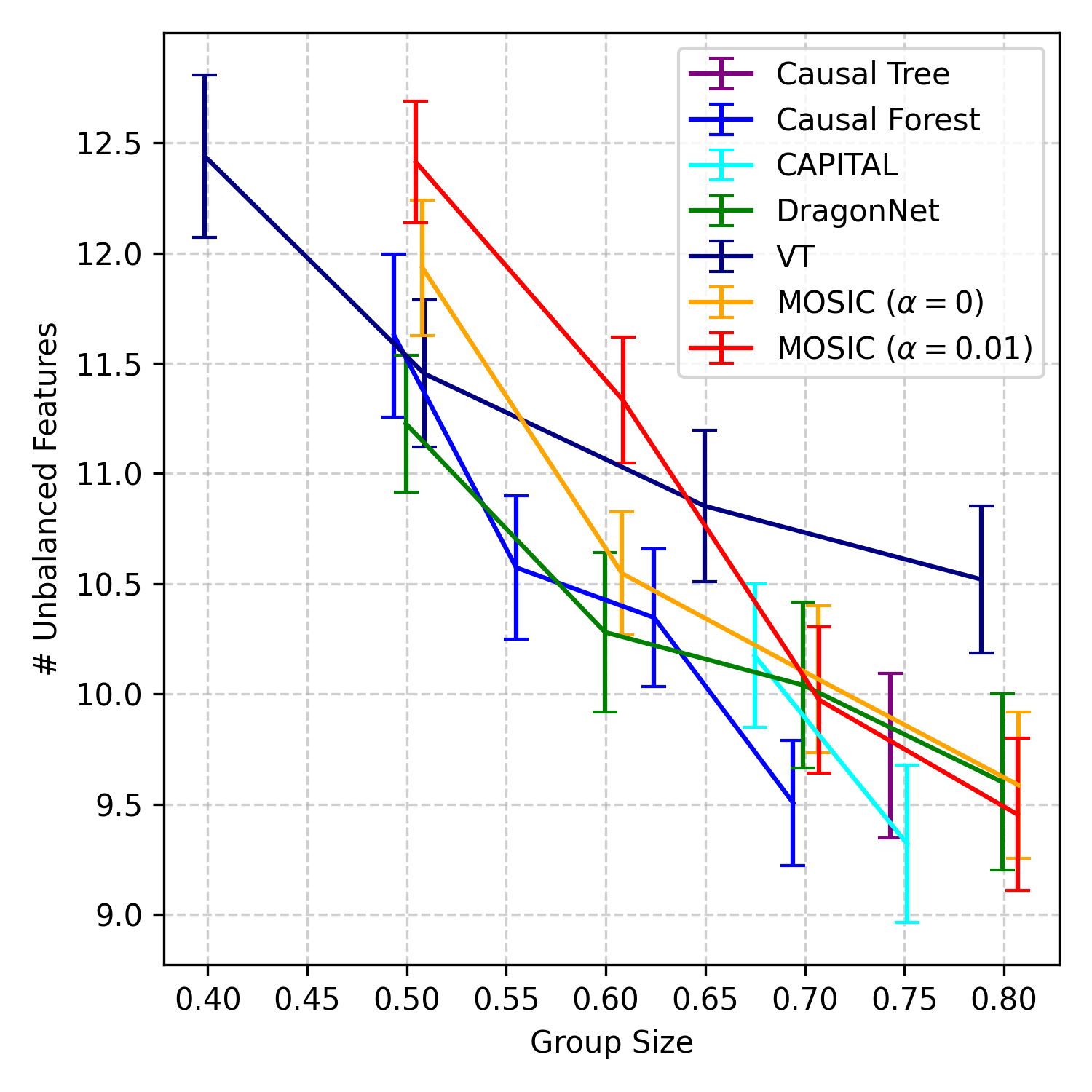}
    }
    \vspace{-5pt}
    \caption{Number of unbalanced features using SMD$>$0.1 as threshold.}
    \vspace{-5pt}\label{fig:n_unbalance_smd_01}
\end{figure}

\begin{table}[t]
\centering
\setlength{\tabcolsep}{4pt}
\caption{P-values for ATE and feature balance comparisons between \methodname\ ($\alpha=0.01$) and baseline methods.}
\label{tab:combined_sig_test}
\resizebox{0.9\columnwidth}{!}{
\begin{tabular}{llccccc}
\toprule
& & $c=0.4$ & $c=0.5$ & $c=0.6$ & $c=0.7$ & $c=0.8$ \\
\midrule

\multicolumn{7}{c}{\textbf{eICU}}\\
\midrule

\multirow{6}{*}{ATE}
& CT         & 0.71 & 0.065 & 0.0011 & 0.0046 & 0.025 \\
& CF         & 0.19 & 0.0048 & 0.00019 & 0.0012 & 0.0057 \\
& CAPITAL    & - & - & - & 3.5E-05 & - \\
& Dragonnet  & 0.30 & 0.0057 & 5.5E-04 & 0.011 & 0.11 \\
& OWL        & 0.27 & 0.0019 & 0.0010 & 0.084 & 0.16 \\
& VT         & 0.38 & 0.07 & 0.039 & 0.31 & 0.46 \\

\midrule

\multirow{6}{*}{Balance}
& CT         & 0.16 & 0.58 & 0.56 & 0.65 & 0.04 \\
& CF         & 0.35 & 0.12 & 0.48 & 0.47 & 0.16 \\
& CAPITAL    & - & - & - & 0.59 & - \\
& Dragonnet  & 0.022 & 9.2E-04 & 0.14 & 0.33 & 0.066 \\
& OWL        & 0.0088 & 0.0077 & 0.13 & 0.22 & 0.046 \\
& VT         & 0.0011 & 0.0060 & 0.036 & 0.23 & 0.39 \\

\midrule
\multicolumn{7}{c}{\textbf{MIMIC-IV}}\\
\midrule

\multirow{5}{*}{ATE}
& CT         & - & - & - & - & 0.02 \\
& CF         & - & 0.42 & 0.68 & 0.35 & 0.16 \\
& CAPITAL    & - & 0.021 & - & 0.30 & 0.72 \\
& Dragonnet  & - & 0.039 & 0.10 & 0.21 & 0.30 \\
& VT         & - & 0.10 & 0.16 & 0.15 & 0.15 \\

\midrule

\multirow{5}{*}{Balance}
& CT         & - & - & - & - & 0.63 \\
& CF         & - & 0.19 & 5.1E-04 & 1.3E-04 & 1.1E-04 \\
& CAPITAL    & - & 6.3E-07 & - & 0.17 & 3.3E-05 \\
& Dragonnet  & - & 0.062 & 0.80 & 0.28 & 0.48 \\
& VT         & - & 0.033 & 0.0020 & 0.018 & 0.20 \\

\bottomrule
\end{tabular}
}
\end{table}

\begin{table}[t]
\centering
\caption{P-values for feature balance comparisons between \methodname\ and baseline methods using SMD$>0.1$ as the threshold for an unbalanced feature.}
\label{tab:smd01_balance_sig}
\resizebox{\columnwidth}{!}{
\begin{tabular}{llccccc}
\toprule
& & $c=0.4$ & $c=0.5$ & $c=0.6$ & $c=0.7$ & $c=0.8$ \\
\midrule

\multicolumn{7}{c}{\textbf{eICU}} \\
\midrule

& CAPITAL   & --     & --      & 0.36 & --      & --      \\
& CF        & 0.54   & 0.39    & 0.50 & 0.050   & 0.0024  \\
& CT        & 0.82   & 0.65    & 0.81 & 0.32    & 0.0022  \\
& Dragonnet & 0.11   & 0.027   & 0.93 & 0.52    & 0.069   \\
& VT        & 0.0044 & 0.0033  & 0.18 & 0.04    & 0.21    \\

\midrule

\multicolumn{7}{c}{\textbf{MIMIC-IV}} \\
\midrule

& CAPITAL   & --      & 1.8E-10 & --    & 0.67 & --    \\
& CF        & --      & 0.091   & 0.081 & 0.41 & 0.91  \\
& CT        & --      & --      & --    & --   & 0.60  \\
& Dragonnet & --      & 0.0049  & 0.024 & 0.89 & 0.78  \\
& VT        & --      & 0.095   & 0.79  & 0.067 & 0.028 \\

\bottomrule
\end{tabular}
}
\end{table}





\subsection{Additional Balance Diagnostics}
\label{appx:subsec_additional_balance}
To further evaluate the reliability of the identified subgroups, we report three additional diagnostics on the alternative outcome setting, ventilation-free days within 7 days of ICU admission: effective sample size (ESS), the proportion of samples with propensity scores outside ([0.05,0.95]) (\% Extreme), and the maximum inverse propensity weight (Max IPW). Consistent with the subgroup ATE and feature-balance results reported in Table~\ref{tab:vfd_results}, MOSIC achieves higher ESS, lower proportions of extreme propensity scores, and smaller maximum inverse propensity weights across subgroup sizes. These complementary metrics further support that MOSIC identifies subgroups with improved overlap and more reliable treatment-effect estimation.

\begin{table*}[t]
\centering
\caption{Additional reliability diagnostics on the ventilation-free-days outcome. Higher ESS and lower \% Extreme / Max IPW indicate better overlap.}
\label{tab:additional_reliability}
\resizebox{\textwidth}{!}{
\begin{tabular}{lccc|ccc|ccc}
\toprule
& \multicolumn{3}{c|}{ESS}
& \multicolumn{3}{c|}{\% Extreme}
& \multicolumn{3}{c}{Max IPW} \\
\cmidrule(lr){2-4}
\cmidrule(lr){5-7}
\cmidrule(lr){8-10}
Method
& $c=0.6$ & $c=0.7$ & $c=0.8$
& $c=0.6$ & $c=0.7$ & $c=0.8$
& $c=0.6$ & $c=0.7$ & $c=0.8$ \\
\midrule
OWL
& 347.27$\pm$5.03 & 401.11$\pm$5.35 & 454.09$\pm$5.92
& 0.24$\pm$0.004 & 0.24$\pm$0.004 & 0.25$\pm$0.003
& 45.88$\pm$1.68 & 47.80$\pm$1.79 & 48.70$\pm$1.77 \\

VT
& 327.13$\pm$4.21 & 378.74$\pm$5.07 & 430.49$\pm$5.43
& 0.25$\pm$0.003 & 0.25$\pm$0.003 & 0.25$\pm$0.003
& 44.97$\pm$1.32 & 47.04$\pm$1.61 & 49.11$\pm$1.73 \\

CT
& 298.06$\pm$4.74 & 360.14$\pm$5.27 & 422.01$\pm$5.46
& 0.24$\pm$0.004 & 0.24$\pm$0.003 & 0.25$\pm$0.003
& 41.74$\pm$1.27 & 43.79$\pm$1.30 & 46.58$\pm$1.45 \\

CF
& 342.29$\pm$4.96 & 392.13$\pm$5.37 & 441.76$\pm$5.75
& 0.23$\pm$0.003 & 0.24$\pm$0.003 & 0.25$\pm$0.003
& 45.13$\pm$1.47 & 46.71$\pm$1.44 & 49.26$\pm$1.74 \\

Dragonnet
& 332.98$\pm$4.72 & 388.79$\pm$5.22 & 441.91$\pm$5.44
& 0.28$\pm$0.00 & 0.27$\pm$0.00 & 0.27$\pm$0.00
& 46.53$\pm$1.47 & 47.49$\pm$1.44 & 48.07$\pm$1.41 \\

\midrule

MOSIC-MLP ($\alpha=0.01$)
& 384.80$\pm$5.63 & 438.94$\pm$5.55 & 488.80$\pm$5.55
& 0.20$\pm$0.003 & 0.20$\pm$0.003 & 0.22$\pm$0.002
& 43.30$\pm$1.75 & 44.88$\pm$1.70 & 46.08$\pm$1.09 \\

MOSIC-DT ($\alpha=0.01$)
& 459.98$\pm$5.54 & 500.53$\pm$6.27 & 518.43$\pm$6.14
& 0.19$\pm$0.003 & 0.20$\pm$0.004 & 0.21$\pm$0.003
& 40.53$\pm$0.87 & 47.49$\pm$1.44 & 48.07$\pm$1.41 \\

MOSIC-Forest ($\alpha=0.01$)
& 489.73$\pm$4.68 & 520.97$\pm$5.87 & 531.92$\pm$5.79
& 0.17$\pm$0.003 & 0.18$\pm$0.004 & 0.20$\pm$0.004
& 45.13$\pm$1.47 & 46.71$\pm$1.44 & 49.26$\pm$1.74 \\

\bottomrule
\end{tabular}
}
\end{table*}

\subsection{Extension to Safety, Budget, and Fairness Constraint on Synthetic Data}\label{appx:syn_more_constraints}

To demonstrate MOSIC's flexibility to handle additional constraints, we extend the synthetic data generation process described in Appendix~\ref{subsec:synthetic} by introducing a safety, budget, and fairness constraint.
In particular:
\begin{itemize}
    \item Safety constraint: Following~\citet{doubleday2022risk}, each sample is assigned a risk score $r_i = 1/(1+\exp(10 * x_i[10]+1))$. We require the average risk of the selected subgroup to be no greater than 0.05: $$\frac{\sum_{i=1}^n\mathds{1}(S(x_i)>0.5)r_i}{\sum_{i=1}^n\mathds{1}(S(x_i)>0.5)} \leq 0.05.$$
    
    \item Budget constraint: Following~\citet{qiu2022individualized}, each sample is assigned a treatment cost value $cost_i = (x_i[3]+5)/5$. The total cost of the selected subgroup must not exceed half the cost of treating the entire population (assuming a unit cost per sample): $$\sum_{i=1}^n\mathds{1}(S(x_i)>0.5)cost_i \leq 0.5*n.$$ The test contains 2500 samples, so the total cost limit $0.5*n=1250$ in this case.
    
    \item Fairness constraint: Let a binary sensitive attribute: $sens_i = \mathds{1}(x[3]>0.5)$. We adopt the conditional statistical parity metric~\cite{mehrabi2021survey}. In our setting, this corresponds to maintaining a sensitive‑group proportion of 0.5 in the selected subgroup: $$\Big|\frac{\sum_{i=1}^n\mathds{1}(S(x_i)>0.5)sens_i}{\sum_{i=1}^n\mathds{1}(S(x_i)>0.5)} - 0.5\Big| \leq 0.01,$$
    where we allow a violation of 0.01. We can then convert this to two ratio-form constraints:
    $$-0.01 \leq \frac{\sum_{i=1}^n\mathds{1}(S(x_i)>0.5)sens_i}{\sum_{i=1}^n\mathds{1}(S(x_i)>0.5)} - 0.5 \leq 0.01.$$
\end{itemize}
In addition to the group size constraint ($c=0.5$) and overlap constraint ($\alpha=0.02$), we progressively add the following constraints: 1) Plus safety constraint; 2) Plus safety and budget constraint; 3) Plus safety, budget, and fairness constraint. Results in Table~\ref{tab:syn_more_constraints} show that \methodname\ can effectively enforce all of these constraints.


\subsection{Extension to Safety Constraint on eICU}\label{appx:eicu_cns_constraint}

To impose the interpretability requirement, we implement MOSIC-DT for this setting. Similarly, we run experiments on 100 random train-test splits and use 5-fold cross-validation to select the tree depth among \{3,5,7\}. Results in Table~\ref{table:eicu_cns4} show MOSIC can effectively exclude patients with GCS$<6$. Additionally, an example (Figure~\ref{subfig:eicu_dt_w_cns}) shows that MOSIC indeed learned to exclude such high-risk patients. 

We additionally run MOSIC-MLP with and without the GCS$<6$ constraints and run post-hoc SHAP-value analysis to evaluate the feature contributions. The dominant features identified by SHAP are generally consistent with the rule structure revealed by the decision tree. In particular, the SHAP values of the GCS features were ambiguous before we enforced the constraint that avoids patients with GCS $<6$, but they became clearly separated after we enforced it, aligning well with our intention. 

We further evaluate subgroup stability across different initializations. For the same train–test split, we reran MOSIC-MLP six times with the same constraints but different random initializations and examined whether the resulting feature contributions remained consistent. As shown in Figure~\ref{fig:model_stablility}, the SHAP value patterns are highly similar across runs, suggesting that the learned subgroups are stable with respect to initialization and all align well with the safety constraint.


\begin{figure*}[t]
    \centering
    \subfigure[Seed 1]{%
        \includegraphics[width=0.32\linewidth]{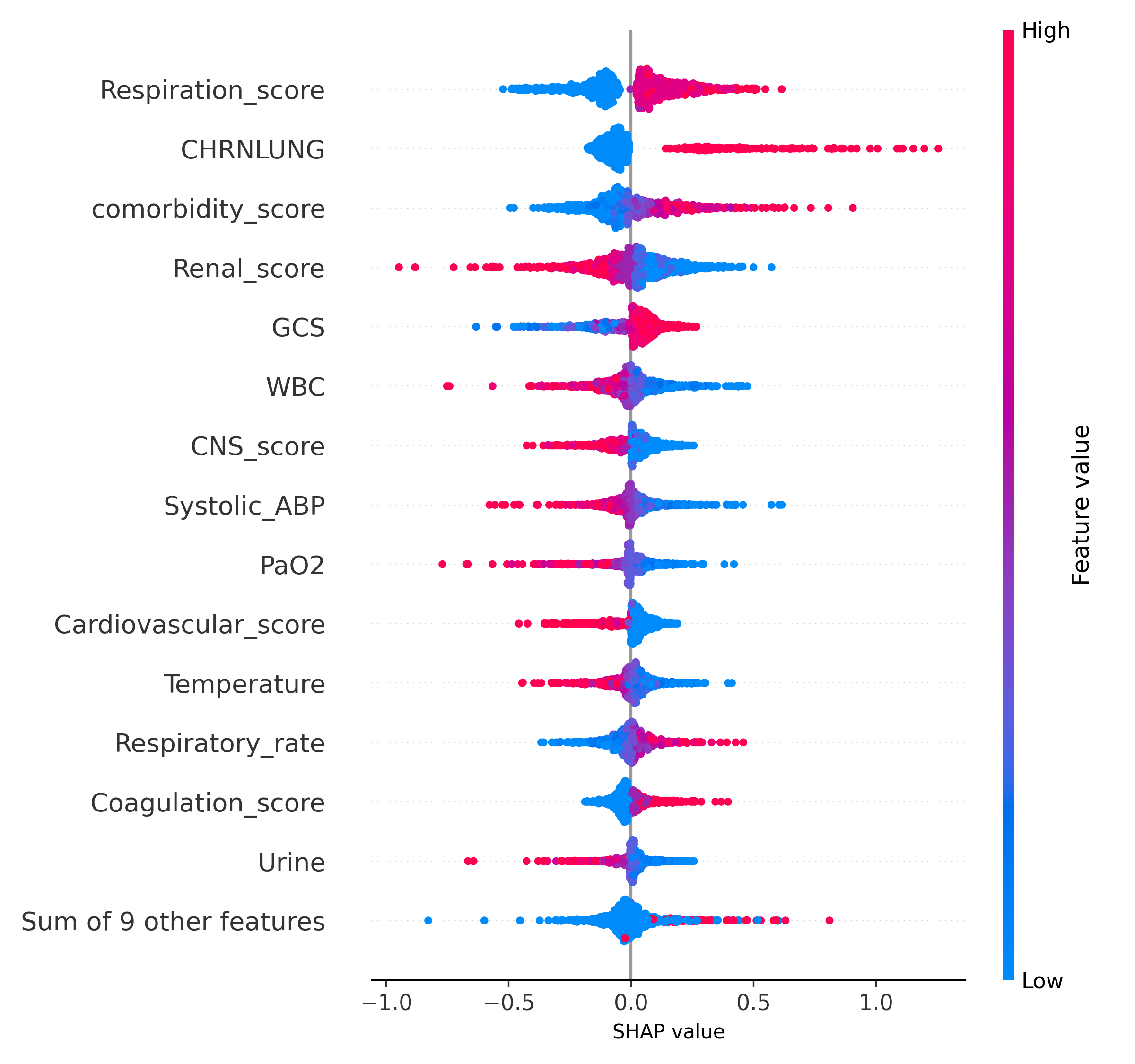}
    }\hfill
    \subfigure[Seed 2]{%
        \includegraphics[width=0.32\linewidth]{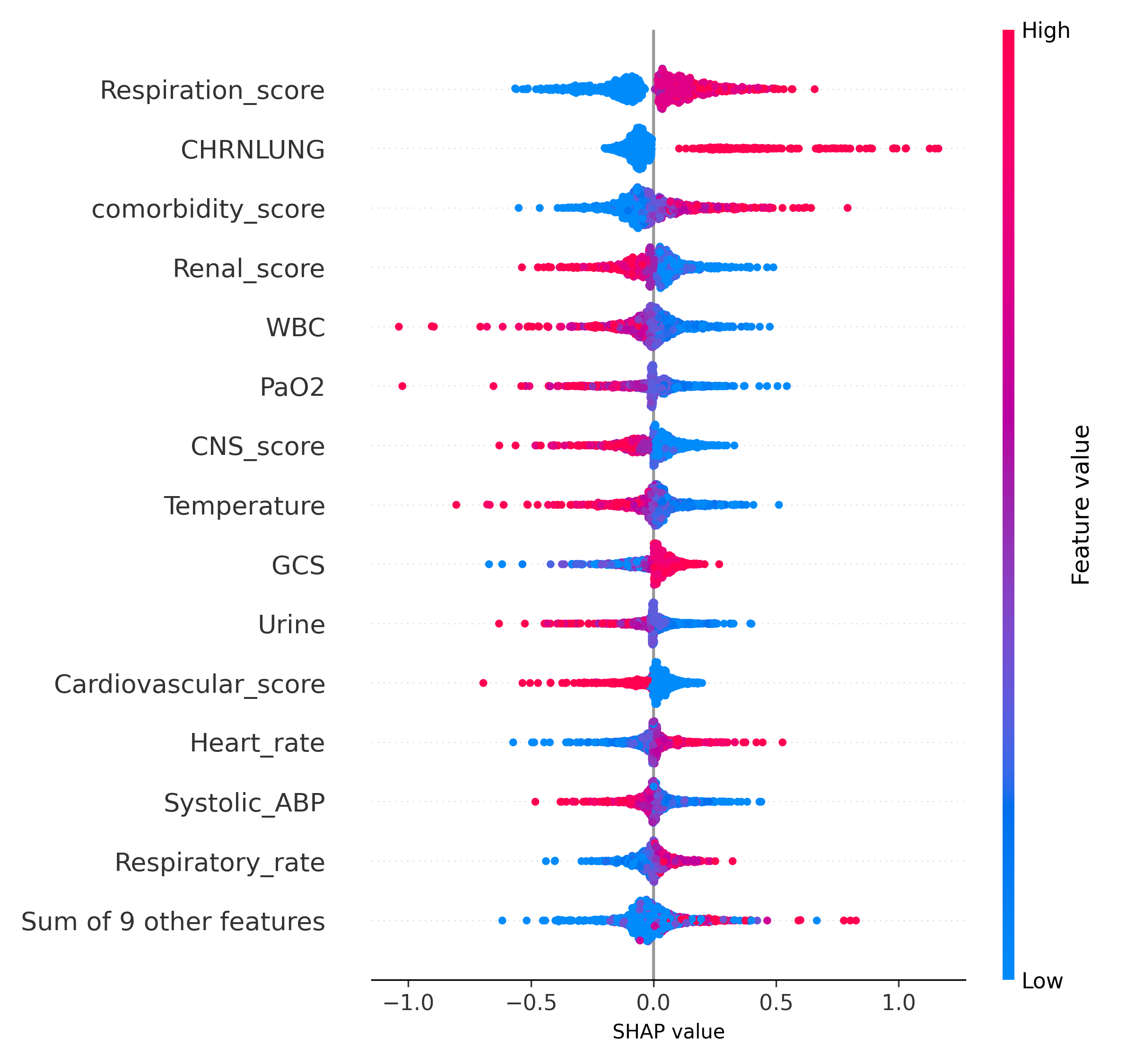}
    }\hfill
    \subfigure[Seed 3]{%
        \includegraphics[width=0.32\linewidth]{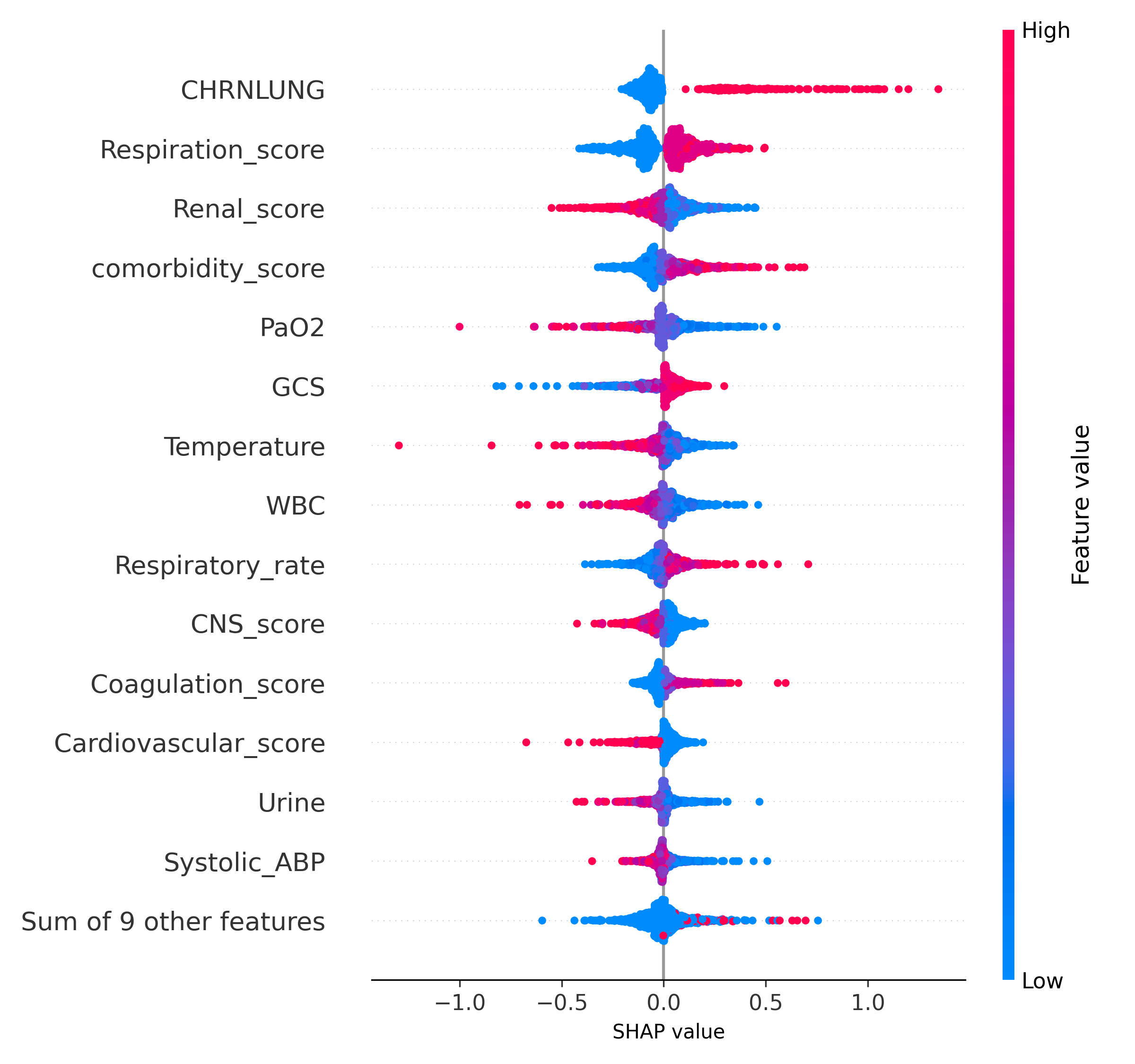}
    }

    \subfigure[Seed 4]{%
        \includegraphics[width=0.32\linewidth]{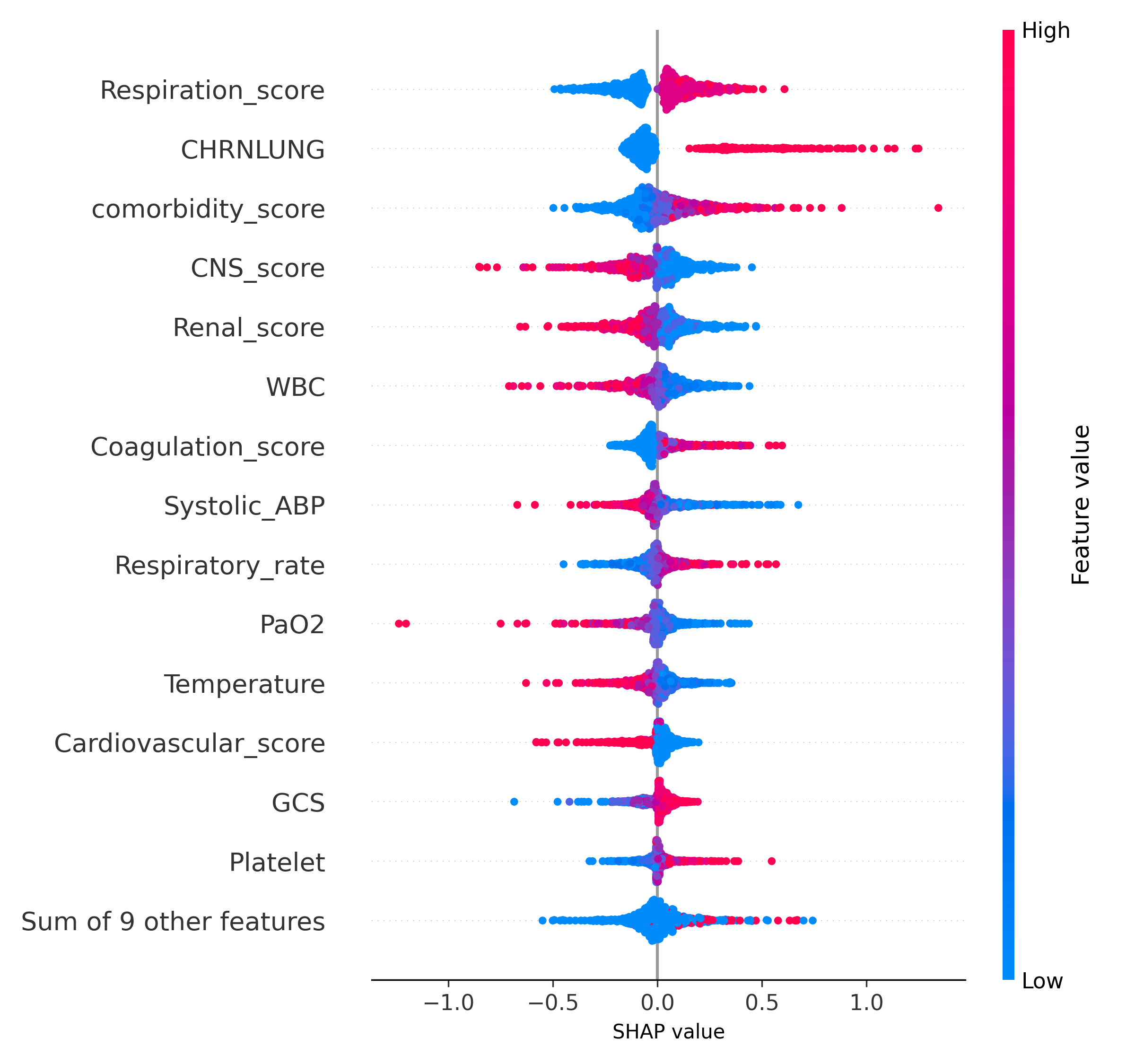}
    }\hfill
    \subfigure[Seed 5]{%
        \includegraphics[width=0.32\linewidth]{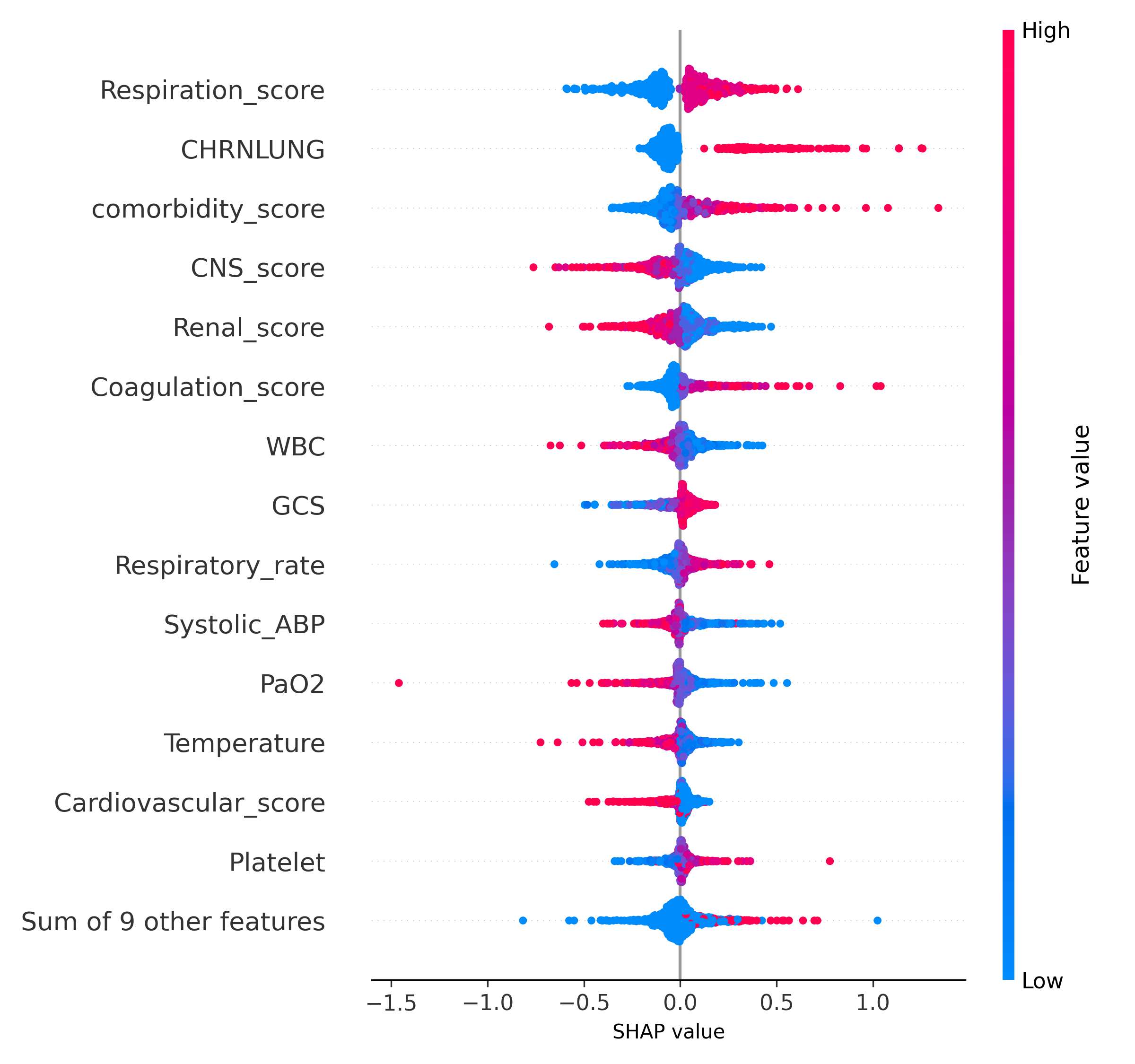}
    }\hfill
    \subfigure[Seed 6]{%
        \includegraphics[width=0.32\linewidth]{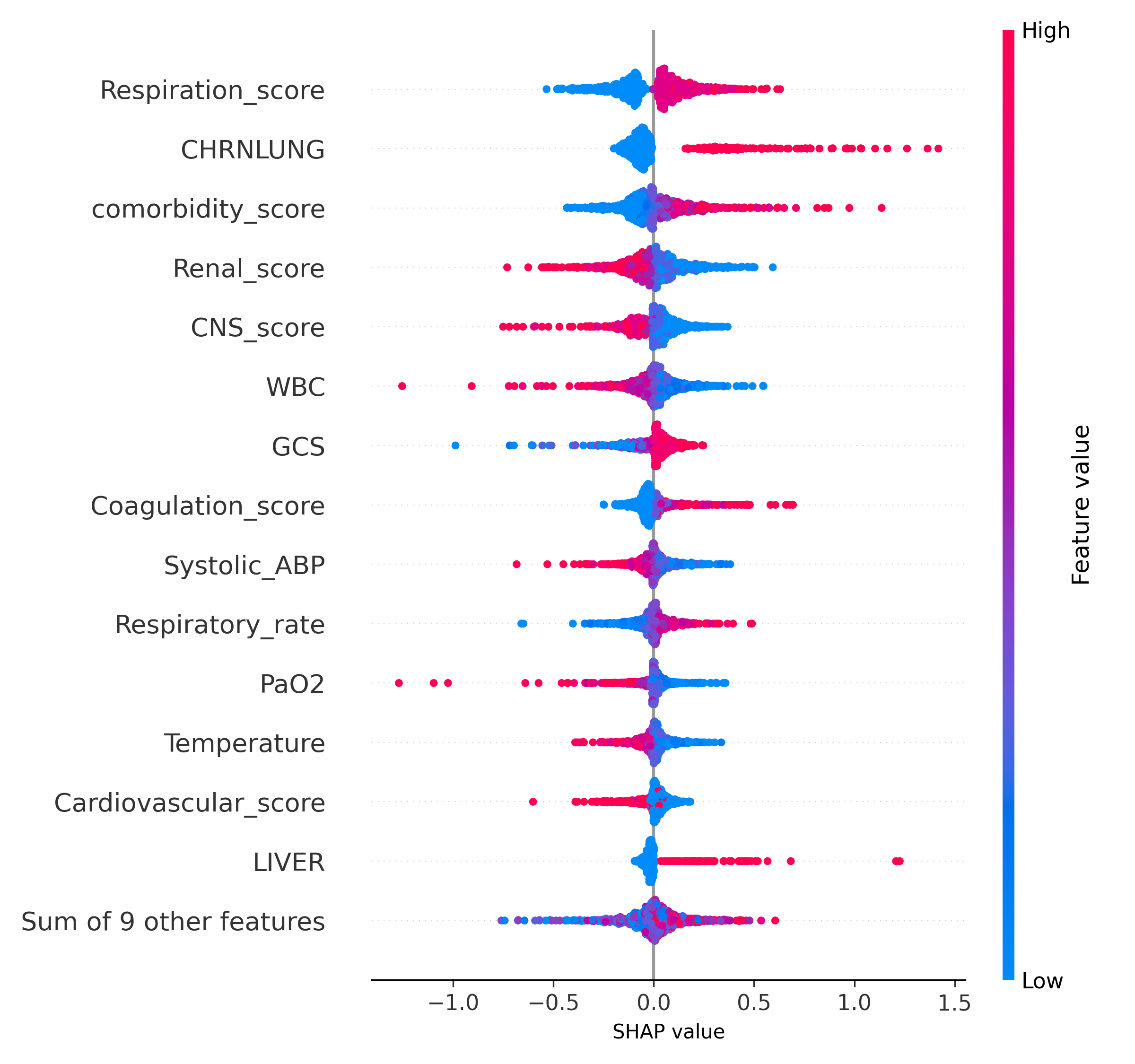}
    }
    \vspace{-5pt}
    \caption{Model Stability across different initialization}
    \vspace{-5pt}\label{fig:model_stablility}
\end{figure*}

\section{Evaluation on Binary Subgroup Setting}\label{sec:binary_subgroup}
Unlike approaches that assume the existence of two or a finite number of subgroups with distinct ATEs, our study focuses on identifying a subset of the population with maximal ATE under real-world constraints, without making structural assumptions about the underlying heterogeneity. This design enables our method to generalize to continuous or complex heterogeneity. Nevertheless, we recognize that such structural assumptions, such as that binary subgroups with distinct ATE exist, are plausible in certain real-world scenarios. In such cases, the real-world constraint will naturally introduce a trade-off for identification performance. 

To illustrate, we modify the DGP of synthetic data by replacing $Y(1)$ in Appendix~\ref{subsec:synthetic} to:
$$Y(1) = (\sin(10 * \bm{X} ) + 5*\bm{X}^2)^T\beta_1 + \mathds{1}(\bm{X}^T > 0.05)\beta_\tau + \epsilon.$$
That is, only patients with covariates $> 0.05$ at positive $\beta_\tau$ indices receive a positive effect; others receive none. This yields a positive subgroup comprising ~68\% of samples. Using this DGP, we test MOSIC with beta = 1e-5, alpha = 0, and c in \{0.6, 0.7, 0.8\}. The precision and recall of subgroup identification are evaluated. Results in Table~\ref{table:binary_subgroup} are reported as mean $\pm$ standard deviation over 100 runs. Performance aligns with theory: precision/recall degrade when $c$ (the group size constraint) exceeds/falls below the true subgroup size, reflecting constraint-driven trade-offs.

\begin{table}[h]
	\centering
	\begin{tabular}{ccccc}
		\toprule
		c & ATE & Group Size & Precision & Recall \\
		\hline
        0.6 & 0.92$\pm$0.07 & 0.60$\pm$0.01 & 0.93$\pm$0.05 & 0.83$\pm$0.05 \\
        0.7 & 0.84$\pm$0.05 & 0.70$\pm$0.01 & 0.88$\pm$0.04 & 0.92$\pm$0.04 \\
        0.8 & 0.75$\pm$0.04 & 0.80$\pm$0.01 & 0.80$\pm$0.03 & 0.96$\pm$0.04 \\
		\bottomrule
	\end{tabular}
    \caption{Performance of binary subgroup identification} \label{table:binary_subgroup}
\end{table}
\section{Evaluation of Type I error}\label{sec:type_i_error}
Although our primary focus is on multi-constraint subgroup identification rather than statistical inference, we also evaluate Type I error of the identified subgroup. We use a data-splitting approach to test whether identified subgroups arise spuriously under the null hypothesis, where all individuals have zero treatment effect. Using data splitting, Type I error of the selected subgroup can be evaluated after we build the subgroup assignment model. Specifically: 
\begin{enumerate}
    \item Split the dataset (e.g., 50-50) into training (subgroup selection) and holdout (inference);
    \item Train MOSIC on the training set to learn the subgroup model;
    \item Apply the model to the holdout set to identify the subgroup, then compute its subgroup ATE, denoted as $ATE_{hold\_out}$;
    \item Test the null hypothesis on the holdout set:
        \begin{enumerate}
            \item Construct the distribution of subgroup ATE under the null (here by directly sampling from the test distribution, with bootstrap subsample size equaling the target subgroup size specified by the parameter $c$);
            \item Compare $ATE_{hold\_out}$ to this distribution, and determine whether the null hypothesis should be rejected.
        \end{enumerate} 
    \item Repeat steps 1-4 on additional synthetic data instances, aggregate results, and estimate type-I error.
\end{enumerate}

We generated 100 synthetic datasets using the same DGP in Appendix~\ref{subsec:synthetic} except that $\beta_\tau$ is set to $\bm{0}$. This aligns with the null hypothesis: all individuals have zero treatment effect. For each instance, we set the bootstrap iterations to 10000. Type I error rate is computed as the proportion of instances in which the null hypothesis is rejected at the 5\% significance level. 

As shown in Table~\ref{table:type_i_error}, the Type I error rate increases when the parameter $c$ is small, consistent with theoretical expectations. Smaller subgroups lead to higher variance in ATE estimates, highlighting a fundamental trade-off between real-world constraints and statistical reliability.

\begin{table}[t]
	\centering
	\begin{tabular}{cccc}
		\toprule
		c & ATE & Group Size & Type I error \\
		\hline
		0.4 & -0.0011 $\pm$ 0.0401 & 0.3991 $\pm$ 0.0141 & 0.12 \\
		0.6 & -0.0018 $\pm$ 0.0298 & 0.6020 $\pm$ 0.0141 & 0.00 \\
        0.8 & -0.0000 $\pm$ 0.0212 & 0.8038 $\pm$ 0.0121 & 0.00 \\
		\bottomrule
	\end{tabular}
    \caption{Type I error under different group size constraints.} \label{table:type_i_error}
\end{table}

\section{Runtime Analysis}\label{appx:runtime}
Each method is run three times, and the average run time is reported; these experiments are run on CPU only to reflect clinical computing environments with limited resources. Since MOSIC's nuisance estimation step uses Dragonnet and Logistic regression, it shares the same nuisance-fitting cost as Dragonnet (Logistic Regression is negligible). Hence, the reported runtime for MOSIC reflects only the additional cost of the optimization step. As shown in Table~\ref{tab:runtime}, this overhead is small relative to nuisance estimation, indicating that computation is unlikely to be a deployment bottleneck.

\end{document}